\documentclass{article}

\usepackage[T1]{fontenc}
\usepackage{url}
\usepackage{ifthen}
\usepackage{cite}
\usepackage{amsmath}

\usepackage[T1]{fontenc}
\usepackage{tikz}
\usetikzlibrary{arrows}
\usetikzlibrary{calc}
\usepackage{cite}
\usepackage{graphicx}
\usepackage{array}
\usepackage{mdwmath}
\usepackage{mdwtab}
\usepackage{eqparbox}
\usepackage{fixltx2e}
\usepackage{url}
\usepackage{pgf}
\usepackage{lipsum}
\usepackage{multirow}
\usepackage{amsfonts,amssymb,amsthm}
\usepackage{xspace}
\usepackage{xcolor}
\usepackage{algorithmic}
\usepackage{bm}
\usepackage{bbm}
\usepackage{booktabs}
\usepackage{etoolbox}
\usepackage{makecell}
\usepackage{nicefrac}
\usepackage{textcomp}
\usepackage{stmaryrd}
\usepackage{mathrsfs}
\usepackage{paralist}
\usepackage{comment}
\usepackage{lmodern}

\usepackage[font=footnotesize,labelsep=space,labelfont=bf]{caption}

\usepackage[colorlinks=true, citecolor=burntumber, linkcolor = bluepigment ,urlcolor=blue, pagebackref=false]{hyperref}

\makeatletter

\newcommand{\Rmnum}[1]{\expandafter\@slowromancap\romannumeral #1@}
\makeatother

\newtheorem{theorem}{Theorem}
\newtheorem*{theorem*}{Theorem}
\newtheorem{definition}{Definition}
\newtheorem{lemma}{Lemma}
\newtheorem*{lemma*}{Lemma}
\newtheorem{claim}{Claim}
\newtheorem*{cor*}{Corollary}
\newtheorem{remark}{Remark}
\newtheorem{proposition}{Proposition}
\newtheorem{fact}{Fact}
\newtheorem*{fact*}{Fact}

\newcommand{\TODO}[1]{{{\color{red}#1}}}

\newcommand{\namedref}[2]{\hyperref[#2]{#1~\ref*{#2}}}

\newcommand{\bzero}{\ensuremath{{\bf 0}}\xspace}

\newcommand{\ba}{\ensuremath{{\bf a}}\xspace}
\newcommand{\bb}{\ensuremath{{\bf b}}\xspace}

\newcommand{\bg}{\ensuremath{{\bf g}}\xspace}

\newcommand{\bx}{\ensuremath{{\bf x}}\xspace}

\newcommand{\btx}{\ensuremath{\widetilde{\bf x}}\xspace}
\newcommand{\blx}{\ensuremath{\overline{\bf x}}\xspace}
\newcommand{\by}{\ensuremath{{\bf y}}\xspace}

\newcommand{\bs}{\ensuremath{{\bf s}}\xspace}

\newcommand{\bu}{\ensuremath{{\bf u}}\xspace}
\newcommand{\bv}{\ensuremath{{\bf v}}\xspace}

\newcommand{\bI}{\ensuremath{{\bf I}}\xspace}
\newcommand{\bF}{\ensuremath{\boldsymbol{F}}\xspace}

\newcommand{\bX}{\ensuremath{{\bf X}}\xspace}
\newcommand{\bV}{\ensuremath{{\bf V}}\xspace}
\newcommand{\bW}{\ensuremath{{\bf W}}\xspace}

\newcommand{\bhX}{\ensuremath{\widehat{\bf X}}\xspace}
\newcommand{\blX}{\ensuremath{\overline{\bf X}}\xspace}
\newcommand{\bxi}{\ensuremath{\boldsymbol{\xi}}\xspace}

\newcommand{\C}{\ensuremath{\mathcal{C}}\xspace}
\newcommand{\bbE}{\ensuremath{\mathbb{E}}\xspace}
\newcommand{\bbN}{\ensuremath{\mathbb{N}}\xspace}
\newcommand{\bbR}{\ensuremath{\mathbb{R}}\xspace}

\newcommand{\I}{\ensuremath{\mathcal{I}}\xspace}

\newcommand{\R}{\ensuremath{\mathbb{R}}\xspace}

\renewcommand{\paragraph}[1]{\smallskip\noindent{\bf #1}~}

\usepackage{subcaption}
\usepackage{algorithm}
\usepackage{fullpage}
\usepackage{xcolor}
\usepackage{authblk}
\usepackage{algorithmic}
\usepackage{enumitem}
\usepackage{nicefrac}
\setlist[itemize]{leftmargin=*,itemsep=1pt,parsep=0.5pt}
\usepackage{etoolbox}\AtBeginEnvironment{algorithmic}{\small}
\newcommand{\verts}[1]{\left\Vert #1 \right\Vert}
\newcommand{\lragnle}[1]{\left\langle #1 \right\rangle}
\newcommand{\hatt}{\hat{t}}
\renewcommand{\(}{\left(}
\renewcommand{\)}{\right)}
\newtheorem{assumption}{Assumption}

\usepackage[margin=1in]{geometry}
\definecolor{brightmaroon}{rgb}{0.76, 0.13, 0.28}
\definecolor{ceruleanblue}{rgb}{0.16, 0.32, 0.75}
\definecolor{bluepigment}{rgb}{0.2, 0.2, 0.6}
\definecolor{amaranth}{rgb}{0.9, 0.17, 0.31}
\definecolor{auburn}{rgb}{0.43, 0.21, 0.1}
\definecolor{burntumber}{rgb}{0.54, 0.2, 0.14}

\begin{document}
	\title{SQuARM-SGD: Communication-Efficient Momentum SGD for Decentralized Optimization}

\author[1]{Navjot Singh}
\author[1]{Deepesh Data}
\author[2]{Jemin George}
\author[1]{Suhas Diggavi}

\affil[1]{University of California, Los Angeles, USA}
\affil[1] {\text{navjotsingh@ucla.edu, deepesh.data@gmail.com, suhas@ee.ucla.edu}\vspace{0.25cm}}
\affil[2]{US Army Research Lab, Maryland, USA}
\affil[2] {\text{jemin.george.civ@mail.mil}}

\date{\vspace{-5ex}}
\maketitle

\begin{abstract}
	In this paper, we propose and analyze SQuARM-SGD, a communication-efficient algorithm for decentralized training of large-scale machine learning models over a network. In SQuARM-SGD, each node performs a fixed number of local SGD steps using Nesterov's momentum and then sends sparsified and quantized updates to its neighbors regulated by a locally computable triggering criterion. We provide convergence guarantees of our algorithm for general (non-convex) and convex smooth objectives, which, to the best of our knowledge, is the first theoretical analysis for compressed decentralized SGD with momentum updates. We show that the convergence rate of SQuARM-SGD matches that of vanilla SGD. We empirically show that including momentum updates in SQuARM-SGD can lead to better test performance than the current state-of-the-art which does not consider momentum updates.
\end{abstract}


\section{Introduction}\label{intro}
\allowdisplaybreaks
{ As machine learning gets deployed over edge (wireless) devices (in
contrast to datacenter applications), the problem of building learning
models on local (heterogeneous) data with communication-efficient training  becomes important. These applications
motivate learning when data is collected/available locally, but
devices collectively help build a model through wireless links with
significant communication rate (bandwidth) constraints.\footnote{This
is also motivated by federated learning \cite{konevcny2016federated},
which is studied mostly for the client-server model.} 
Several methods have been developed recently to obtain communication-efficiency
in \emph{distributed} stochastic gradient descent (SGD). These methods
can be broadly divided into two categories. In the first one, 
workers \emph{compress} information/gradients before communicating -
either with \emph{sparsification} \cite{speech2,AjiHeafield17,DeepCompICLR18,stich_sparsified_2018,alistarh_convergence_2018}, \emph{quantization} \cite{alistarh-qsgd-17,terngrad,teertha,stich-signsgd-19,bernstein2018signsgd},
or both \cite{basu_qsparse-local-sgd:_2019}. Another way to reduce
communication is to skip communication rounds while performing a
certain number of \emph{local SGD} steps, thus trading-off computation and
communication
time \cite{stich_local_2018,alibaba_local,Coppola15}. Since
momentum-based methods generally converge faster and generalize well,
they have been adopted ubiquitously for training large-scale machine
learning models \cite{UnifiedAnalysis_Momentum18}.

To reduce communication load on the central-coordinator in the
distributed framework, a \emph{decentralized} setting has been
considered in literature \cite{lian2017can}, where the central
coordinator is absent, and training is performed collaboratively
among workers, which are connected by a (sparse) graph.\footnote{This
can also be motivated through learning over local wireless mesh (or ad
hoc) networks.}  Compressed communication has been studied recently
for decentralized training as
well \cite{Tong_decentralized18,singh2019sparq,
koloskova_decentralized_2019,koloskova_decentralized_2019-1,tang2019doublesqueeze}. 
Out of these \cite{Tong_decentralized18,
koloskova_decentralized_2019,koloskova_decentralized_2019-1,tang2019doublesqueeze} 
only employ either quantization or sparsification (without
local iterations or event-triggered communication), whereas, \cite{singh2019sparq} also incorporates event-triggering to achieve communication efficiency; see related work for a detailed comparison. We would like to remark two important aspects of these works: {\sf (i)} They rely on strong set of assumptions for their theoretical analyses: all of them assume a uniform bound on variance of stochastic gradients and also on the gradient dissimilarity across the clients, while \cite{singh2019sparq,koloskova_decentralized_2019,koloskova_decentralized_2019-1, tang2019doublesqueeze} assume a bound on the second moment of stochastic gradients. {\sf (ii)} None of these works incorporates momentum in their theoretical analyses, which has been very successful in achieving good generalization error in training large-scale machine learning models.


In this paper, we propose and analyze SQuARM-SGD,\footnote{Acronym stands for Sparsified
	and Quantized Action Regulated Momentum Stochastic Gradient Descent. See Algorithm \ref{alg_dec_sgd_li} for a description of
	SQuARM-SGD.} a communication efficient SGD algorithm for decentralized optimization that
incorporates Nesterov's momentum, compression and local iterations while considering a much weaker set of assumptions than existing literature. 

For compression, SQuARM-SGD uses both sparsification and quantization.
For event-triggered communication, each worker first performs a
certain number of {\em local SGD} iterations with momentum updates;
then in order to further reduce communication, it only does so if
there is a significant change in the local model parameters (greater
than a prescribed threshold) since its last communication. If there is
a significant model change, the worker communicates a sparsified and
quantized version of (the difference of) its local parameters (model)
to its neighbors. Therefore, this combines lazy updates along with
quantization and sparsification to enable communication-efficient
decentralized training.

\paragraph{Our contributions.} 
In this paper, we propose and analyze SQuARM-SGD, a communication efficient decentralized training algorithm incorporating compression and local iterations. Our analysis is the first to establish convergence rates of compressed decentralized training algorithms with momentum. We provide separate convergence results for SQuARM-SGD with two sets of assumptions: {\sf (i)} Commonly used assumptions in decentralized optimization, including bounded second moment of stochastic gradients \cite{koloskova_decentralized_2019, koloskova_decentralized_2019-1, singh2019sparq} (presented in Section \ref{subsec:strong_assump}),{\sf (ii)} A relatively weaker set of assumptions on the node variance and the gradient dissimilarity across nodes (presented in Section \ref{subsec:weak_assump}). Specifically, the bounds on the variance and the gradient dissimilarity depend on the local geometry of the true gradients; see Assumption~\ref{assump:variance} for the bounded variance assumption and Assumption~\ref{assump:grad-dissim} for the bounded gradient dissimilarity assumption. 
Both these assumptions are strictly weaker than assuming uniform bounds on the respective quantities;
see Remark~\ref{remark:assump-comparison} for a detailed discussion.
 For assumptions set {\sf (i)}, we show a convergence rate of
$\mathcal{O} \left(\nicefrac{1}{\sqrt{nT}}\right)$ for smooth convex and non-convex objectives, where $n$ is the number of worker nodes and $T$ is the
number of iterations, thus matching the convergence rate of vanilla
distributed SGD. Similarly, for the weaker assumption set {\sf (ii)}, we show a convergence rate of $\mathcal{O}\left( \nicefrac{1}{\sqrt{T}} \right)$ for smooth non-convex objectives.
We note that compression and
event triggered communication do affect our convergence rate
expressions for results in both sets of assumptions, but they appear only in the higher order terms; thus, for
a large enough $T$, we can converge at the same rate as that of
distributed vanilla SGD while enjoying the savings in communication
from our method essentially for free; see Theorem~\ref{convergence_relaxed_assump} and Theorem \ref{convergence_results} and comments after that
for details. As mentioned earlier, we use Nesterov's momentum in
SQuARM-SGD and theoretically analyze its convergence rate; a first
theoretical analysis of convergence of such compressed gradient
updates with momentum in the decentralized setting. In order to
achieve this, we had to solve several technical difficulties; see
Section \ref{prelim} and also the related work below.
Our numerical results for decentralized training of ResNet20 \cite{he2015deep} model on CIFAR-10 \cite{cifar} dataset shows that including momentum updates as in SQuARM-SGD can lead to around $2\%$ increase in test accuracy performance in comparison to the recently proposed communication efficient algorithms CHOCO-SGD \cite{koloskova_decentralized_2019} or SPARQ-SGD \cite{singh2019sparq} which do not use momentum.

\paragraph{Related work.}
Communication-efficient decentralized training has received recent attention; see \cite{Tong_decentralized18,singh2019sparq,reisizadeh2018quantized,assran2018stochastic,tatarenko2017non,koloskova_decentralized_2019,Momentum_linear-speedup19, wang2018cooperative,wang2019matcha} and references therein.
CHOCO-SGD proposed by \cite{koloskova_decentralized_2019, koloskova_decentralized_2019-1} was the first to perform arbitrary compressed training for decentralized optimization by considering sparsification or quantization of the model parameters. Recently, in \cite{singh2019sparq} we proposed SPARQ-SGD incorporating compression using both sparsification and quantization and also event-driven communication with local iterations to save on communicated bits.
We remark that \cite{koloskova_decentralized_2019,koloskova_decentralized_2019-1, singh2019sparq} rely on (a strong) assumption of bounded second moment of stochastic gradients for their theoretical analysis and do not incorporate momentum updates,
which has been shown to empirically improve generalization performance in deep learning applications \cite{wilson2017marginal,Momentum_linear-speedup19}.
Our convergence analyses are very different and more involved than CHOCO-SGD or SPARQ-SGD,
as we rely on a much weaker set of assumptions and provide our analyses using virtual sequences, specifically, to handle the use of momentum.
Use of local iterations in decentralized setting with a weaker set of assumptions similar to ours has been considered recently in \cite{KoloskovaDecenRelaxAssum20}, however, without any compression of updates, and importantly, without incorporating momentum in the theoretical analysis.
The use of local iterations with momentum updates in decentralized setting has been studied in \cite{wang2019slowmo},
but without any compression of exchanged
information and with a stronger set of assumptions. \cite{distributed_momentum19} studied momentum SGD with
compressed updates (but no local iterations or event-triggering) for the \emph{distributed} setting only, assuming
that all workers have access to unbiased gradients. 
Extending the
analysis to the \emph{decentralized} setting (where different workers may
have local data, potentially generated from different distributions)
while incorporating momentum, compression, local iterations, and event triggered
communication\footnote{Event-triggered communication with compression and local iterations is also considered in \cite{singh2019sparq}, however, with the strong bounded second moment gradient assumption and without momentum updates in the theoretical analysis. Relaxing the assumptions and incorporating momentum significantly changes the convergence analysis (see Section \ref{prelim}). 
} 
(as in SQuARM-SGD) while assuming a weaker set of assumptions than existing works poses several challenges; see
Section \ref{prelim} for a detailed discussion. The idea of
event-triggering has been explored in the control
community \cite{Paulo-event-triggered12,distri-event-triggered12,event-triggered-consensus13,dynamic-event-triggered15,liu2017asynchronous}
and in the optimization
literature \cite{convex-optimize-event15,event-gradient-sum-consensus16,optimize-dynamic-event-triggered18}. These
papers focus on continuous-time, deterministic optimization algorithms for convex problems; in contrast, our event-driven stochastic
gradient descent algorithm is for both convex and general (non-convex) smooth objectives, e.g., neural network training for large-scale deep learning. \cite{chen2018lag} proposed an adaptive scheme to skip
gradient computations in a \emph{distributed} setting
for \emph{deterministic} gradients; moreover, their focus is on saving
communication rounds, without compressed communication. To the best of our knowledge, ours is the first paper to develop and analyze
convergence of momentum-based decentralized stochastic optimization,
using compressed lazy communication (as described earlier). Moreover,
our numerics demonstrate better test-accuracy performance compared to recently proposed methods for communication efficiency on account of using momentum updates.

\paragraph{Paper organization.}  The problem setup and our algorithm SQuARM-SGD are
described in Section \ref{algo}. Section \ref{main-results} provides
two sets of convergence results, one with weak assumptions (Theorem~\ref{convergence_relaxed_assump}), and the other (a slightly general result) with strong assumptions (Theorem~\ref{convergence_results}). We prove Theorem~\ref{convergence_relaxed_assump} in Section~\ref{sec:relaxed-assump-results} (which is a novel analysis and the main technical contribution of our paper) and defer the proof of Theorem~\ref{convergence_results} to the supplementary material.
Section \ref{experiments} gives numerical results comparing our algorithm to the state-of-the-art.
Omitted proofs/details are provided in appendices.}

\section{Problem Setup and Our Algorithm}\label{algo}
 \allowdisplaybreaks
 {
 We first formalize the decentralized optimization setting that we work with and set up the notation we follow throughout the paper. Consider an undirected connected graph $\mathcal{G}=(\mathcal{V}, \mathcal{E})$ with $\mathcal{V}=[n]:= \{1,2,\hdots, n \}$, where node $i\in[n]$ corresponds to worker $i$ and we denote the neighbors of node $i$ by $\mathcal{N}_i:=\{(i,j):(i,j)\in\mathcal{E}\}$.
To each node $i\in[n]$, we associate a dataset $\mathcal{D}_i$ and an objective function $f_i : \mathbb{R}^d \rightarrow \mathbb{R}$. We allow the datasets and objective functions to be different for each node and assume that for $i \in [n]$, the objective function $f_i$ has the form $f_i (\bx) = \mathbb{E}_{\xi_i \sim \mathcal{D}_i}  [ F_i(\bx, \xi_i)  ] $ where $\xi_i \sim \mathcal{D}_i$ denotes a random sample from $\mathcal{D}_i$, $\bx$ denotes the parameter vector, and $F_i (\bx, \xi_i)$ denotes the risk associated with sample $\xi_i$ with respect to (w.r.t.) the parameter vector $\bx$.
Consider the following empirical risk minimization problem, where $f:\R^d\to\R$ is called the global objective function:
\begin{align}\label{eq:obj-fn}
\arg\min_{\bx\in\R^d} \Big(f(\bx) := \frac{1}{n} \sum_{i=1}^{n} f_i(\bx)\Big),
\end{align}
The nodes in $\mathcal{G}$ wish to minimize \eqref{eq:obj-fn} collaboratively in a communication-efficient manner while incorporating momentum updates of worker nodes. 

We now state the notation relevant to describing our algorithm.
Let $\mathbf{W} \in \mathbb{R}^{n \times n}$ denote the connectivity matrix of $\mathcal{G}$, 
where for every $(i,j) \in \mathcal{E}$, the $(i,j)$'th entry of $\mathbf{W}$ denotes the weight $w_{ij}$ on the edge $(i,j)$ -- e.g., $w_{ij}$ may represent the strength of the connection on the edge $(i,j)$ -- and for other pairs $(i,j) \notin \mathcal{E}$, the weight $w_{ij}$ is zero. We assume that $\mathbf{W}$ is symmetric and doubly stochastic, which means it has non-zero entries with each row and column summing up to 1. Consider the ordered eigenvalues of $\mathbf{W}$, $|\lambda_1(\mathbf{W})| \geq |\lambda_2(\mathbf{W})| \geq \hdots\geq |\lambda_n(\mathbf{W})| $. For such a $\mathbf{W}$ associated with a connected graph $\mathcal{G}$, it is known that $\lambda_1(\mathbf{W}) = 1$ and $\lambda_i(\mathbf{W}) \in (-1,1) $ for all $i \in \{2, \hdots,n\}$. The spectral gap $\delta \in (0,1] $ is defined as $\delta := 1 - |\lambda_2(\mathbf{W})|$. Simple matrices $\mathbf{W}$ having $\delta \in (0,1]$ are known to exist for connected graphs \cite{koloskova_decentralized_2019-1}.

To achieve compression on the communication exchanged between workers, we use arbitrary compression operators as defined next.
\begin{definition}[Compression, \cite{stich_sparsified_2018}]\label{definition:compression}
A (possibly randomized) function $\C:\R^d\to\R^d$ is called a {\em compression} operator, if there exists a positive constant $\omega \in (0,1]$, such that for every $\bx\in\R^d$:
\begin{align}\label{eq:compression}
\textstyle \mathbb{E}_{\C}[\|\bx-\C(\bx)\|_2^2]\leq (1-\omega)\|\bx\|_2^2,
\end{align}
where expectation is taken over the randomness of $\C$.
We assume that $\C(\bzero)=\bzero$.
\end{definition}
We now list some important sparsifiers and quantizers following the above definition of a compression operator: \\
{\sf (i)} $Top_k$ and $Rand_k$ sparsifiers (where only $k$ entries are selected and the rest are set to zero) with $\omega=k/d$ \cite{stich_sparsified_2018},
{\sf (ii)} Stochastic quantizer $Q_s$ from \cite{alistarh-qsgd-17}\footnote{$Q_s:\R^d\to\R^d$ is a stochastic quantizer, if for every $\bx\in\R^d$, we have {\sf(i)} $\mathbb{E}[Q_s(\bx)]=\bx$ and {\sf (ii)} $\mathbb{E}[\|\bx-Q_s(\bx)\|_2^2]\leq\beta_{d,s}\|\bx\|_2^2$. $Q_s$ from \cite{alistarh-qsgd-17} satisfies this definition with $\beta_{d,s}=\min\left\{\frac{d}{s^2},\frac{\sqrt{d}}{s}\right\}$.} with $\omega=(1-\beta_{d,s})$ for $\beta_{d,s}<1$, and 
{\sf (iii)} Deterministic quantizer $\frac{\|\bx\|_1}{d}Sign(\bx)$ from \cite{stich-signsgd-19} with $\omega=\frac{\|\bx\|_1^2}{d\|\bx\|_2^2}$. 
For $Comp_k\in\{Top_k,Rand_k\}$, the following are compression operators\footnote{\cite{basu_qsparse-local-sgd:_2019} show that the composition of sparsification and quantization operators is also a valid compression operator, outperforming its individual components in terms of communication savings while maintaining similar performance.}:
{\sf (iv)} $\frac{1}{(1+\beta_{k,s})}Q_s(Comp_k)$ with $\omega=\left(1-\frac{k}{d(1+\beta_{k,s})}\right)$ for any $\beta_{k,s}\geq0$, and {\sf (v)} $\frac{\|Comp_k(\bx)\|_1Sign(Comp_k(\bx))}{k}$ with $\omega=\max\left\{\frac{1}{d},\frac{k}{d}\left(\frac{\|Comp_k(\bx)\|_1^2}{d\|Comp_k(\bx)\|_2^2}\right)\right\}$ \cite{basu_qsparse-local-sgd:_2019}. 

\subsection{Our Algorithm: SQuARM-SGD}
We propose SQuARM-SGD to minimize \eqref{eq:obj-fn}, which is 
a decentralized algorithm that combines compression and Nesterov's momentum, together with event-driven communication exchange, where compression is achieved by sparsifying \emph{and} quantizing the exchanges. Each worker is required to complete a fixed number of {\em local SGD} steps with {\em momentum}, and communicate {\em compressed} updates to its neighbors when there is a {\em significant change} in its local parameters since the last communication round. 

To realize exchange of compressed parameters between workers, for each node $i \in [n]$, all nodes $j \in \mathcal{N}_i$ maintain an estimate $\hat{\bx}_i $ of $\bx_i$, so, each node $i \in [n]$ has access to $\hat{\bx}_j$ for all $j \in \mathcal{N}_i$. Our algorithm runs for $T$ iterations and the set of synchronization indices is defined as 
$\mathcal{I}_T = \{ 0, H , 2H \hdots, mH, \hdots  \} \subseteq [T]  $ for some constant $H \in \mathbb{N}$ 
, which are same for all workers and denote the time steps at which workers are allowed to communicate, provided they satisfy a triggering condition.\footnote{The Zeno phenomenon \cite{Paulo-event-triggered12} does not occur in our setup as we have a discrete sampling period as well as a fixed number of local iterations, giving a lower bound to the event intervals of at least $H$ times the sampling period.} 


\begin{algorithm}[t]
	\caption{SQuARM-SGD: Sparsified and Quantized Action Regulated Momentum SGD}
	\label{alg_dec_sgd_li}
	{\bf Parameters}: $G = ([n],E)$, $\mathbf{W}$, Compression operator $\C$\\
	\vspace{-0.3cm}
	\begin{algorithmic}[1]
		\STATE {\bf Initialize:} For every $i\in[n]$, set arbitrary $\bx_i^{(0)} \in \mathbb{R}^d$, $\hat{\bx}_i^{(0)} := \bzero$, $\bv_i^{(-1)} := \mathbf{0}$. Fix the momentum coefficient $\beta$, consensus step-size $\gamma$, learning rate $\eta$, triggering thresholds $\{ c_t \}_{t=0}^T$, and synchronization set $\mathcal{I}_T$.
		\FOR{$t=0$ {\bfseries to} $T-1$ in parallel for all workers $i \in [n]$ } 
		\STATE Sample $\xi_i^{(t)}$, compute stochastic gradient $\bg_i^{(t)}:= \nabla F_i(\bx_i^{(t)}, \xi_i^{(t)})$
		\STATE $\bv_i^{(t)} = \beta  \bv_i^{(t-1)} + \bg_i^{(t)} $ 
		\STATE $\bx_i^{(t+\frac{1}{2})} := \bx_i^{(t)} - \eta ( \beta \bv_i^{(t)} + \bg_i^{(t)}) $
		\IF{$(t+1) \in I_T $ }
		\FOR{neighbors $j \in \mathcal{N}_i \cup i $  }
		\IF{ $\Vert \bx_i^{(t+\frac{1}{2})} - \hat{\bx}_i^{(t)} \Vert_2^2 > {c_t \eta^2} $}
		\STATE Compute $\mathbf{q}_i^{(t)} := \C(\bx_i^{(t+\frac{1}{2})} - \hat{\bx}_i^{(t)} )$
		\STATE Send $\mathbf{q}_i^{(t)}$ and receive $\mathbf{q}_j^{(t)}$
		\ELSE 
		\STATE Send $\mathbf{0}$ and receive $\mathbf{q}_j^{(t)}$
		\ENDIF
		\STATE $\hat{\bx}_j^{(t+1)} := \mathbf{q}_j^{(t)} + \hat{\bx}_j^{(t)} $
		\ENDFOR
		\STATE $\bx_i^{(t+1)} = \bx_i^{(t+\frac{1}{2})} + \gamma \sum \limits_{j \in \mathcal{N}_i } w_{ij} (\hat{\bx}_j^{(t+1)} - \hat{\bx}_i^{(t+1)} ) $
		\ELSE
		\STATE $\hat{\bx}_i^{(t+1)} = \hat{\bx}_i^{(t)}$ ,  $\bx_i^{(t+1)} = \bx_i^{(t+\frac{1}{2})}$ for all $i \in [n]$
		\ENDIF
		\ENDFOR
	\end{algorithmic}
\end{algorithm}
For a given connected graph $\mathcal{G}$ with connectivity matrix $\mathbf{W}$, we first initialize a consensus step-size $\gamma$ (see Theorem~\ref{convergence_relaxed_assump} for definition), momentum factor $\beta$, learning rate $\eta$, triggering threshold sequence $\{c_t\}_{t=0}^T$, and momentum vector $\bv_i$ for each node $i$ initialized to $\mathbf{0}$. 
We initialize the copies of all the nodes $\hat{\bx}_i = \mathbf{0}$ 
and allow each node to communicate in the first round. At each time step $t$, each worker $i \in [n]$  samples a stochastic gradient $\nabla F_i(\bx_i^{(t)}, \xi_i) $ and takes a local SGD step on parameter $\bx_i^{(t)}$ using Nesterov's momentum to form an intermediate parameter $\bx_i^{(t+\nicefrac{1}{2})}$ (lines 3-5). If the next iteration corresponds to a synchronization index, i.e., $(t+1) \in \mathcal{I}_T$, then each worker checks the triggering condition (line 8). If satisfied, that worker communicates the compressed change in its copy to all its neighbors $\mathcal{N}_i$ (lines 9-10); otherwise, it does not communicate in that round (denoted by `Send $\mathbf{0}$' in our algorithm for illustration, line 12). After receiving the compressed updates of copies from all its neighbors, the node $i$ updates the locally available copies and its own copy (line 14). With these updated copies, the worker nodes finally take a consensus (line 16) with appropriate weighting decided by entries of $\mathbf{W}$. In the case when $(t+1) \notin \mathcal{I}_T$, the nodes maintain their copies and move on to next iteration (line 18); thus no communication takes place.

\paragraph{Difference from SPARQ-SGD \cite{singh2019sparq}:} There are two major differences between this work and our previous work \cite{singh2019sparq} which uses a similar framework of local iterations, compression and triggering to save on communication. Firstly, and most importantly, the results presented in this work do not use any strong assumptions like the bounded second moment of stochastic gradients used in \cite{singh2019sparq, koloskova_decentralized_2019, koloskova_decentralized_2019-1}: Both the variance bound on stochastic gradients as well as the data heterogeneity bound depend on local geometry of the true gradients (and we allow these to scale with the true gradient norm); and thus, neither of them are assumed to be uniformly bounded, as in \cite{singh2019sparq, koloskova_decentralized_2019, koloskova_decentralized_2019-1}. The assumptions in this work are thus much weaker than the ones in existing decentralized literature; see Section \ref{prelim} for details. Working with these relaxed assumptions calls for completely different and much more nuanced analyses to establish the convergence rates as compared to \cite{singh2019sparq}.    
Secondly, the addition of lines 4-5 in Algorithm \ref{alg_dec_sgd_li} which now incorporate momentum calls for a significantly different analysis than \cite{singh2019sparq} to arrive at the convergence rate even if we consider the same set of assumptions.
Even though momentum updates are almost always used in practice, incorporating them in convergence analyses in modern large-scale settings with communication constraints has received attention only recently, e.g., for distributed training with compressed update exchanges \cite{distributed_momentum19} and for decentralized training without compression or local SGD in \cite{Momentum_linear-speedup19}.
To the best of our knowledge, our work provides the first convergence analysis for compressed decentralized training with momentum using a weaker set of assumptions than existing literature while incorporating the local SGD and event triggered communication framework of \cite{singh2019sparq}. We note the technical challenges that arise and provide a detailed comparison to SPARQ-SGD \cite{singh2019sparq} and other recent works analyzing momentum in Section \ref{prelim}. Furthermore, our experimental results in Section \ref{experiments} show that incorporating momentum can empirically improve the generalization performance of the trained model by about $2$-$3\%$ when compared to training without momentum.

\paragraph{Memory-efficient version of Algorithm \ref{alg_dec_sgd_li}:}
At the first glance, it may seem that in Algorithm~\ref{alg_dec_sgd_li}, every node has to store estimates of all its neighbors' parameters in order to perform the consensus step, which may be impractical in large-scale learning. 
Note that in the consensus step (line 16), nodes only require the weighted sum of their neighbors' parameters. So, it suffices for each node to store only the weighted sum of all its neighbors' parameters (in addition to its own local parameters and its estimate), and thus avoiding the need to store all neighbor parameters.
A memory-efficient version of SQuARM-SGD is given in Appendix I.

\paragraph{Equivalence to error-feedback mechanisms:}
In Algorithm \ref{alg_dec_sgd_li}, though nodes do not explicitly perform local error-compensation (\cite{stich-signsgd-19,basu_qsparse-local-sgd:_2019}), the error-compensation happens implicitly. 
To see this, note that nodes maintain copies of their neighbors' parameters and update them as $\hat{\bx}_j^{(t+1)} = \hat{\bx}_j^{(t)} + \C(\bx_j^{(t+\frac{1}{2})} - \hat{\bx}_j^{(t)})$ (line 14) and then perform consensus (line 16). Thus, the error gets accumulated into $\hat{\bx}_j^{(t)}$ and is compensated by the term $\C(\bx_j^{(t+\frac{1}{2})} - \hat{\bx}_j^{(t)})$ in the next round.

}
\section{Main Results}\label{main-results}
\allowdisplaybreaks
{

In this section we provide the convergence results for SQuARM-SGD (Algorithm \ref{alg_dec_sgd_li}) under two sets of assumptions: We present our results with the weakest set of assumptions available in existing literature in Section~\ref{subsec:weak_assump} and slightly more general results with stronger assumptions in Section~\ref{subsec:strong_assump}.

	
\subsection{Theoretical Results with Relaxed Assumptions} \label{subsec:weak_assump}
	
		\begin{assumption}[Smoothness]\label{assump:smoothness}
			We assume that each local function $f_i$ for $i \in [n]$ is $L$-smooth, i.e., $\forall \bx,\by \in \mathbb{R}^d$, we have $f_i(\by) \leq f_i(\bx) + \langle \nabla f_i(\bx), \by-\bx  \rangle + \frac{L}{2} \Vert \by-\bx \Vert^2$.
		\end{assumption}
		\begin{assumption}[Bounded Variance]\label{assump:variance}
			We assume that there exists finite constants $\sigma,M\geq0$, such that for all $\bx\in\bbR^d$ we have:
			\begin{align}\label{eq:variance}
				\frac{1}{n}\sum_{i=1}^n\bbE_{\xi_i}\|\nabla F_i(\bx_i,\xi_i) - \nabla f_i(\bx_i)\|_2^2 &\leq \sigma^2 + \frac{M^2}{n}\sum_{i=1}^n\|\nabla f_i(\bx_i)\|_2^2,
			\end{align}
			where $\nabla F_i (\bx,\xi_i)$, $i\in[n]$, denotes an unbiased stochastic gradient, i.e., $\mathbb{E}_{\xi_i} [\nabla F_i (\bx ,\xi_{i})] = \nabla f_i (\bx)$.
		\end{assumption}
		\begin{assumption}[Bounded Gradient Dissimilarity]\label{assump:grad-dissim}
			We assume that there exists finite constants $G\geq0$ and $B\geq1$, such that for all $\bx\in\bbR^d$ we have:
			\begin{align}\label{eq:grad-dissim}
				\frac{1}{n}\sum_{i=1}^n\|\nabla f_i(\bx)\|_2^2 &\leq G^2 + B^2\|\nabla f(\bx)\|_2^2.
			\end{align}
		\end{assumption}
\noindent These assumptions have appeared in literature before in \cite{KoloskovaDecenRelaxAssum20} to study decentralized optimization with local iterations; and we extend their results and analyses by incorporating compression and momentum. This extension posed many fundamental technical difficulties, which we describe in detail in Section~\ref{prelim}. 

	\begin{remark}[Comparison with Existing Assumptions]\label{remark:assump-comparison}
		Assumptions~\ref{assump:variance},~\ref{assump:grad-dissim} are weaker than assuming uniform bounds on the variance and the gradient dissimilarity: {\sf (i)} The uniform bound on the variance \cite{Momentum_linear-speedup19}, i.e., $\bbE_{\xi_i}\|\nabla F_i(\bx_i,\xi_i) - \nabla f_i(\bx_i)\|_2^2 \leq \sigma_i^2$ for all $i\in[n]$, implies Assumption~\ref{assump:variance} with $\sigma^2=\frac{1}{n}\sum_{i=1}^n\sigma_i^2$ and $M=0$; and {\sf (ii)} The uniform bound on the gradient similarity \cite{Momentum_linear-speedup19}, i.e., $\frac{1}{n}\sum_{i=1}^n\|\nabla f_i(\bx) - \nabla f(\bx)\|_2^2 \leq \kappa^2$, implies Assumption~\ref{assump:grad-dissim} with $G=\kappa$ and $B=1$ -- this follows from the identity $\frac{1}{n}\sum_{i=1}^n\|\nabla f_i(\bx) - \nabla f(\bx)\|_2^2 = \frac{1}{n}\sum_{i=1}^n\|\nabla f_i(\bx)\|_2^2 - \|\nabla f(\bx)\|_2^2$.
		Both Assumptions~\ref{assump:variance} and \ref{assump:grad-dissim} are weaker than the uniformly bounded second moment assumption $\bbE_{\xi_i}\|\nabla F_i(\bx_i,\xi_i)\|_2^2 \leq G^2$, which has been standard in the stochastic optimization with compressed gradients \cite{stich_sparsified_2018,basu_qsparse-local-sgd:_2019,koloskova_decentralized_2019,distributed_momentum19}.
	\end{remark}

Our convergence result (stated below) is for general smooth (non-convex) objectives; and can be readily extended to convex objectives. We derive this result for SQuARM-SGD under Assumptions~\ref{assump:smoothness}-\ref{assump:grad-dissim} without event-triggered communication; in other words, our analysis is for compressed decentralized momentum SGD with local iterations. We would like to emphasize that incorporating event-triggering component into our analysis can only complicate the calculations and can be done. In order to bring out the novelty of our convergence analysis without adding unnecessary technicality, we present the result in this subsection and its subsequent analysis without incorporating event-triggered communication.

\begin{theorem}\label{convergence_relaxed_assump}
		Let $\C$ be a compression operator with parameter $\omega \in (0,1]$ and  $gap(\I_T)= H$. Consider running SQuARM-SGD for $T$ iterations with consensus step-size $\gamma = \frac{2\delta \omega^3}{ 4 \delta^2 \omega^2 + \delta^2 + 128 \lambda^2 + 24 \omega^2 \lambda^2 }$, $(\text{where }\lambda = \text{max}_i \{ 1- \lambda_i(\mathbf{W}) \})$,
		momentum coefficient $\beta\in[0,1)$, and constant learning rate $\eta=(1-\beta)\sqrt{\frac{n}{T}}$. Let the algorithm generate $\{ \bx_i ^{(t)}  \}_{t=0}^{T-1}$ for $i\in[n]$. 
	 Running the algorithm for $ T \geq U_0 $ for some constant $U_0$ defined in Appendix C-F, the averaged iterates $\blx^{(t)} := \frac{1}{n} \sum_{i=0}^n \bx_i^{(t)}$ satisfy:
	\begin{align*}
		&\frac{\sum_{t=0}^{T-1}  \mathbb{E} \Vert \nabla f(\blx^{(t)}) \Vert_2^2}{T} =  \mathcal{O}  \left(  \frac{J^2 + \sigma^2 + (M^2+n)G^2 }{\sqrt{nT}} \right)  +\mathcal{O} \left( \frac{ (1-\beta)^2 nH^2( (M^2+1)G + \sigma^2) }{T \delta^2 \omega^3  }  \right),
	\end{align*}
	where $J^2 < \infty$ is such that $\mathbb{E}[f(\blx^{(0)})] - f^* \leq J^2 $.
	\end{theorem}
	\noindent 	We prove Theorem~\ref{convergence_relaxed_assump} in Section~\ref{sec:relaxed-assump-results}. Note that we have used simplified convergence rate expressions in the above result, and derive precise rate expressions in Section~\ref{sec:relaxed-assump-results}.

\subsection{Theoretical Results with Bounded Second Moment of Stochastic Gradients} \label{subsec:strong_assump}	
In this section, we consider a stronger set of assumptions than the ones before along with the smoothness of objectives:\\
{\sf (i)} {\textit{Uniformly bounded variance:}} For every $i\in[n]$, we have $\mathbb{E}_{\xi_i} \Vert \nabla F_i (\bx,\xi_i) - \nabla f_i(\bx) \Vert^2 \leq \sigma_i^2$, for some finite $\sigma_i$,
where $\nabla F_i (\bx,\xi_i)$ denotes an unbiased stochastic gradient at worker $i$ with $\mathbb{E}_{\xi_i} [\nabla F_i (\bx ,\xi_{i})] = \nabla f_i (\bx)$. We define $\bar{\sigma}^2 := \frac{1}{n} \sum_{i=1}^{n}  \sigma_i ^2$.\\ 
{\sf (ii)} {\textit{Uniformly bounded second moment:}} For every $i\in[n]$, we have $\mathbb{E}_{\xi_i} \Vert \nabla F_i (\bx,\xi_{i}) \Vert^2 \leq G^2<\infty$. 


\begin{theorem}\label{convergence_results}
	Let $\C$ be a compression operator with parameter $\omega \in (0,1]$ and  $gap(\I_T)= H$. Consider running SQuARM-SGD for $T$ iterations with consensus step-size $\gamma = \frac{2\delta \omega}{64 \delta + \delta^2 + 16 \lambda^2 + 8 \delta \lambda^2 - 16\delta \omega}$, $(\text{where }\lambda = \text{max}_i \{ 1- \lambda_i(\mathbf{W}) \})$, a threshold sequence $c_t \leq \frac{c_0}{\eta^{1-\epsilon}}$ for all $t$ where $\epsilon \in (0,1)$ and $c_0$ is a constant, momentum coefficient $\beta\in[0,1)$, and constant learning rate $\eta=(1-\beta)\sqrt{\frac{n}{T}}$. Let the algorithm generate $\{ \bx_i ^{(t)}  \}_{t=0}^{T-1}$ for $i\in[n]$. Then, we have:
	\begin{itemize}
		\item{\bf[Non-convex:]} \label{thm_noncvx_fix_li} For $ T \geq \max \{ 16L^2n, \frac{8L^2 \beta^4 n}{(1-\beta)^2}  \} $, the averaged iterates $\blx^{(t)} := \frac{1}{n} \sum_{i=0}^n \bx_i^{(t)}$ satisfy:
		\begin{align*}
		&\frac{\sum_{t=0}^{T-1}  \mathbb{E} \Vert \nabla f(\blx^{(t)}) \Vert_2^2}{T} =  \mathcal{O}  \left(  \frac{J^2 + \bar{\sigma}^2 }{\sqrt{nT}}  \right) +  \mathcal{O}  \left( \frac{c_0 n^{\nicefrac{(1+\epsilon)}{2}}}{ \delta^2 T^{\nicefrac{(1+\epsilon)}{2}}  }  +   \frac{nH^2G^2}{T \delta^4 \omega^2  } + \frac{  \beta^4 \bar{\sigma}^2 }{T (1-\beta)^2 }  \right),
		\end{align*}
		where $J^2 < \infty$ is such that $\mathbb{E}[f(\blx^{(0)})] - f^* \leq J^2 $.
		\item{\bf[Convex:]} \label{thm_cvx_li}  If $\{f_i\}_{i \in [n]}$ are convex, then for $T \geq \max\{ (8L)^2n,\frac{(8\beta^2L)^4n}{(1-\beta)^2}\}$, we have:
		\begin{align*}
		& \mathbb{E}[f(\blx^{(T)}_{avg})] - f^* =  \mathcal{O}\left( \frac{\|\blx^{(0)} - \bx^*\|^2 + \bar{\sigma}^2} {\sqrt{nT}} \right)  + \mathcal{O}\left( \frac{c_0  n^{\nicefrac{(1+\epsilon)}{2}}}{\delta^2 T^{\nicefrac{(1+\epsilon)}{2}}} + \frac{n^{\nicefrac{3}{4}} \beta^2 G^2}{(1-\beta)^{\nicefrac{3}{2}} T^{\nicefrac{3}{4}}}   + \frac{nH^2G^2}{\delta^4 \omega^2T}     \right),
		\end{align*}
		where $\blx^{(T)}_{avg} := \frac{1}{T} \sum_{t=0}^{T-1} \blx^{(t)}$ for $\blx^{(t)} = \frac{1}{n} \sum_{i=1}^{n} \bx^{(t)}_{i}$ and $\bx^{*}$ is an optimizer of $f$ attaining optimal value $f^{*}$.

	\end{itemize}
\end{theorem}
\noindent We have used simplified convergence rate expressions in the above results, and provide precise rate expressions in the proofs provided in Appendix E and Appendix F for non-convex and convex objectives, respectively. 



\subsection{Effects of parameters on convergence} 
The factors arising due to communication efficiency -- $H$ (and $c_0$ for Theorem~\ref{convergence_results}) for the event-triggered communication, $\omega$ for compression, and $\delta$ for the connectivity of the underlying graph -- do not affect the dominant terms in convergence rate for either Theorem~\ref{convergence_relaxed_assump} or Theorem~\ref{convergence_results} and appear only in the higher order terms. 
This implies that if we run SQuARM-SGD for sufficiently long, precisely, for at least $T_{w_0} = C_{w_0} \times \left( \frac{n^3}{\delta^4 \omega^4} \frac{(1-\beta^2)^2   H^4 \left[  (M^2 +1)G + \sigma^2  \right]^2    }{ \left[ J^2 + \sigma^2 + (M^2 +n)G^2 \right]^2  }\right) $ where $G,\sigma,M$ are defined in the weaker set of assumptions provided in Subsection~\ref{subsec:weak_assump} and $C_{w_0}$ is a sufficiently large constant, then SQuARM-SGD converges at a rate $\mathcal{O} \left( \nicefrac{1}{\sqrt{T}} \right)$ . Similarly, if we consider the stronger set of assumptions stated in Subsection~\ref{subsec:strong_assump}, and run SQuARM-SGD for at least $T_{s_0}:=C_{s_0}\times\max\left\{\left(\frac{c_0^2 n^{(2+\epsilon)}}{(J^2+\bar{\sigma}^2)^{2}  \delta^4 }\right)^{\nicefrac{1}{\epsilon}}{,}  \frac{n}{(J^2+\bar{\sigma}^2)^2} \left(\frac{nG^2H^2}{\omega^2\delta^4 }  {+} \frac{\beta^4 \bar{\sigma}^2}{(1-\beta)^2}  \right)^2 \right\}$ iterations for non-convex objectives and for $T_{s_1} := C_{s_1} \times \max \left\{ \left(\frac{c_0^2 n^{2+\epsilon}}{\delta^4 ( \Vert \blx^{(0)} - \bx^{*} \Vert^2 + \bar{\sigma}^2 )^2}\right)^{\nicefrac{1}{\epsilon}} , \frac{n^3 H^4 G^2}{\delta^8 \omega^4 ( \Vert \blx^{(0)} - \bx^{*} \Vert^2 + \bar{\sigma}^2 )^2} \right. ,$  $ \left. \frac{n^{5} G^8 \beta^8 }{(1-\beta)^6 ( \Vert \blx^{(0)} - \bx^{*} \Vert^2 + \bar{\sigma}^2 )^4 }    \right\}$ for convex objectives with sufficiently large constants $C_{s_0}$ and $C_{s_1}$, respectively, then SQuARM-SGD converges at a rate of $\mathcal{O}  \left(\nicefrac{1}{\sqrt{nT}} \right)$.
Note that this is the convergence rate of distributed {\em vanilla} SGD with the same speed-up w.r.t.\ the number of nodes $n$ in both these settings. 
 Thus, we essentially converge at the same rate as that of vanilla SGD, while saving significantly in terms of total communicated bits; this can also be seen in our numerical results in Section \ref{experiments}.


}

\section{Preliminaries} \label{prelim}
{\allowdisplaybreaks
\noindent In this section, we first establish a matrix notation which would be used throughout the proofs. We then state SQuARM-SGD in matrix notation (which is equivalent to Algorithm~\ref{alg_dec_sgd_li}) and list important facts regarding our updates. We conclude this section with a brief discussion of technical challenges involved in the proofs. 	

\paragraph{Matrix notation.}	
Consider the set of parameters $\{ \bx_i^{(t)} \}_{i=1}^{n} $  at all nodes at timestep $t$ as well as the estimates of the parameters $\{ \hat{\bx}_i^{(t)} \}_{i=1}^{n} $. The matrix notation is given by:
\begin{align*} 
\bX^{(t)} &:= [\bx_1^{(t)}, \hdots, \bx_n^{(t)} ] \in \mathbb{R}^{d \times n } \\
\hat{\bX}^{(t)} &:= [\hat{\bx}_1^{(t)}, \hdots, \hat{\bx}_n^{(t)} ] \in \mathbb{R}^{d \times n } \\
\Bar{\bX}^{(t)} &:= [\Bar{\bx}^{(t)}, \hdots, \Bar{\bx}^{(t)} ] \in \mathbb{R}^{d \times n} \\
\bV^{(t)} &:= [\bv_1^{(t)}, \bv_2^{(t)}, \hdots, \bv_n^{(t)}] \in \mathbb{R}^{d \times n} \\
{\boldsymbol{\nabla F}}(\bX^{(t)}, \boldsymbol{\xi}^{(t)}) &{:=} [ {\nabla F_1}(\bx_1^{(t)}, \xi_1^{(t)}),{\hdots}, {\nabla F_n}(\bx_n^{(t)}, \xi_n^{(t)}) ]  \in \hspace{-0.1cm} \mathbb{R}^{d \times n}
\end{align*}
Here, $\nabla F_i(\bx_i^{(t)}, \xi_i^{(t)})$ denotes the stochastic gradient at node $i$ at timestep $t$ and the vector $  \Bar{\bx}^{(t)}  := \frac{1}{n}\sum_{i=1}^{n} \bx_i^{(t)} $  denotes the average of node parameters at time $t$. Let $\Gamma^{(t)} \subseteq [n] $ be the set of nodes that do not communicate at time $t$. We define $\mathbf{P}^{(t)} \in \mathbb{R}^{n \times n} $,  a diagonal matrix with $\mathbf{P}_{ii}^{(t)} = 0$ for $i \in \Gamma^{(t)} $ and $\mathbf{P}_{ii}^{(t)} = 1$ otherwise.\\ \\
\paragraph{SQuARM-SGD in matrix notation.}  \label{mat_not_li} 
	Consider  Algorithm \ref{alg_dec_sgd_li}  with synchronization indices given by the set 
	$\mathcal{I}_T = \{ 0, H , 2H \hdots, mH, \hdots  \} \subseteq [T]  $ for some constant $H \in \mathbb{N}$. Using the above notation, the sequence of parameters' updates from synchronization index $mH$ to $(m+1)H$ is:
	\begin{align}
	\bV ^{(t)}  &= \beta \bV ^{(t-1)} + {\boldsymbol{\nabla F}}(\bX^{(t)},\boldsymbol{\xi}^{(t)}) \label{mat_not_algo_1} \\
	\bX^{((m+\nicefrac{1}{2})H)} & {=} \bX^{I_{(t)}} {-} \sum_{t' = mH}^{(m+1)H-1} \eta (\beta\bV^{(t')} + {\boldsymbol{\nabla F}}(\bX^{(t')},\boldsymbol{\xi}^{(t')}))  \label{mat_not_algo_4} \\
	\hat{\bX}^{((m+1)H)} & {=} \hat{\bX}^{(mH)} {+} \C((\bX^{((m{+}\nicefrac{1}{2})H)} {-} \hat{\bX}^{(mH)})\mathbf{P}^{((m{+}1)H{-}1)} ) \label{mat_not_algo_3} \\
	\bX^{((m+1)H)} & = \bX^{((m+\nicefrac{1}{2})H)} + \gamma\hat{\bX}^{((m+1)}(\mathbf{W}-\mathbf{I}) \label{mat_not_algo_2}
	\end{align}
	where $\C( . )$ denotes the compression operator applied column-wise to the argument matrix and $\mathbf{I}$ is the identity matrix. 
Note that in the update rule for $\hat{\bX}^{((m+1)H)}$, we used 
{\sf (i)} the fact that $\mathbf{P}$ is a diagonal matrix and that $\C$ is applied column-wise to write $\C(\bX^{((m+\nicefrac{1}{2})H)} - \hat{\bX}^{(mH)})\mathbf{P}^{((m+1)H-1)}  = \C((\bX^{((m+\nicefrac{1}{2})H)} - \hat{\bX}^{(mH)})\mathbf{P}^{((m+1)H-1)} )$, and
{\sf (ii)} that $\hat{\bX}^{((m+1)H-1)} = \hat{\bX}^{(mH)}$, because $\hat{\bX}$ does not change in between the synchronization indices.

We now note some useful properties of the iterates in matrix notation which would be used throughout the paper:
\begin{enumerate}
	\item 	Since $\mathbf{W} \in [0,1]^{n \times n}$ is a doubly stochastic matrix, we have: $\mathbf{W}=\mathbf{W}^T \, , \mathbf{W} \mathbf{\mathbf{1}} = \mathbf{\mathbf{1}}$ and $\mathbf{\mathbf{1}}^T\mathbf{W} = \mathbf{\mathbf{1}}^T$ (where $\mathbf{1}$ is the all ones vector in $\mathbb{R}^n$). This also gives us: 
	\begin{align} \label{mean_prop}
	\Bar{\bX}^{(t)} := \bX^{(t)} \frac{1}{n} \mathbf{\mathbf{1}}\mathbf{\mathbf{1}}^T, \hspace{1cm} \Bar{\bX}^{(t)}\mathbf{W} = \Bar{\bX}^{(t)}
	\end{align}
	where the first expression follows from the definition of $\bar{\bX}^{(t)}$ and the second expression follows because $\mathbf{W} \frac{\mathbf{\mathbf{1}} \mathbf{\mathbf{1}}^T }{n} = \frac{\mathbf{\mathbf{1}} \mathbf{\mathbf{1}}^T }{n} \mathbf{W} = \frac{1}{n} \mathbf{\mathbf{1}}\mathbf{\mathbf{1}}^T $. \\
	\item The average of the iterates in Algorithm \ref{alg_dec_sgd_li}  follows  : 
	\begin{align} \label{mean_seq_iter}
	\Bar{\bX}^{(t+1)} & = \Bar{\bX}^{(t+ \frac{1}{2})} + \mathbf{\mathbf{1}}_{(t+1) \in \mathcal{I}_T} \left[\gamma\hat{\bX}^{{(t+1)}}( \mathbf{W}- \mathbf{I})\frac{1}{n} \mathbf{\mathbf{1}}\mathbf{\mathbf{1}}^T \right]  =   \Bar{\bX}^{(t+ \frac{1}{2})} 
	\end{align}
	where $\mathcal{I}_T $ denotes the set of synchronization indices of Algorithm \ref{alg_dec_sgd_li}. We use $( \mathbf{W}- \mathbf{I})\frac{1}{n}\mathbf{\mathbf{1}}\mathbf{\mathbf{1}}^T = \mathbf{W} \frac{\mathbf{\mathbf{1}} \mathbf{\mathbf{1}}^T }{n} -\frac{\mathbf{\mathbf{1}} \mathbf{\mathbf{1}}^T }{n} = \mathbf{0} $.	
\end{enumerate}


\begin{proposition}[Variance Reduction with Independent Samples]\label{prop:variance-reduction_relaxed}
Consider the variance bound \eqref{eq:variance} on the stochastic gradient for nodes. If ${\boldsymbol{\xi}^{(t)}} = \{ \xi_1^{(t)}, \xi_2^{(t)}, \hdots, \xi_n^{(t)} \}$ denotes the collection of independent stochastic samples for the nodes at any time-step $t$. Then we have:
	\begin{align}\label{eq:variance-reduction_relaxed}
		& \mathbb{E}_{\boldsymbol{\xi}^{(t)}} \verts{\frac{1}{n} \sum_{i=1}^{n} \nabla \big(F_i(\bx^{(t)}_i, \xi^{(t)}_i) - \nabla f_i(\bx^{(t)}_i)\big)}^2  \leq \frac{\sigma^2}{n} + \frac{M^2}{n^2}\sum_{i=1}^n\verts{\nabla f_i(\bx^{(t)}_i)}_2^2.
	\end{align}
\end{proposition}
\begin{proposition}\label{prop:bound_v}
For any $t$, $\bbE\verts{\bV^{(t)}}_F^2$ is bounded as follows:
\begin{align}\label{eq:bound_v}
(1{-}\beta)\bbE \Vert \bV^{(t)} \Vert_F^2 \leq \Lambda^{(t)} := \sum_{k=0}^{t} \beta^{t{-}k}\bbE \Vert\nabla \bF(\bX^{(k)}, \bxi^{(k)}) \Vert_F^2
\end{align}
\end{proposition}
We prove the above propositions in Appendix B.


}

\paragraph{Technical Challenges:}  
We focus on two major aspects of our work to compare with existing literature: {\sf (i)} Analysis of compressed decentralized training with triggered communication with mild assumptions. {\sf (ii)} Performing the resulting analysis by taking into account the momentum updates.

	The assumption on bounded second moment of stochastic gradients is commonly used in communication efficient decentralized training literature\cite{singh2019sparq,koloskova_decentralized_2019,koloskova_decentralized_2019-1,tang2019doublesqueeze}, and is also used to derive the result of Theorem~\ref{convergence_results} in our paper. However, this assumption can be quite strong for settings where the data distribution among clients is heterogeneous, as the gradient dissimilarity between clients can be bounded trivially using the second moment bound (see the note on comparison of assumptions in Remark~\ref{remark:assump-comparison} on page~\pageref{remark:assump-comparison}). In contrast, in Theorem~\ref{convergence_relaxed_assump}, we work with a much weaker set of assumptions (see Section \ref{subsec:weak_assump}) by not assuming any uniform bound on norm of stochastic gradients, and further allow both the gradient diversity and the variance of stochastic gradients to scale with the norm of gradients compared to existing works \cite{Momentum_linear-speedup19}. Performing the analyses with these relaxed assumptions is challenging, as it requires us to carefully consider the error due to quantization and local iterations per communication round and construct a recursion equation for it (see Lemmas~\ref{lem:similar-eq16},~\ref{lem:similar-eq17} on page~\pageref{lem:similar-eq16}) and then delicately handle the recursion to bound the error for any time index (see Lemma~\ref{lem:similar-lem14} on page~\pageref{lem:similar-lem14}).
We remark that the assumptions considered for Theorem~\ref{convergence_relaxed_assump} in our paper have appeared in literature before in \cite{KoloskovaDecenRelaxAssum20} to study decentralized optimization with only local iterations; our work is a significant extension of their results and analyses as we incorporate compression and momentum while achieving a convergence rate of $\mathcal{O} \left( \nicefrac{1}{\sqrt{T}}\right) $.

While momentum updates are almost always used in practice to empirically speedup the training process and to improve generalization performance, it has remained unclear whether convergence with linear speedup with number of nodes $n$ (as in the case of SGD without momentum \cite{lian2017asynchronous,basu_qsparse-local-sgd:_2019,singh2019sparq, KoloskovaDecenRelaxAssum20}) is still possible when using momentum. Recently, \cite{Momentum_linear-speedup19,distributed_momentum19} provided a positive answer to this question, where \cite{Momentum_linear-speedup19} studies local SGD with momentum in a decentralized setup, but {\em without} any compressed or event-triggered communication, and \cite{distributed_momentum19} studies compressed {\em distributed} SGD with momentum for non-convex objectives, but without local iterations or event-triggered communication. Our result in Theorem~\ref{convergence_results} is the first to provide convergence rates showing linear speedup with $n$ for compressed \emph{decentralized} optimization using momentum while incorporating local iteration and triggered communication in the analysis (see Section~\ref{subsec:strong_assump} for the convergence result and the assumptions made). To achieve this, our convergence proofs require the use of virtual sequences as defined in \eqref{virt_seq} on page \pageref{virt_seq}.
Proving convergence results using virtual sequences has been promising lately in stochastic optimization; see, for example, \cite{stich_sparsified_2018,alistarh_convergence_2018,stich-signsgd-19,basu_qsparse-local-sgd:_2019,Momentum_linear-speedup19,distributed_momentum19}.

We would like to emphasize that even without momentum and local iterations, analyzing compression in decentralized optimization \cite{koloskova_decentralized_2019-1,koloskova_decentralized_2019, singh2019sparq} (whose analysis does not require virtual sequences) is significantly more involved and requires different technical tools than analyzing compression in distributed optimization \cite{alistarh_convergence_2018,stich-signsgd-19}. One of the main reasons for this is as follows: In a decentralized setup, we need to separately show that nodes eventually reach to the same parameters (i.e., consensus happens), which happens trivially in a distributed setup, because in each iteration all worker nodes have the same parameters sent by the master node.
	On top of that, incorporating momentum updates (which has only been analyzed with compression in distributed setups so far) in decentralized setting is non-trivial and gives similar challenges.

As a consequence, it is not surprising that our proofs are fundamentally different and significantly more challenging from existing works, including \cite{distributed_momentum19, Momentum_linear-speedup19,koloskova_decentralized_2019-1,koloskova_decentralized_2019,singh2019sparq, KoloskovaDecenRelaxAssum20}, as we study momentum updates for decentralized setup with compression, local iterations and event-triggered communication to save on communication bits. 
Unlike \cite{distributed_momentum19}, we allow {\em heterogeneous} setting, where different nodes may have different datasets. Moreover, with all these, we achieve vanilla SGD like convergence rates for non-convex and convex objectives.

\section{Results with Relaxed Assumptions: Proof of Theorem \ref{convergence_relaxed_assump}}\label{sec:relaxed-assump-results}
\allowdisplaybreaks{

In order to prove Theorem \ref{convergence_relaxed_assump}, we define a virtual sequence ${\btx}_i^{(t)}$ for each node $i\in[n]$, as follows:
\begin{align} \label{virt_seq}
{\btx}_i^{(t)} = {\bx}_i^{(t)} - \frac{\eta \beta^2}{(1-\beta)} \bv_i^{(t-1)}; \qquad \btx_i^{(0)}:=\bx_i^{(0)}.
\end{align}
This remaining section is divided into seven subsections. 
In Section~\ref{subsec:sgd-update-virtual}, we derive an SGD like update rule for the virtual sequence.
In Section~\ref{subsec:proof-outline-relaxed-thm}, we provide a proof-outline of Theorem~\ref{convergence_relaxed_assump}. The remaining subsections are dedicated to prove the lemmas stated in the proof outline given in Section~\ref{subsec:proof-outline-relaxed-thm}.
\subsection{Deriving an SGD-Like Update Rule for the Virtual Sequene}\label{subsec:sgd-update-virtual}
In \eqref{virt_seq}, $\bx_i^{(t)}$ is the true local parameter at node $i$ at the $t$'th iteration, which is equal to (see line 16 of Algorithm \ref{alg_dec_sgd_li}):
\begin{align*}
\bx_i^{(t)} = \bx_i^{(t-\frac{1}{2})} + \mathbbm{1}_{\{t \in \mathcal{I}_T\}}\gamma \sum_{j=1}^n w_{ij} (\hat{\bx}_j^{(t)} - \hat{\bx}_i^{(t)}),
\end{align*}
where $\bx_i^{(t-\frac{1}{2})} = \bx_i ^{(t-1)} - \eta (\beta \bv_i^{(t-1)} + \nabla F_i(\bx_i^{(t-1)}, \xi_i^{(t-1)}))$ (line 5 in Algorithm \ref{alg_dec_sgd_li}).
Note that we changed the summation from $j\in\mathcal{N}_i$ to $j=1$ to $n$; this is because $w_{ij}=0$ whenever $j\notin\mathcal{N}_i$.

Let ${\blx}^{(t)} = \frac{1}{n} \sum_{i=1}^{n} \bx_{i}^{(t)}$ denote the average of the local iterates at time $t$.
Now we argue that ${\blx}^{(t)}={\blx}^{(t-\frac{1}{2})}$.
This trivially holds when $t\notin \I_T$. For the other case, i.e., $t\in \I_T$, this follows because
$\sum_{i=1}^n\sum_{j=1}^n w_{ij} (\hat{\bx}_j^{(t)} - \hat{\bx}_i^{(t)})=0$, which uses the fact that $W$ is a doubly stochastic matrix. 
Thus, we have 
\begin{equation}\label{eq:thm_cvx_interim0}
{\blx}^{(t)} =  {\blx}^{(t-1)} - \frac{\eta}{n}\sum_{i=1}^n \left(\beta \bv_i^{(t-1)} + \nabla F_i(\bx_i^{(t-1)}, \xi_i^{(t-1)})\right).
\end{equation}
Taking average over all the nodes in \eqref{virt_seq} and defining ${\btx}^{(t)} := \frac{1}{n}\sum_{i=1}^n{\btx}_i^{(t)}$, we get
\begin{align*}
{\btx}^{(t)} = {\blx}^{(t)} - \frac{\eta \beta^2}{(1-\beta)} \frac{1}{n} \sum_{i=1}^n \bv_i^{(t-1)}.
\end{align*}
We now note a recurrence relation for the sequence ${\btx}^{(t+1)}$:
\begin{align} \label{virt_seq_rec}
& {\btx}^{(t+1)}  = {\blx}^{(t+1)} - \frac{\eta \beta^2}{(1-\beta)} \frac{1}{n}\sum_{i=1}^n \bv_i^{(t)} \notag \\
& =  {\blx}^{(t)} {-} \frac{\eta}{n}\sum_{i=1}^n \left(\beta \bv_i^{(t)} {+} \nabla F_i(\bx_i^{(t)}, \xi_i^{(t)})\right) {-} \frac{\eta \beta^2}{(1-\beta)} \frac{1}{n}\sum_{i=1}^n \bv_i^{(t)} \notag \\
& =  {\blx}^{(t)} {-} \frac{\eta}{n} \sum_{i=1}^n \nabla F_i(\bx_i^{(t)},\xi_i^{(t)}) {-}  \left( \eta \beta {+} \frac{\eta \beta^2}{(1-\beta)}\right)  \frac{1}{n}\sum_{i=1}^n \bv_i^{(t)} \notag \\
& =  {\blx}^{(t)} {-} \frac{\eta}{n} \sum_{i=1}^n \nabla F_i(\bx_i^{(t)},\xi_i^{(t)}) {-} \frac{\eta \beta}{(1-\beta)} \frac{1}{n}\sum_{i=1}^n \beta\bv_i^{(t-1)} \notag \\
& \qquad  {-} \frac{\eta \beta}{(1-\beta)} \frac{1}{n}\sum_{i=1}^n \nabla F_i(\bx_i^{(t)},\xi_i^{(t)}) \notag \\
& =  {\btx}^{(t)} - \frac{\eta}{(1-\beta)}\frac{1}{n} \sum_{i=1}^n \nabla F_i(\bx_i^{(t)},\xi_i^{(t)})
\end{align}

\subsection{Proof Outline of Theorem~\ref{convergence_relaxed_assump}}\label{subsec:proof-outline-relaxed-thm}
The proof is divided into four lemmas. The first lemma (stated in Lemma~\ref{lem:similar-lem11}) derives the required convergence bound, however, the RHS depends on the deviation of local parameter vectors from the average parameter vector (i.e., $\Xi^{(t)}:=\sum_{i=1}^n\bbE \Vert\bx^{(t)}_i - \blx^{(t)}\Vert_2^2$), which we have to bound. The remaining three lemmas are dedicated to bounding this quantity. 

Note that bounding this in the {\em distributed} setup is not difficult, as at synchronization indices all parameters are the same because it is coordinated by a central server. This means that at any time index $t\in[T]$, there is always a time index $t-H\leq t'\leq t$ when $\bx_i^{(t')}$ for all $i\in[n]$ are the same, and we have a reference point no too far in the past. However, in the decentralized setup, there is no central server for coordinating the updates, and hence there is no reference point in the past when the local parameters are the same. Moreover, our assumptions are arguably the weakest in literature, and we also are working with compression and momentum updates. Thus, bounding $\Xi^{(t)}$ in our setup is highly non-trivial, and is one of the major technical contributions of our work.


\begin{lemma}\label{lem:similar-lem11}
Under the setting of Theorem~\ref{convergence_relaxed_assump}, when $\eta\leq\min\left\{\frac{2(1-\beta)^3}{9\beta^4},\frac{2(1-\beta)^2}{3\beta^2L}\sqrt{\frac{n}{M^2+n}},\frac{(1-\beta)^2}{6\beta^2LB}\sqrt{\frac{n}{2(M^2+n)}}\right\}$, we get:
\begin{align*}
&\frac{1}{T}\sum_{t=0}^{T-1}\bbE\verts{\nabla f(\blx^{(t)})}_2^2 \leq \frac{16\eta L}{(1-\beta)}\Big(\frac{\sigma^2+2(M^2+n)G^2}{n}\Big) + \frac{16(1{-}\beta)(f(\blx^{(0)}) {-} f^*)}{\eta T}    + \frac{64 L^2}{n}\frac{1}{T}\sum_{t=0}^{T-1}\sum_{i=1}^n\bbE\Vert\bx^{(t)}_i {-} \blx^{(t)}\Vert_2^2
\end{align*}
\end{lemma}
We provide a proof for Lemma~\ref{lem:similar-lem11} in Section~\ref{subsec:proof-similar-lem11}.

Consider any arbitrary $t\in[T]$. We bound $\Xi^{(t)}=\sum_{i=1}^n\bbE\verts{\bx^{(t)}_i - \blx^{(t)}}_2^2$ via another quantity $S^{(t)}$ defined as
$S^{(t)} := \Xi^{(t)} + \bbE\Vert \bX^{(t)} - \bhX^{((m+1)H)} \Vert_F^2,  \text{ where } m=\lfloor\frac{t}{H}\rfloor - 1.$
We derive two upper bounds on $S^{(t)}$ depending on the value of $t$. Note that in both the following lemmas, $m=\lfloor\frac{t}{H}\rfloor - 1$.
\begin{lemma}\label{lem:similar-eq16}
Consider any $t\in[T]$. Then for $m=\lfloor\frac{t}{H}\rfloor -1$, we have the following bound for $(m+1)H\leq t \leq (m+2)H-1$:
\begin{align*}
S^{(t)} & \leq \(1-\frac{\gamma\delta}{4}\)S^{(mH)} + 2c_1\eta^2H^2n\(2(M^2+1)G^2 + \sigma^2\)   + c_1\eta^2H\beta^2\sum_{t'=mH}^{t-1}\bbE \Vert \bV^{(t')} \Vert_F^2 \\
& \quad  + 2c_1\eta^2H(M^2{+}1)L^2 \hspace{-0.2cm} \sum_{t'=mH}^{t-1} S^{(t')}  + 2c_1\eta^2H(M^2+1)nB^2\sum_{t'=mH}^{t-1}\bbE\verts{\nabla f(\blx^{(t')})}_2^2, \, 
\end{align*}
where
 $c_1 \leq 2(1{+}\frac{\gamma\delta}{4})\(\frac{3}{\gamma\delta} {+} \frac{9\lambda^2}{\delta^2} {+} \frac{45\gamma\lambda^2}{\delta\omega} {+} \frac{104\gamma^2\lambda^2}{\omega^2} {+} \frac{4}{\omega} {-} 2\) {+} 4(1{+}\frac{4}{\gamma\delta})$.
\end{lemma}
We provide a proof of Lemma~\ref{lem:similar-eq16} in Section~\ref{subsec:proof-similar-eq16}.
\begin{lemma}\label{lem:similar-eq17}
For $mH\leq \hatt < (m+1)H$, we have:
\begin{align*}
S^{(\hatt)} & \leq \(1+\frac{\gamma\delta}{4}\)S^{(mH)} + 2c_1\eta^2H^2n\(2(M^2+1)G^2 + \sigma^2\) + c_1\eta^2H\beta^2\sum_{t'=mH}^{\hatt-1}\bbE \Vert\bV^{(t')}\Vert_F^2  \\
& \quad + 2c_1\eta^2H(M^2{+}1)L^2  \sum_{t'=mH}^{\hatt-1}S^{(t')} + 2c_1\eta^2H(M^2+1)nB^2\sum_{t'=mH}^{\hatt-1}\bbE\verts{\nabla f(\blx^{(t')})}_2^2, 
\end{align*} 
where $c_1$ is exactly the same as in Lemma~\ref{lem:similar-eq16}.
\end{lemma}
We prove Lemma~\ref{lem:similar-eq17} in Section~\ref{subsec:proof-similar-eq17}.
Using both these lemmas, we will be able to bound $\Xi^{(t)}$. We state the result in the following lemma, which we prove in Section~\ref{subsec:proof-similar-lem14}.
\begin{lemma}\label{lem:similar-lem14}
	Under setting of Theorem~\ref{convergence_relaxed_assump}, when $\eta \leq \min\left\{\sqrt{\frac{\gamma\delta}{512c_1H^2(M^2+1)L^2}},\sqrt{ \frac{  \alpha (1-\beta) }{128 D H (M^2+1)L^2}}\right\}$, we have:
	\begin{align*}
	\frac{1}{T}\sum_{t=0}^{T-1}&\sum_{i=1}^n\bbE\verts{\bx^{(t)}_i - \blx^{(t)}}_2^2 =	\frac{1}{T}\sum_{t=0}^{T-1} S^{(t)}  \leq 2\eta^2 J_1 + 2\eta^2 J_2 \frac{1}{T}\sum_{t=0}^{T-1}  \bbE \verts{ \nabla f (\blx^{(t)}) }^2,
	\end{align*}
	where $J_1= \left(\frac{32HA}{\alpha} + \left( \frac{32 D H}{\alpha} \right)  \left( \frac{2(M^2+1)nG^2 + n \sigma^2}{(1-\beta)}   \right)   \right)$ and $J_2=\left(\frac{32C  H}{\alpha}   + \left( \frac{32 D H }{\alpha} \right) \frac{2(M^2+1)nB^2}{(1-\beta)}\right)$, where $A=2c_1H^2n\(2(M^2+1)G^2 + \sigma^2\)$, $C=2c_1H(M^2+1)nB^2$, and $D=\frac{c_1H\beta^2}{(1-\beta)}$, and $c_1$ is exactly the same as in Lemma~\ref{lem:similar-eq16}.
\end{lemma}
Our proofs of Lemmas \ref{lem:similar-lem11}, \ref{lem:similar-eq16}, \ref{lem:similar-eq17}, and \ref{lem:similar-lem14} are adapted from the proofs of Lemmas 12, 13, and 14 in \cite{KoloskovaDecenRelaxAssum20}, however with significant changes, as we incorporate momentum updates and compression in the analysis.

Substituting the bounds from Lemma~\ref{lem:similar-lem14} into Lemma~\ref{lem:similar-lem11} and choosing $\eta=(1-\beta)\sqrt{\frac{n}{T}}$ (and running the algorithm for a sufficiently long time) completes the proof. Details with exact numbers are provided in Appendix C-F.

\subsection{Proof of Lemma~\ref{lem:similar-lem11}}\label{subsec:proof-similar-lem11}
Consider the quantity $\mathbb{E}_{\xi_{(t)}} [f(\btx^{(t+1)})] $ where expectation is taken w.r.t. the sampling at time $t$. From the recurrence relation of the virtual sequence \eqref{virt_seq_rec}, we have:
\begin{align}
&\mathbb{E}_{\xi_{(t)}} [f(\btx^{(t+1)})] = \mathbb{E}_{\xi_{(t)}} f \left(\btx^{(t)} { -} \frac{\eta}{n(1{-}\beta)}  \sum_{i=1}^{n} \nabla F_i(\bx^{(t)}_i, \xi^{(t)}_i) \right) \notag \\
&\stackrel{\text{(a)}}{\leq} f(\btx^{(t)}) \underbrace{- \lragnle{\nabla f(\btx^{(t)}) , \frac{\eta}{(1-\beta)} \frac{1}{n} \sum_{i=1}^{n}\nabla f_i(\bx^{(t)}_i)}}_{=:\ P_1} + \frac{L}{2} \frac{\eta^2}{(1-\beta)^2} \underbrace{\mathbb{E}_{\xi_{(t)}} \verts{\frac{1}{n} \sum_{i=1}^{n} \nabla F_i(\bx^{(t)}_i, \xi^{(t)}_i)}^2}_{=:\ P_2},\label{mom_noncvx_relaxed-interim6}
\end{align}
where (a) follows from the $L$-smoothness of $f$.
We show the following bounds on $P_1$ and $P_2$ in Appendix C-A.
\begin{align}
P_1 &\leq - \frac{\eta \Vert\nabla f(\btx^{(t)}) \Vert^2 }{2(1-\beta)}  + \frac{\eta L^2}{2n(1-\beta)}\sum_{i=1}^n \Vert\btx^{(t)} {-} \bx^{(t)}_i \Vert^2 \label{bound_P1} \\
P_2 &\leq \frac{\sigma^2}{n} + \frac{2(M^2+n)L^2}{n^2}\sum_{i=1}^n\verts{\bx^{(t)}_i - \btx^{(t)}}_2^2  + \frac{2(M^2+n)}{n}\big(G^2+B^2\verts{\nabla f(\btx^{(t)})}_2^2\big) \label{bound_P2}
\end{align}

Substituting the bounds \eqref{bound_P1} and \eqref{bound_P2} in \eqref{mom_noncvx_relaxed-interim6}, we get:
\begin{align}
\mathbb{E}_{\xi_{(t)}} [f(\btx^{(t+1)})] & \leq f(\btx^{(t)}) + \frac{\eta^2L}{2(1{-}\beta)^2}\Big(\frac{\sigma^2{+}2(M^2{+}n)G^2}{n}\Big)  + \Big(\frac{\eta L^2}{2n(1-\beta)}+ \frac{\eta^2L^3(M^2+n)}{n^2(1-\beta)^2}\Big) \sum_{i=1}^n\verts{\bx^{(t)}_i - \btx^{(t)}}_2^2 \notag \\
&\quad - \Big(\frac{\eta}{2(1-\beta)} - \frac{\eta^2L(M^2+n)B^2}{n(1-\beta)^2}\Big)\verts{\nabla f(\btx^{(t)})}_2^2. \label{mom_noncvx_interim7}
\end{align}
When $\eta\leq \frac{n(1-\beta)}{2L(M^2+n)}$, we get $\Big(\frac{\eta L^2}{2n(1-\beta)}+ \frac{\eta^2L^3(M^2+n)}{n^2(1-\beta)^2}\Big) \leq \frac{\eta L^2}{n(1-\beta)}$; and when $\eta \leq \frac{n(1-\beta)}{4LB^2(M^2+n)}$, we get $\Big(\frac{\eta}{2(1-\beta)} - \frac{\eta^2L(M^2+n)B^2}{n(1-\beta)^2}\Big)\geq\frac{\eta}{4(1-\beta)}$. Therefore, when $\eta\leq\min\{\frac{n(1-\beta)}{2L(M^2+n)},\frac{n(1-\beta)}{4LB^2(M^2+n)}\}$, we get
\begin{align}
&\mathbb{E}_{\xi_{(t)}} [f(\btx^{(t+1)})] \leq f(\btx^{(t)}) + \frac{\eta^2L}{2(1{-}\beta)^2}\Big(\frac{\sigma^2{+}2(M^2{+}n)G^2}{n}\Big) + \frac{\eta L^2}{n(1{-}\beta)}\sum_{i=1}^n\verts{\bx^{(t)}_i {-} \btx^{(t)}}_2^2 {-} \frac{\eta}{4(1{-}\beta)} \Vert\nabla f(\btx^{(t)})\Vert_2^2 \label{temp_mom_noncvx_interim8}
\end{align}
By Jensen's inequality and $L$-smoothness of $f$, we have $\verts{\nabla f(\blx^{(t)})}_2^2 \leq 2\verts{\nabla f(\blx^{(t)}) - \nabla f(\btx^{(t)})}_2^2 + 2\verts{\nabla f(\btx^{(t)})}_2^2 \leq 2L^2\verts{\blx^{(t)} - \btx^{(t)}}_2^2 + 2\verts{\nabla f(\btx^{(t)})}_2^2$. Rearranging this gives $\verts{\nabla f(\btx^{(t)})}_2^2 \geq \frac{1}{2}\verts{\nabla f(\blx^{(t)})}_2^2 - L^2\verts{\blx^{(t)} - \btx^{(t)}}_2^2$. Substituting this in \eqref{temp_mom_noncvx_interim8} and rearranging: 
\begin{align}
\frac{\eta}{8(1-\beta)}\verts{\nabla f(\blx^{(t)})}_2^2 & \leq f(\btx^{(t)}) - \mathbb{E}_{\xi_{(t)}} [f(\btx^{(t+1)})]  + \frac{\eta^2L}{2(1-\beta)^2}\Big(\frac{\sigma^2{+}2(M^2{+}n)G^2}{n}\Big) + \frac{\eta L^2}{4(1-\beta)} \Vert\blx^{(t)} - \btx^{(t)}\Vert_2^2 \notag \\
&\quad + \frac{\eta L^2}{n(1-\beta)}\sum_{i=1}^n\verts{\bx^{(t)}_i - \btx^{(t)}}_2^2  \notag \\
& \leq f(\btx^{(t)}) {-} \mathbb{E}_{\xi_{(t)}} [f(\btx^{(t+1)})] {+} \frac{\eta^2L}{2(1-\beta)^2}\frac{\sigma^2{+}2(M^2{+}n)G^2}{n}  + \frac{2\eta L^2}{n(1-\beta)}\sum_{i=1}^n\Vert\bx^{(t)}_i - \blx^{(t)}\Vert_2^2 \notag \\
& \quad + \frac{9\eta L^2}{4(1-\beta)}\Vert\blx^{(t)} - \btx^{(t)}\Vert_2^2 \label{mom_noncvx_interim91}
\end{align}
Now we bound $\verts{\blx^{(t)} - \btx^{(t)}}_2^2$ in the following lemma, which we prove in Appendix C-A in supplementary material: 
\begin{lemma}\label{glob_virt_bound}
Consider the deviation of the global average parameter $\blx^{(t)}$  and the virtual sequence $\btx^{(t)}$ defined in \eqref{virt_seq} for constant stepsize $\eta$. Then at any time step $t$, we have:
\begin{align*}
\Vert \blx^{(t)} {-}  \btx^{(t)}\Vert^2 & \leq \frac{\beta^4 \eta^2}{(1{-}\beta)^3}  \sum_{\tau=0}^{t-1} \beta^{t-\tau-1} \Vert \frac{1}{n} \sum_{i=1}^{n} \nabla F_i (\bx^{(\tau)}_i , \xi^{(\tau)}_i ) \Vert^2
\end{align*}
\end{lemma}
Substituting the bound from Lemma~\ref{glob_virt_bound} into \eqref{mom_noncvx_interim91} and then taking the expectation w.r.t.\ the entire past and average over $t=0$ to $t=T-1$ gives
\begin{align}\label{mom_noncvx_interim910}
\frac{\eta}{8T(1-\beta)}\sum_{t=0}^{T-1}\bbE\verts{\nabla f(\blx^{(t)})}_2^2 & \leq \frac{\eta^2L}{2(1{-}\beta)^2}\frac{\sigma^2{+}2(M^2{+}n)G^2}{n}  \notag \\
&  +  \frac{1}{T}\bbE[f(\btx^{(0)}) - f(\btx^{(T)})] {+} \sum_{t=0}^{T-1}  \frac{2\eta L^2}{Tn(1{-}\beta)}\sum_{i=1}^n\bbE \Vert\bx^{(t)}_i {-} \blx^{(t)}\Vert_2^2  \notag \\
& + \frac{9\eta^3\beta^4 L^2}{4T(1{-}\beta)^4}\sum_{t=0}^{T-1} \sum_{\tau=0}^{t-1} {\beta^{t-\tau-1}}\bbE \Vert\frac{1}{n} \sum_{i=1}^{n} \nabla F_i (\bx^{(\tau)}_i , \xi^{(\tau)}_i ) \Vert^2 
\end{align}
In the following lemma (which we prove in Appendix C-A) we bound the last term of \eqref{mom_noncvx_interim910}.
\begin{lemma}\label{lem:bound-decaying-grads}
	Under setting of Theorem~\ref{convergence_relaxed_assump}, it follows that:
\begin{align}\label{eq:bound-decaying-grads}
\frac{1}{T}\sum_{t=0}^{T-1} \sum_{\tau=0}^{t-1} \left[\beta^{t-\tau-1}\bbE\verts{\frac{1}{n} \sum_{i=1}^{n} \nabla F_i (\bx^{(\tau)}_i , \xi^{(\tau)}_i ) }^2\right] & \leq \frac{\sigma^2}{n(1{-}\beta)}  + \frac{2(M^2+n)}{n(1-\beta)} \left( G^2 + \frac{L^2}{T}\sum_{\tau=0}^{T-2}\sum_{i=1}^n\bbE\verts{\bx^{(\tau)}_i - \blx^{(\tau)}}_2^2  \right)   \notag \\
& \quad + \frac{2(M^2+n)B^2}{n(1-\beta)}\frac{1}{T}\sum_{\tau=0}^{T-2}\bbE\|\nabla f(\blx^{(\tau)})\|_2^2.
\end{align}
\end{lemma}
Substituting the bound from \eqref{eq:bound-decaying-grads} into \eqref{mom_noncvx_interim910} and noting that $\btx^{(0)}=\blx^{(0)}$ and $f(\btx^{(T)})\geq f^*$, where $f^*=f(\bx^*)$, we get:
\begin{align}
& \frac{\eta}{8(1-\beta)}\frac{1}{T}\sum_{t=0}^{T-1}\bbE\verts{\nabla f(\blx^{(t)})}_2^2 \notag \\
& \leq \frac{f(\blx^{(0)}) - f^*}{T} + \frac{\eta^2\sigma^2L}{2n(1-\beta)^2} + \frac{\eta^2L(M^2+n)G^2}{n(1-\beta)^2} + \frac{2\eta L^2}{n(1-\beta)} \frac{1}{T}\sum_{t=0}^{T-1}\sum_{i=1}^n\bbE\verts{\bx^{(t)}_i - \blx^{(t)}}_2^2 + \frac{9\eta^3\beta^4 L^2 \sigma^2}{4n(1-\beta)^5} \notag \\
& \quad   +  \frac{9\eta^3\beta^4 L^4(M^2{+}n)}{2(1-\beta)^5n^2}\frac{1}{T}\sum_{t=0}^{T-1}\sum_{i=1}^n\bbE\verts{\bx^{(t)}_i {-} \blx^{(t)}}_2^2  + \frac{9\eta^3\beta^4 L^2(M^2{+}n)}{2n(1-\beta)^5} \left( G^2 + \frac{B^2}{T} \sum_{\tau=0}^{T-1}\bbE\|\nabla f(\blx^{(\tau)})\|_2^2 \right)     \notag \\
&= \frac{f(\blx^{(0)}) {-} f^*}{T} {+} \frac{\eta^2L}{2(1{-}\beta)^2}\Big(\frac{\sigma^2{+}2(M^2{+}n)G^2}{n}\Big)\Big(1+\frac{9\eta\beta^4}{2(1{-}\beta)^3}\Big) +  \frac{9\eta^3\beta^4 L^2(M^2+n)B^2}{2n(1-\beta)^5} \frac{1}{T}\sum_{\tau=0}^{T-1}\bbE\|\nabla f(\blx^{(\tau)})\|_2^2   \notag \\
&\quad  + \left(\frac{2\eta L^2}{n(1-\beta)} + \frac{9\eta^3\beta^4 L^4(M^2+n)}{2n^2(1-\beta)^5} \right)\frac{1}{T}\sum_{t=0}^{T-1}\sum_{i=1}^n\bbE \Vert\bx^{(t)}_i - \blx^{(t)} \Vert_2^2 \label{mom_noncvx_interim11}
\end{align}
Note that 
{\sf (i)} when $\eta\leq\frac{2(1-\beta)^3}{9\beta^4}$, we have $\Big(1+\frac{9\eta\beta^4}{2(1-\beta)^3}\Big)\leq2$; 
{\sf (ii)} when $\eta\leq\frac{2(1-\beta)^2}{3\beta^2L}\sqrt{\frac{n}{M^2+n}}$, we have $\left(\frac{2\eta L^2}{n(1-\beta)} + \frac{9\eta^3\beta^4 L^2}{4(1-\beta)^4}\frac{2(M^2+n)L^2}{n^2(1-\beta)} \right) \leq \frac{4\eta L^2}{n(1-\beta)}$; and
{\sf (iii)} when $\eta\leq\frac{(1-\beta)^2}{6\beta^2LB}\sqrt{\frac{n}{2(M^2+n)}}$, we have $\frac{9\eta^3\beta^4 L^2}{4(1-\beta)^4}\frac{2(M^2+n)B^2}{n(1-\beta)}\leq \frac{\eta}{16(1-\beta)}$.
So, when $\eta\leq\min\{\frac{2(1-\beta)^3}{9\beta^4},\frac{2(1-\beta)^2}{3\beta^2L}\sqrt{\frac{n}{M^2+n}},\frac{(1-\beta)^2}{6\beta^2LB}\sqrt{\frac{n}{2(M^2+n)}}\}$, we get:
\begin{align}
\frac{\eta}{8(1-\beta)}\frac{1}{T}\sum_{t=0}^{T-1}\bbE\verts{\nabla f(\blx^{(t)})}_2^2 & \leq \frac{f(\blx^{(0)}) - f^*}{T} + \frac{\eta^2\sigma^2L}{n(1{-}\beta)^2}  + \frac{2(M^2+n)G^2\eta^2L}{n(1{-}\beta)^2} + \frac{\eta}{16(1{-}\beta)}\frac{1}{T}\sum_{\tau=0}^{T-1}\bbE\|\nabla f(\blx^{(\tau)})\|_2^2 \notag \\
& \qquad + \frac{4\eta L^2}{n(1-\beta)}\frac{1}{T}\sum_{t=0}^{T-1}\sum_{i=1}^n\bbE\verts{\bx^{(t)}_i - \blx^{(t)}}_2^2  \label{mom_noncvx_interim12}
\end{align}
Taking $\frac{\eta}{16(1-\beta)}\frac{1}{T}\sum_{\tau=0}^{T-1}\bbE\|\nabla f(\blx^{(\tau)})\|_2^2$ to the LHS and multiplying both sides by $\frac{16(1-\beta)}{\eta}$ gives
\begin{align}
\frac{1}{T}\sum_{t=0}^{T-1}\bbE\verts{\nabla f(\blx^{(t)})}_2^2 & \leq \frac{16(1-\beta)(f(\blx^{(0)}) - f^*)}{\eta T} + \frac{16\eta L}{(1-\beta)}\Big(\frac{\sigma^2+2(M^2+n)G^2}{n}\Big)  \notag \\
& \qquad + \frac{64 L^2}{nT}\sum_{t=0}^{T-1}\sum_{i=1}^n\bbE \Vert\bx^{(t)}_i - \blx^{(t)}\Vert_2^2 \label{mom_noncvx_interim13}
\end{align}
\subsection{Useful Lemmas}\label{subsec:useful-lemmas}
\noindent The following two lemmas (which we prove in Appendix C-B) will be useful for proving Lemma~\ref{lem:similar-eq16} and Lemma~\ref{lem:similar-eq17}.

\begin{lemma}\label{lem:cvx_e1_relaxed}
Under the setting of Theorem \ref{convergence_relaxed_assump}, for any $m\in\bbN$: 
\begin{align}
\bbE\verts{\bX^{((m+1)H)} - \blX^{((m+1)H)}}_F^2 & \leq  a_1 \bbE\verts{\bX^{(mH)} - \blX^{(mH)}}_F^2 + a_2\bbE\verts{\bX^{(mH)} - \bhX^{(mH)}}_F^2 \notag \\
& \quad + a_3\eta^2\bbE \verts{ \textstyle \sum_{t'=mH}^{(m+1)H-1}\beta \bV^{(t')}+\nabla \bF(\bX^{(t')},\bxi^{(t')})}_F^2, \label{eq:cvx_e1_relaxed}
\end{align}
where $a_1=(1+\alpha_5^{-1})R_1,$ $a_2=(1+\alpha_5^{-1}) R_2 (1+\tau_1)(1-\omega)(1+\tau_2),$ and $a_3=(R_1+R_2)(1+\alpha_5)+(1+\alpha_5^{-1}) R_2\((1+\tau_1^{-1}) + (1+\tau_1)(1-\omega)(1+\tau_2^{-1})\).$
Here, $\tau_1,\tau_2,\alpha_5>0$ are arbitrary numbers, $R_1=(1+\alpha_1)(1-\gamma \delta)^2, R_2=(1+\alpha_1^{-1}) \gamma^2 \lambda^2 $, $\alpha_1>0$, $\delta$ is the spectral gap, $H$ is synchronization gap, $\gamma$ is consensus step-size, $\lambda := \verts{\mathbf{W}-\mathbf{I}}_2$ where $\mathbf{W}$ is a doubly stochastic mixing matrix. 
\end{lemma}
\begin{lemma}\label{lem:cvx_e2_relaxed}
Under the setting of Theorem \ref{convergence_relaxed_assump}, for any $m\in\bbN$: 
\begin{align}\label{eq:cvx_e2_relaxed}
\bbE\Vert\bX^{((m+1)H)} - \bhX^{((m+1)H)}\Vert_F^2 & \leq b_1\bbE\Vert\bX^{(mH)} - \blX^{(mH)}\Vert_F^2 + b_2\bbE \Vert\bX^{(mH)} - \bhX^{(mH)}\Vert_F^2 \notag \\
& \quad + b_3\eta^2\bbE\verts{ \textstyle\sum_{t'=mH}^{(m+1)H-1} \beta \bV^{(t')}+\nabla \bF(\bX^{(t')},\bxi^{(t')})}_F^2,
\end{align}
where $b_1 = (1+\tau_3^{-1})\gamma^2\lambda^2(1+\tau_5)(1+\tau_6)$,
$b_2 = (1+\tau_3)(1-\omega)(1+\tau_4) + (1+\tau_3^{-1})\gamma^2\lambda^2(1+\tau_5)(1+\tau_6^{-1})(1+\tau_7)(1-\omega)(1+\tau_8)$,
$b_3 = (1+\tau_3)(1-\omega)(1+\tau_4^{-1}) + (1+\tau_3^{-1})\gamma^2\lambda^2(1+\tau_5)(1+\tau_6^{-1})\((1+\tau_7^{-1}) + (1+\tau_7)(1-\omega)(1+\tau_8^{-1})\) + (1+\tau_3^{-1})\gamma^2\lambda^2(1+\tau_5^{-1})$. Here, $\tau_3,\tau_4,\tau_5,\tau_6,\tau_7,\tau_8>0$ are free parameters.
\end{lemma}

\subsection{Proof of Lemma~\ref{lem:similar-eq16}}\label{subsec:proof-similar-eq16}
\noindent For any $t\in[T]$, define $m\in\lfloor\frac{t}{H}\rfloor-1$. This implies that $(m+1)H\leq t < (m+2)H$. Now we note that:
\begin{align}
\Xi^{(t)} & := \bbE\verts{\bX^{(t)} - \blX^{(t)}}_F^2 \notag \\
&   = \bbE\verts{\bX^{(t)} - \blX^{((m+1)H)} - \(\blX^{(t)} - \blX^{((m+1)H)}\)}_F^2 \notag \\
&\stackrel{\text{(a)}}{\leq} \bbE\verts{\bX^{(t)} - \blX^{((m+1)H)}}_F^2 \label{equi16-interim0} \\
&\leq (1+\nu_1)\bbE\verts{\bX^{((m+1)H)} - \blX^{((m+1)H)}}_F^2  + (1{+}\nu_1^{-1})\eta^2\bbE\verts{\sum_{t'=(m+1)H}^{t-1} \(\beta \bV^{(t')}{+}\nabla \bF(\bX^{(t')},\bxi^{(t')})\)}_F^2 \notag \\
&\stackrel{\text{(b)}}{\leq} (1{+}\nu_1)(a_1 \Xi^{(mH)} + a_2\bbE \Vert\bX^{(mH)} {-} \bhX^{(mH)}\Vert_F^2)  + (1{+}\nu_1)a_3\eta^2\bbE\verts{\sum_{t'=mH}^{(m+1)H-1}\beta \bV^{(t')}{+}\nabla \bF(\bX^{(t')},\bxi^{(t')})}_F^2 \notag \\
& \qquad + (1{+}\nu_1^{-1})\eta^2\bbE\verts{\sum_{t'=(m+1)H}^{t-1} \beta \bV^{(t')}{+}\nabla \bF(\bX^{(t')},\bxi^{(t')})}_F^2 \notag \\
&\leq (1{+}\nu_1)(a_1 \Xi^{(mH)} + a_2\bbE \Vert\bX^{(mH)} {-} \bhX^{(mH)}\Vert_F^2) + (1+\nu_1)a_3 \eta^2H\sum_{t'=mH}^{t-1}\bbE \Vert\beta \bV^{(t')}{+}\nabla \bF(\bX^{(t')},\bxi^{(t')}) \Vert_F^2 \notag \\
& \qquad +  (1{+}\nu_1^{-1})\eta^2H\sum_{t'=mH}^{t-1}\bbE \Vert\beta \bV^{(t')}{+}\nabla \bF(\bX^{(t')},\bxi^{(t')}) \Vert_F^2 \notag \\
&\leq (1+\nu_1)a_1 \Xi^{(mH)} + (1+\nu_1)a_2\bbE\Vert \bX^{(mH)} - \bhX^{(mH)} \Vert_F^2   \notag \\
& \qquad + 2\((1+\nu_1)a_3 + (1+\nu_1^{-1})\)\eta^2H\sum_{t'=mH}^{t-1}\bbE \Vert\nabla \bF(\bX^{(t')},\bxi^{(t')}) \Vert_F^2 \notag \\
& \qquad + 2\((1{+}\nu_1)a_3 + (1+\nu_1^{-1})\)\eta^2H\sum_{t'=mH}^{t-1} \beta^2\bbE\verts{\bV^{(t')}}_F^2 
 \label{equi-16-interim1}
\end{align}
Here, (a) follows from the inequality: $\frac{1}{n}\sum_{i=1}^n\verts{\ba_i-\frac{1}{n}\sum_{i=1}^n\ba_i}_2^2 \leq \frac{1}{n}\sum_{i=1}^n\verts{\ba_i}_2^2$ and (b) follows from \eqref{eq:cvx_e1_relaxed} (in Lemma~\ref{lem:cvx_e1_relaxed}). The coefficients $a_1,a_2,a_3$ in the RHS of (b) are defined in Lemma~\ref{lem:cvx_e1_relaxed}.
\begin{proposition}\label{prop:mom-update-norm}
For any $t'$, we have:
\begin{align}\label{eq:mom-update-norm}
&\bbE\verts{\nabla \bF(\bX^{(t')},\bxi^{(t')})}_F^2 \leq 2(M^2+1) (L^2\Xi^{(t')} + nG^2)+ 2(M^2+1)nB^2\bbE\verts{\nabla f(\blx^{(t')})}_2^2 + n\sigma^2
\end{align}
\end{proposition}
\noindent Substituting \eqref{eq:mom-update-norm} into \eqref{equi-16-interim1}, for $(m+1)H\leq t < (m+2)H$:
\begin{align}
\Xi^{(t)} & \leq (1+\nu_1)\(a_1 \Xi^{(mH)} + a_2\bbE\verts{\bX^{(mH)} - \bhX^{(mH)}}_F^2\) \notag \\
& \quad + 2c_2\eta^2H^2n\(2(M^2{+}1)G^2 {+} \sigma^2\)   {+} c_2\eta^2H\beta^2\sum_{t'=mH}^{t-1}\bbE\Vert\bV^{(t')} \Vert_F^2  \notag \\
& \quad + 2c_2\eta^2H(M^2{+}1)\sum_{t'=mH}^{t-1} L^2\Xi^{(t')} {+} nB^2\bbE \Vert\nabla f(\blx^{(t')}) \Vert^2   \label{equi-16-interim2}
\end{align}
where 
$c_2=2\((1+\nu_1)a_3 + (1+\nu_1^{-1})\)$.
For any $j\in[T]$ and $m'=\lfloor\frac{j}{H}\rfloor - 1$, define 
\begin{align}\label{Sj-defn}
S^{(j)} := \Xi^{(j)} + \bbE\verts{\bX^{(j)} - \bhX^{((m'+1)H)}}_F^2.
\end{align}
By definition, we have $S^{(mH)} = \Xi^{(mH)} + \bbE\verts{\bX^{(mH)} - \bhX^{(mH)}}_F^2$ and also that $\Xi^{(t')}\leq S^{(t')}$ for any $t'$. Using these in \eqref{equi-16-interim2}, we get
\begin{align}
\Xi^{(t)} & \leq (1+\nu_1)\(a_1 \Xi^{(mH)} + a_2\bbE\verts{\bX^{(mH)} - \bhX^{(mH)}}_F^2\) \notag \\
& \quad + 2c_2\eta^2H^2n\(2(M^2{+}1)G^2{+} \sigma^2\) {+} c_2\eta^2H\beta^2\sum_{t'=mH}^{t-1}\bbE \Vert\bV^{(t')}\Vert_F^2 \notag \\
& \quad + 2c_2\eta^2H(M^2{+}1)\sum_{t'=mH}^{t-1} L^2 S^{(t')}  {+} nB^2\bbE \Vert\nabla f(\blx^{(t')})\Vert_2^2   \label{equi-16-interim3}
\end{align}
Our aim is to get an upper-bound on $S^{(t)}$, which is defined in \eqref{Sj-defn} as $S^{(t)} = \Xi^{(t)} + \bbE\verts{\bX^{(t)} - \bhX^{(\lfloor t/H\rfloor H)}}_F^2$. However, in \eqref{equi-16-interim3}, we have only derived  an upper-bound on $\Xi^{(t)}$ in terms of $S^{(t')}$ for $t'<t$. So,, we need to derive a similar upper-bound on the other term $\bbE\verts{\bX^{(t)} - \bhX^{(\lfloor t/H\rfloor H)}}_F^2$, and then we will add both the upper-bounds to get an upper-bound on $S^{(t)}$. In the following, we derive an upper bound on $\bbE \Vert\bX^{(t)} - \bhX^{(\lfloor t/H\rfloor H)}\Vert_F^2$. 
Let $m=\lfloor\frac{t}{H}\rfloor - 1$, we have:
\begin{align}
& \bbE\verts{\bX^{(t)} - \bhX^{((m+1)H)}}_F^2 \notag \\
&  = \bbE \left\Vert \bX^{((m+1)H)} - \bhX^{((m+1)H)} 	- \eta\sum_{t'=(m+1)H}^{t-1} \(\beta \bV^{(t')}+\nabla \bF(\bX^{(t')},\bxi^{(t')})\)  \right\Vert_F^2 \notag \\
&\leq (1+\nu_1)\bbE\verts{\bX^{((m+1)H)} - \bhX^{((m+1)H)}}_F^2 + (1{+}\nu_1^{-1})\eta^2\bbE\verts{\sum_{t'=(m+1)H}^{t-1} \(\beta \bV^{(t')}{+}\nabla \bF(\bX^{(t')},\bxi^{(t')})\)}_F^2 \notag \\
&\stackrel{\text{(a)}}{\leq} (1{+}\nu_1)(b_1\Xi^{(mH)}   + b_2\bbE \Vert\bX^{(mH)} {-} \bhX^{(mH)} \Vert_F^2 )  + (1{+}\nu_1)b_3\eta^2\bbE\verts{\sum_{t'=mH}^{(m+1)H-1} \(\beta \bV^{(t')}{+}\nabla \bF(\bX^{(t')},\bxi^{(t')})\)}_F^2 \notag \\
&\quad + (1{+}\nu_1^{-1})\eta^2\bbE\verts{\sum_{t'=(m+1)H}^{t-1} \(\beta \bV^{(t')}{+}\nabla \bF(\bX^{(t')},\bxi^{(t')})\)}_F^2 \notag \\
&\leq (1+\nu_1)\(b_1\Xi^{(mH)} + b_2\bbE\verts{\bX^{(mH)} - \bhX^{(mH)}}_F^2\) + 2\((1+\nu_1)b_3 + (1+\nu_1^{-1})\)\eta^2H\sum_{t'=mH}^{t-1}\beta^2\bbE \Vert\bV^{(t')}\Vert_F^2 \notag \\
& \quad + 2\((1{+}\nu_1)b_3 {+} (1{+}\nu_1^{-1})\)\eta^2H \hspace{-0.2cm} \sum_{t'=mH}^{t-1} \hspace{-0.2cm} \bbE \Vert\nabla \bF(\bX^{(t')},\bxi^{(t')}) \Vert_F^2 \notag \\
&\stackrel{\text{(b)}}{\leq} (1+\nu_1)\(b_1\Xi^{(mH)} + b_2\bbE\verts{\bX^{(mH)} - \bhX^{(mH)}}_F^2\) + 2c_4\eta^2H^2n\(2(M^2{+}1)G^2 + \sigma^2\)  \notag \\
& \quad + c_4\eta^2H\beta^2  \sum_{t'=mH}^{t-1} \bbE \Vert\bV^{(t')} \Vert_F^2 + 2c_4\eta^2H(M^2{+}1)\sum_{t'=mH}^{t-1} L^2\Xi^{(t')} {+} nB^2 \bbE \Vert \nabla f(\blx^{(t')}) \Vert_2^2  \label{equi-16-interim4}
\end{align}
where (a) follows from \eqref{eq:cvx_e2_relaxed} in Lemma~\ref{lem:cvx_e2_relaxed} and the coefficients $b_1,b_2,b_3$ in the RHS of (a) are defined in Lemma~\ref{lem:cvx_e2_relaxed}, and (b) follows from substituting the bound from \eqref{eq:mom-update-norm} (in Proposition~\ref{prop:mom-update-norm}). In the RHS of (b), 
$c_4=2\((1+\nu_1)b_3 + (1+\nu_1^{-1})\)$.

\noindent Adding \eqref{equi-16-interim3}, \eqref{equi-16-interim4} for  $S^{(t)}=\Xi^{(t)}+\bbE \Vert\bX^{(t)} - \bhX^{((m+1)H)} \Vert_F^2$:
\begin{align}
S^{(t)} & \leq (1+\nu_1)\max\{a_1+b_1,a_2+b_2\}S^{(mH)} + 2c_1\eta^2H^2\varGamma + c_1\eta^2H\beta^2\sum_{t'=mH}^{t-1}\bbE \Vert\bV^{(t')}\Vert_F^2  \notag \\
& \quad + 2c_1\eta^2H(M^2{+}1)L^2\sum_{t'=mH}^{t-1} S^{(t')} + 2c_1\eta^2H(M^2+1)nB^2 \textstyle \sum_{t'=mH}^{t-1}\bbE\verts{\nabla f(\blx^{(t')})}_2^2 \label{equi-16-interim5}
\end{align}
where $\varGamma=n\(2(M^2+1)G^2 + \sigma^2\)$ and $c_1=c_2+c_4$ with 
$c_2=2\((1+\nu_1)a_3 + (1+\nu_1^{-1})\)$ 
and \\ $c_4=2\((1+\nu_1)b_3 + (1+\nu_1^{-1})\)$. Here, $\nu_1>0$ is a free coefficient, and $a_1,a_2,a_3$ and $b_1,b_2,b_3$ are defined in Lemma~\ref{lem:cvx_e1_relaxed} and Lemma~\ref{lem:cvx_e2_relaxed}, respectively. We will set the free variables such that the coefficients of $S^{(t')}$ for any $t'=mH,...,t-1$ on the RHS become strictly less than one.

In Appendix C-C, we show that if we set the free parameters to be the following:
\begin{align*}
&\tau_i = \frac{\omega}{4}, \text{ for } i=1,2,3,4,5,7,8; \quad \tau_6 = \frac{4}{\omega};  \quad \nu_1=\frac{\gamma^*\delta}{4}; \\
&\alpha_1 = \frac{\gamma\delta}{2}; \quad \alpha_5^{-1} = \frac{\gamma\delta}{2}; \quad \gamma = \frac{2\delta\omega^3}{(128\lambda^2 + 24\lambda^2\omega^2 + 4\delta^2\omega^2)};
\end{align*}
Then we get 
\begin{align}
&(1{+}\nu_1)\max\{a_1{+}b_1,a_2{+}b_2\} \leq 1{-}\frac{\gamma^*\delta}{4} \leq 1{-} \frac{\delta^2\omega^3}{1224}, \label{final-bound_coeff_St} \\
&c_1 \leq 2(1+\frac{\gamma\delta}{4})\(\frac{3}{\gamma\delta} + \frac{9\lambda^2}{\delta^2} + \frac{45\gamma\lambda^2}{\delta\omega} + \frac{104\gamma^2\lambda^2}{\omega^2} + \frac{4}{\omega} - 2\) + 4(1+\frac{4}{\gamma\delta}). \label{final-bound-c2+c4}
\end{align}
Putting these bounds back into \eqref{equi-16-interim5}, we get the following upper bound for $(m+1)H\leq t \leq (m+2)H-1$:
\begin{align}
 S^{(t)} & \leq \(1-\frac{\gamma\delta}{4}\)S^{(mH)} + 2c_1\eta^2H^2n\(2(M^2+1)G^2 + \sigma^2\) \notag \\
& \quad + c_1\eta^2H\beta^2\sum_{t'=mH}^{t-1}\bbE \Vert\bV^{(t')} \Vert_F^2 + 2c_1\eta^2H(M^2{+}1)L^2\sum_{t'=mH}^{t-1}S^{(t')} \notag \\
& \quad + 2c_1\eta^2H(M^2+1)nB^2\sum_{t'=mH}^{t-1}\bbE\verts{\nabla f(\blx^{(t')})}_2^2. \label{temp_equi-16-final-bound}
\end{align} 
\subsection{Proof of Lemma~\ref{lem:similar-eq17}}\label{subsec:proof-similar-eq17}
For any fixed $t\in[T]$ and the corresponding $m\in\lfloor\frac{t}{H}\rfloor-1$, in Section~\ref{subsec:proof-similar-eq16}, we derived an upper-bound on $S^{(\hatt)}$ all $\hatt\in[T]$ such that $(m+1)H\leq \hatt < (m+2)H$ (note that $t$ and $\hatt$ will give exactly the same terms in Section~\ref{subsec:proof-similar-eq16}, so we just kept $t$ everywhere). In this section, we consider the case when $mH\leq \hatt < (m+1)H$.
\begin{align}
& \Xi^{(\hatt)} \stackrel{\text{(a)}}{\leq} \bbE\verts{\bX^{(\hatt)} - \blX^{(mH)}}_F^2 \label{equi-17-interim0} \\
&\leq (1+\nu_3)\bbE\verts{\bX^{(mH)} - \blX^{(mH)}}_F^2  + (1+\nu_3^{-1})\eta^2\bbE\verts{\sum_{t'=mH}^{\hatt-1}\(\beta \bV^{(t')}+\nabla \bF(\bX^{(t')},\bxi^{(t')})\)}_F^2 \notag \\
&\stackrel{\text{(b)}}{\leq} (1+\nu_3)\Xi^{(mH)} + 2(1+\nu_3^{-1})\eta^2H\beta^2\sum_{t'=mH}^{\hatt-1}\bbE\verts{\bV^{(t')}}_F^2  + 2(1+\nu_3^{-1})\eta^2H \sum_{t'=mH}^{\hatt-1}\bbE\verts{\nabla \bF(\bX^{(t')},\bxi^{(t')})}_F^2 \notag \\
&\stackrel{\text{(c)}}{\leq} (1+\nu_3)\Xi^{(mH)} + 2(1+\nu_3^{-1})\eta^2H\beta^2\sum_{t'=mH}^{\hatt-1}\bbE\verts{\bV^{(t')}}_F^2  + 4(M^2+1)(1+\nu_3^{-1})\eta^2H \textstyle \sum_{t'=mH}^{\hatt-1}\(L^2\Xi^{(t')} + nG^2\) \notag \\
& \quad + 2(1+\nu_3^{-1})\eta^2H\sum_{t'=mH}^{\hatt-1}\(2(M^2{+}1)nB^2\bbE\verts{\nabla f(\blx^{(t')})}_2^2 {+} n\sigma^2\) \notag \\
&\leq (1+\nu_3)\Xi^{(mH)} + 2(1+\nu_3^{-1})\eta^2H^2n\(2(M^2+1)G^2 + \sigma^2\)  + 4(1+\nu_3^{-1})\eta^2H(M^2{+}1)\sum_{t'=mH}^{\hatt-1}L^2\Xi^{(t')}  \notag \\
& \quad +nB^2\bbE\Vert\nabla f(\blx^{(t')})\Vert_2^2  + 2(1+\nu_3^{-1})\eta^2H\beta^2 \textstyle \sum_{t'=mH}^{\hatt-1}\bbE \Vert\bV^{(t')}\Vert_F^2 
\label{equi-17-interim1}
\end{align}
where (a) follows from the same reasoning using which we obtained \eqref{equi16-interim0}, (b) uses $\Xi^{(mH)}=\bbE\verts{\bX^{(mH)} - \blX^{(mH)}}_F^2$, and (c) follows from \eqref{eq:mom-update-norm} (in Proposition~\ref{prop:mom-update-norm}).\\
As mentioned in Section~\ref{subsec:proof-similar-eq16}, our aim is to get an upper-bound on $S^{(\hatt)}$, which is defined in \eqref{Sj-defn} as $S^{(\hatt)} = \Xi^{(\hatt)} + \bbE\verts{\bX^{(\hatt)} - \bhX^{(\lfloor \hatt/H\rfloor H)}}_F^2$. However, in \eqref{equi-17-interim1}, we have only derived an upper-bound on $\Xi^{(\hatt)}$. So, we need to derive a similar upper-bound on the other term $\bbE\verts{\bX^{(\hatt)} - \bhX^{(\lfloor \hatt/H\rfloor H)}}_F^2$, and then adding both the upper-bounds gives a bound on $S^{(\hatt)}$. 

\noindent Note that since $mH\leq \hatt < (m+1)H$, we have $\lfloor\frac{\hatt}{H}\rfloor = m$. In order to upper-bound $\bbE\verts{\bX^{(\hatt)} - \bhX^{(mH)}}_F^2$, we can follow the same steps that we used from \eqref{equi-17-interim0} to \eqref{equi-17-interim1} (just replace $\blX^{(mH)}$ with $\bhX^{(mH)}$). This would give the following bound:
\begin{align}
& \bbE \Vert\bX^{(\hatt)} - \bhX^{(mH)}\Vert_F^2  \leq (1+\nu_3)\bbE \Vert\bX^{(mH)} - \bhX^{(mH)}\Vert_F^2  + 2(1{+}\nu_3^{-1})\eta^2H^2   [n\(2(M^2{+}1)G^2 {+} \sigma^2\) {+}  \beta^2  \hspace{-0.3cm}  \sum_{t'=mH}^{\hatt-1}\hspace{-0.2cm}\bbE\Vert\bV^{(t')}\Vert_F^2  ] \notag \\
& \quad + 4(1+\nu_3^{-1})\eta^2H(M^2+1)nB^2\sum_{t'=mH}^{\hatt-1}\bbE\verts{\nabla f(\blx^{(t')})}_2^2  +  4(1+\nu_3^{-1})\eta^2H(M^2+1)L^2\sum_{t'=mH}^{\hatt-1}\Xi^{(t')}  \label{equi-17-interim2}
\end{align}
Adding \eqref{equi-17-interim1} and \eqref{equi-17-interim2}, and using the definition that $S^{(\hatt)} = \Xi^{(\hatt)} + \bbE\verts{\bX^{(\hatt)} - \bhX^{(\lfloor \hatt/H\rfloor H)}}_F^2$ together with that $\Xi^{(t')}\leq S^{(t')}$, and taking $\nu_3=\frac{\gamma\delta}{4}$, we get:
\begin{align}
S^{(\hatt)} & \leq (1+\frac{\gamma\delta}{4})S^{(mH)} + 4(1{+}\frac{4}{\gamma\delta})\eta^2H^2n\(2(M^2{+}1)G^2 {+} \sigma^2\) + 4(1+\frac{4}{\gamma\delta})\eta^2H\beta^2\sum_{t'=mH}^{\hatt-1}\bbE \Vert\bV^{(t')}\Vert_F^2 \notag \\
& \quad + 8(1+\frac{4}{\gamma\delta})\eta^2H(M^2+1)\hspace{-0.2cm}\sum_{t'=mH}^{\hatt-1}(L^2S^{(t')}  {+} nB^2\bbE \Vert\nabla f(\blx^{(t')})\Vert_2^2) \label{equi-17-interim3}
\end{align}
In order to make our calculations less cluttered later, we would like to write all terms (except the first one) in the RHS above in the same form as given in \eqref{temp_equi-16-final-bound}. Indeed, it can be verified easily that $4(1+\frac{4}{\gamma\delta})\leq c_1$, where $c_1$ is exactly the same as in \eqref{temp_equi-16-final-bound}. Substituting this in \eqref{equi-17-interim3} above yields the bound below for $mH\leq \hatt < (m+1)H$, where $m\in\lfloor\frac{t}{H}\rfloor-1$:
\begin{align}
S^{(\hatt)} & \leq (1+\frac{\gamma\delta}{4})S^{(mH)} + 2c_1\eta^2H^2n\(2(M^2+1)G^2 + \sigma^2\) + c_1\eta^2H\beta^2\sum_{t'=mH}^{\hatt-1}\bbE\verts{\bV^{(t')}}_F^2 \notag \\
& \quad + 2c_1\eta^2H(M^2+1)\sum_{t'=mH}^{\hatt-1}(L^2S^{(t')}  + nB^2\bbE\verts{\nabla f(\blx^{(t')})}_2^2) \label{equi-17-interim4}
\end{align} 
where $c_1$ is exactly the same as in \eqref{temp_equi-16-final-bound}.

\subsection{Proof of Lemma~\ref{lem:similar-lem14}}\label{subsec:proof-similar-lem14}
Let $A=2c_1H^2n\(2(M^2{+}1)G^2 {+} \sigma^2\)$, $D=\frac{c_1H\beta^2}{(1-\beta)}$,$C=2c_1H(M^2{+}1)nB^2$, and $\Lambda^{(t')}=(1-\beta)\bbE\verts{\bV^{(t)}}_F^2$, where $c_1$ is the same as in \eqref{temp_equi-16-final-bound}.
Since $\eta\leq\sqrt{\frac{\gamma\delta}{512c_1H^2(M^2+1)L^2}}$, we have $2c_1\eta^2H(M^2+1)L^2\leq \frac{\gamma\delta}{4}\frac{1}{64H}$.

Take any $t\in[T]$ and let $m=\lfloor\frac{t}{H}\rfloor-1$. With these substitutions and letting $\alpha=\frac{\gamma\delta}{4}$, the bound from \eqref{temp_equi-16-final-bound} for any $t$ such that $(m+1)H\leq t\leq(m+2)H-1$ becomes:
	\begin{align}\label{bound1-on-St}
		& S^{(t)} \leq \left(1 - \frac{\alpha}{2} \right) S^{(mH)} + A \eta^2 + \frac{\alpha}{64H} \sum_{t' = mH}^{t-1} S^{(t')} + C \eta^2 \sum_{t' = mH}^{t-1} \bbE \verts{ \nabla f (\blx^{(t')}) }^2 + D \eta^2 \sum_{t' = mH}^{t-1} \Lambda^{(t')}.
	\end{align}
	And for any $\hatt$ such that $mH \leq \hatt < (m+1)H$, the bound from \eqref{equi-17-interim4} becomes:
	\begin{align}\label{bound2-on-St}
		& S^{(\hatt)} \leq \left(1 - \frac{\alpha}{2} \right) S^{(mH)} + A \eta^2 + \frac{\alpha}{64H} \sum_{t' = mH}^{\hatt-1} S^{(t')}  + C \eta^2 \sum_{t' = mH}^{\hatt-1} \bbE \verts{ \nabla f (\blx^{(t')}) }^2 + D \eta^2 \sum_{t' = mH}^{\hatt-1} \Lambda^{(t')}.
	\end{align}	
	Consider \eqref{bound1-on-St}. Substituting the value of $S^{(t-1)}$ recursively in the RHS of \eqref{bound1-on-St}, we get:
	\begin{align*}
		S^{(t)} & \leq \left(1 - \frac{\alpha}{2} \right) S^{(mH)} + A \eta^2 + \frac{\alpha}{64H} \sum_{t' = mH}^{t-2} S^{(t')}  + C \eta^2 \sum_{t' = mH}^{t-1} \bbE \verts{ \nabla f (\blx^{(t')}) }^2 + D \eta^2 \sum_{t' = mH}^{t-1} \Lambda^{(t')} \\
		& \quad + \frac{\alpha}{64H} \left(  \left(1 - \frac{\alpha}{2} \right) S^{(mH)} + A \eta^2 + \frac{\alpha}{64H} \sum_{t' = mH}^{t-2} S^{(t')}  + C \eta^2 \sum_{t' = mH}^{t-2} \bbE \verts{ \nabla f (\blx^{(t')}) }^2 + D \eta^2 \sum_{t' = mH}^{t-2} \Lambda^{(t')}   \right) \\
		& = \left(1 {-} \frac{\alpha}{2} \right) \left( 1 {+} \frac{\alpha}{64H} \right) S^{(mH)} + A \left( 1 + \frac{\alpha}{64H} \right)  \eta^2 + D \eta^2 \Lambda_{t-1}  + \frac{\alpha}{64H} \left( 1 {+} \frac{\alpha}{64H} \right) \sum_{t' = mH}^{t-2} S^{(t')}  \\
		& \quad + \left( 1 + \frac{\alpha}{64H} \right) D \eta^2 \sum_{t' = mH}^{t-2 } \Lambda^{(t')} + \left( 1 {+} \frac{\alpha}{64H} \right) C\eta^2 \sum_{t' = mH}^{t-2} \bbE \verts{ \nabla f (\blx^{(t')}) }^2 {+} C\eta^2 \bbE \verts{ \nabla f (\blx^{(t-1)}) }^2 
	\end{align*}
	Substituting the values in the RHS till $(m+1)H$, we get:
	\begin{align*}
		S^{(t)} & \leq \left(1 - \frac{\alpha}{2} \right) \left( 1 + \frac{\alpha}{64H} \right)^{H} S^{(mH)} + A \left( 1 + \frac{\alpha}{64H} \right)^H  \eta^2  + \frac{\alpha}{64H} \left( 1 + \frac{\alpha}{64H} \right)^H \sum_{t' = mH}^{(m+1)H-1} S^{(t')} \\
		& \quad + \left( 1 + \frac{\alpha}{64H} \right)^H \eta^2 \sum_{t' = mH}^{(m+1)H-1}( C\bbE \verts{ \nabla f (\blx^{(t')}) }^2 + D\Lambda^{(t')}) \\
		& \quad + \eta^2  \sum_{t' = (m+1)H}^{t-1} \left( 1 {+} \frac{\alpha}{64H} \right)^{t-1-t'}( C\bbE \Vert \nabla f (\blx^{(t')}) \Vert^2 + D\Lambda^{(t')})  
	\end{align*}
	Now consider $t'$ such that $mH \leq t' < (m+1)H$. Substituting the value of $S^{((m+1)H -1)}$ from \eqref{bound2-on-St} int the R.H.S above gives: 
	\begin{align*}
		& S^{(t)}   \leq \left(1 - \frac{\alpha}{2} \right) \left( 1 + \frac{\alpha}{64H} \right)^{H} S^{(mH)} + A \left( 1 + \frac{\alpha}{64H} \right)^H  \eta^2  + \frac{\alpha}{64H} \left( 1 + \frac{\alpha}{64H} \right)^H \sum_{t' = mH}^{(m+1)H-2} S^{(t')} \\
		& \quad + \left( 1 + \frac{\alpha}{64H} \right)^H \eta^2 \sum_{t' = mH}^{(m+1)H-1} (C \bbE \verts{ \nabla f (\blx^{(t')}) }^2 +  D\Lambda^{(t')})  \\
		& \quad + \frac{\alpha}{64H} \left( 1 {+} \frac{\alpha}{64H} \right)^H \left[ (1{+}\frac{\alpha}{2}) S^{(mH)}   {+} \frac{\alpha}{64H} \sum_{j = mH}^{(m+1)H-2} \hspace{-0.35cm} S^{(j)} + C\eta^2 \hspace{-0.35cm} \sum_{j = mH}^{(m+1)H-2} \hspace{-0.35cm} \bbE \Vert\nabla f(\blx^{(j)}) \Vert^2 + D \eta^2 \hspace{-0.35cm} \sum_{j=mH}^{(m+1)H-2} \hspace{-0.2cm} \Lambda^{(j)} {+} A \eta^2 \right] \\
		& \quad + \eta^2  \sum_{t' = (m+1)H}^{t-1} \left( 1 {+} \frac{\alpha}{64H} \right)^{t-1-t'}\hspace{-0.2cm} (C \bbE \Vert \nabla f (\blx^{(t')}) \Vert^2 + D \Lambda^{(t')})  \\
		& \leq \left(\left(1 - \frac{\alpha}{2} \right) +  \frac{\alpha}{64H} \left( 1 + \frac{\alpha}{2} \right) \right) \left( 1 + \frac{\alpha}{64H} \right)^{H} S^{(mH)}  + A \left( 1 + \frac{\alpha}{64H} \right)^{H+1}  \eta^2 {+} \frac{\alpha}{64H} \left( 1 {+} \frac{\alpha}{64H} \right)^{H+1} \sum_{t' = mH}^{(m+1)H-2} \hspace{-0.2cm} S^{(t')} \\
		&  + \left( 1 + \frac{\alpha}{64H} \right)^{H+1} \eta^2 \hspace{-0.35cm} \sum_{t' = mH}^{(m+1)H-2} \hspace{-0.2cm} (C\bbE \Vert \nabla f (\blx^{(t')}) \Vert^2 + D\Lambda^{(t')})   + \eta^2 \hspace{-0.35cm} \sum_{t' = (m+1)H}^{t-1} \hspace{-0.35cm} \left( 1 {+} \frac{\alpha}{64H} \right)^{t-1-t'} \hspace{-0.2cm}(C\bbE \verts{ \nabla f (\blx^{(t')}) }^2 {+} D\Lambda^{(t')})   \\
		& \quad + \eta^2  \left( 1 {+} \frac{\alpha}{64H} \right)^{H}  (C\bbE \Vert \nabla f (\blx^{((m+1)H-1)}) \Vert^2 + D\Lambda^{((m+1)H-1)}) 
	\end{align*} 
	Now we note that for $0 < \alpha \leq 1$, $  \frac{\alpha}{64H} \left( 1 + \frac{\alpha}{2} \right)  \leq \left( 1 - \frac{\alpha}{2} \right) \frac{\alpha}{16H}$. Using this fact in the first term and $(1 + \frac{\alpha}{64H}) \leq (1+ \frac{\alpha}{16H})$, and  $\left( 1 + \frac{\alpha}{64H} \right)^{t-1-t'} \leq ( 1+ \frac{\alpha}{16H})^{H}$ for all $t' \in \{(m+1)H,\hdots,t-1\}$ in the R.H.S above gives:
		\begin{align*}
		S^{(t)}  & \leq  \left(1 {-} \frac{\alpha}{2} \right)  \left( 1 + \frac{\alpha}{16H} \right)^{H+1} S^{(mH)}  {+} A \left( 1 {+} \frac{\alpha}{16H} \right)^{H+1}  \eta^2   + \frac{\alpha}{64H} \left( 1 + \frac{\alpha}{16H} \right)^{H+1} \sum_{t' = mH}^{(m+1)H-2} S^{(t')} \\
		& \quad  + \left( 1 + \frac{\alpha}{16H} \right)^{H+1} \eta^2 \sum_{t' = mH}^{(m+1)H-2} (C\bbE \verts{ \nabla f (\blx^{(t')}) }^2+ D\Lambda^{(t')})  \\
		& \quad  + \eta^2  \left( 1 + \frac{\alpha}{16H} \right)^{H}\sum_{t' = (m+1)H}^{t-1}  (C\bbE \verts{ \nabla f (\blx^{(t')}) }^2 + D\Lambda^{(t')})  \\
		& \quad + \eta^2  \left( 1 {+} \frac{\alpha}{16H} \right)^{H} \hspace{-0.2cm} (C\bbE \verts{ \nabla f (\blx^{((m+1)H-1)}) }^2 + D\Lambda^{((m+1)H-1)}) 		
	\end{align*}
	Using $\left( 1 + \frac{\alpha}{16H} \right)^{H}\leq\left( 1 + \frac{\alpha}{16H} \right)^{H+1}$ in the last two terms and then clubbing together terms respectively with $C$ and $D$:
		\begin{align*}
		S^{(t)}  & \leq  \left(1 - \frac{\alpha}{2} \right)  \left( 1 {+} \frac{\alpha}{16H} \right)^{H+1} S^{(mH)} + A \left( 1 + \frac{\alpha}{16H} \right)^{H+1}  \eta^2 + \left( 1 + \frac{\alpha}{16H} \right)^{H+1} \eta^2 \hspace{-0.25cm} \sum_{t' = mH}^{t-1} (C\bbE \verts{ \nabla f (\blx^{(t')}) }^2 {+} D\Lambda^{(t')}) \\
		& \quad + \frac{\alpha}{64H} \left( 1 + \frac{\alpha}{16H} \right)^{H+1} \sum_{t' = mH}^{(m+1)H-2} S^{(t')} 
	\end{align*}
Recursively substituting the values till $mH$ gives us:
	\begin{align*}
		S^{(t)} & \leq   \left(1 - \frac{\alpha}{2} \right)  \left( 1 + \frac{\alpha}{16H} \right)^{2H} S^{(mH)} + A \left( 1 + \frac{\alpha}{16H} \right)^{2H}  \eta^2  + \left( 1 + \frac{\alpha}{16H} \right)^{2H} \eta^2 \hspace{-0.2cm} \sum_{t' = mH}^{t-1} (C \bbE \verts{ \nabla f (\blx^{(t')}) }^2  +D\Lambda^{(t')}) 
	\end{align*}
	For $\alpha \leq 1$, we note that $\left( 1 + \frac{\alpha}{16H} \right)^{2H} \leq e^{\frac{\alpha}{8}} \leq 1 + \frac{\alpha}{4}$. Plugging this in the first term on the RHS and using $\left(1 - \frac{\alpha}{2} \right)\left(1 + \frac{\alpha}{4} \right)\leq \left(1 - \frac{\alpha}{4} \right)$ and $\left( 1 + \frac{\alpha}{16H} \right)^{2H} \leq 1 + \frac{\alpha}{4} \leq 2$ gives us the following recursion equation for any $t \in [T]$:
	\begin{align} \label{lem14-eq1}
		& S^{(t)}  \leq   \left(1 - \frac{\alpha}{4} \right)  S^{(mH)} + 2 A  \eta^2 + 2 C\eta^2 \sum_{t' = mH}^{t-1} \bbE \verts{ \nabla f (\blx^{(t')}) }^2 + 2 D \eta^2 \sum_{t' = mH}^{t-1} \Lambda^{(t')} 
	\end{align}
	Unrolling recursion equation in \eqref{lem14-eq1} for $S^{(mH)}$ till $0$, we get:
	\begin{align}\label{lem14-eq11}
		S^{(t)} & \leq   2 A  \eta^2  \sum_{j=0}^{m-1}  \left(1 - \frac{\alpha}{4} \right)^{j} + 2 D \eta^2 \sum_{j=0}^{t-1} \left(1 - \frac{\alpha}{4} \right)^{\lfloor \frac{t-j}{H} \rfloor}  \Lambda^{(j)}  + 2 C\eta^2  \sum_{j=0}^{t-1}  \left(1 - \frac{\alpha}{4} \right)^{\lfloor \frac{t-j}{H} \rfloor}  \bbE \verts{ \nabla f (\blx^{(j)}) }^2
	\end{align}
	Note that $\sum_{j=0}^{m-1}  \left(1 - \frac{\alpha}{4} \right)^{j} \leq \frac{4}{\alpha}$. Using this and the bound $\left(1 - \frac{\alpha}{4} \right)^{\lfloor \frac{t-j}{H} \rfloor} \leq  2 \left( 1 - \frac{\alpha}{8H} \right)^{t-j}$ (proved in Appendix C-E) into \eqref{lem14-eq11} gives us:
	\begin{align*}
		S^{(t)}  & \leq   \frac{8A\eta^2}{\alpha} 
		+ 4 C\eta^2  \sum_{j=0}^{t-1}  \left(1 - \frac{\alpha}{8H} \right)^{t-j}  \bbE \verts{ \nabla f (\blx^{(j)}) }^2 + 4 D \eta^2 \sum_{j=0}^{t-1} \left(1 - \frac{\alpha}{8H} \right)^{t-j}  \Lambda^{(j)} 
	\end{align*}
	Taking summation from $t=0$ to $T-1$, we get:
	\begin{align} \label{lem14-eq2}
		\sum_{t=0}^{T-1} S^{(t)}  & \leq   
		 4 C\eta^2  \sum_{t=0}^{T-1} \sum_{j=0}^{t-1}  \left(1 - \frac{\alpha}{8H} \right)^{t-j}  \bbE \verts{ \nabla f (\blx^{(j)}) }^2  + 4 D \eta^2 \sum_{t=0}^{T-1} \sum_{j=0}^{t-1} \left(1 - \frac{\alpha}{8H} \right)^{t-j}  \Lambda^{(j)}  + \frac{8A\eta^2}{\alpha} T  \notag \\
		& \leq \frac{8A\eta^2}{\alpha} T + 4C \eta^2 \sum_{j=0}^{T-1} \sum_{t=j+1}^{T-1}  \left(1 - \frac{\alpha}{8H} \right)^{t-j}  \bbE \verts{ \nabla f (\blx^{(j)}) }^2 + 4 D \eta^2 \sum_{j=0}^{T-1} \sum_{t=j+1}^{T-1} \left(1 - \frac{\alpha}{8H} \right)^{t-j}  \Lambda^{(j)} \notag \\
		& \leq \frac{8A\eta^2T}{\alpha} + \frac{32C \eta^2 H}{\alpha} \sum_{t=0}^{T-1}  \bbE \Vert \nabla f (\blx^{(t)}) \Vert^2  + \frac{32 D H \eta^2}{\alpha} \sum_{t=0}^{T-1}  \Lambda^{(t')} 
	\end{align}
	To bound the last term in the RHS of \eqref{lem14-eq2}, from the definition of $\Lambda^{(t')}$ in \eqref{eq:bound_v}, note that:
	\begin{align*}
		\sum_{t=0}^{T-1}  \Lambda^{(t')}  = \sum_{t=0}^{T-1} \sum_{j=0}^{t} \beta^{t-j} \bbE \verts{\nabla \bF(\bX^{(j)},\boldsymbol{\xi}^{(j)})}_F^2
	\end{align*}
	From Proposition \ref{prop:mom-update-norm} (from page~\pageref{prop:mom-update-norm}) to bound the stochastic gradient in the RHS of above equation gives us:
	\begin{align*}
		\sum_{t=0}^{T-1}  \Lambda^{(t')}  & \leq \sum_{t=0}^{T-1} \sum_{j=0}^{t} \beta^{t-j} \left[ 2(M^2+1) (L^2\Xi^{(j)} + nG^2) \right]   +\sum_{t=0}^{T-1} \sum_{j=0}^{t} \beta^{t-j} \left[ 2(M^2+1)nB^2\bbE\verts{\nabla f(\blx^{(j)})}_2^2 + n\sigma^2 \right] \\
		& \leq \frac{2(M^2+1)nG^2 + n \sigma^2}{(1-\beta)} T + \frac{2(M^2+1)L^2}{(1-\beta)} \sum_{t=0}^{T-1} \Xi^{(t)}  + \frac{2(M^2+1)nB^2}{(1-\beta)} \sum_{t=0}^{T-1} \bbE\verts{\nabla f(\blx^{(t)})}_2^2
	\end{align*}
	Substituting the above bound in \eqref{lem14-eq2}, we have:
	\begin{align*}
		\sum_{t=0}^{T-1} S^{(t)}  & \leq \eta^2 T \left(\frac{8A\eta^2}{\alpha} {+} \left( \frac{32 D H}{\alpha} \right)  \left( \frac{2(M^2+1)nG^2 {+} n \sigma^2}{(1-\beta)}   \right)   \right) + \frac{64 D H (M^2+1)L^2\eta^2 }{\alpha(1-\beta)}  \sum_{t=0}^{T-1}  \Xi^{(t)}  \\
		& \quad + \eta^2 \left(\frac{32C  H}{\alpha}   + \left( \frac{32 D H }{\alpha} \right) \frac{2(M^2{+}1)nB^2}{(1-\beta)}   \right) \sum_{t=0}^{T-1}  \bbE \Vert \nabla f (\blx^{(t)}) \Vert^2 
	\end{align*}
	Choose $\eta \leq \sqrt{ \frac{  \alpha (1-\beta) }{128 D H (M^2+1)L^2}}$ and using that fact that $\Xi^{(t)} \leq S^{(t)}$ for all $t \in [T]$ and rearranging the summation term gives:
	\begin{align}\label{upper-bound-avg-St}
		\frac{1}{T}\sum_{t=0}^{T-1} S^{(t)} & \leq 2\eta^2 J_1 + 2\eta^2 J_2 \frac{1}{T}\sum_{t=0}^{T-1}  \bbE \verts{ \nabla f (\blx^{(t)}) }^2,
	\end{align}
	where $J_1= \left(\frac{8A\eta^2}{\alpha} + \left( \frac{32 D H}{\alpha} \right)  \left( \frac{2(M^2+1)nG^2 + n \sigma^2}{(1-\beta)}   \right)   \right)$ and $J_2=\left(\frac{32C  H}{\alpha}   + \left( \frac{32 D H }{\alpha} \right) \frac{2(M^2+1)nB^2}{(1-\beta)}   \right)$.
	

}

\section{Experiments}\label{experiments}

\allowdisplaybreaks
{


\noindent In this section, we provide comparison of our proposed algorithm SQuARM-SGD, which uses momentum updates to CHOCO-SGD \cite{koloskova_decentralized_2019} and SPARQ-SGD \cite{singh2019sparq} which consider compressed decentralized training (and local SGD, triggered communication for \cite{singh2019sparq}) but do not incorporate momentum in their algorithms. 
We empirically demonstrate that using momentum based updates can increase the test performance of the learned model in large-scale decentralized training.


\begin{figure}[tp!]
	\centering
	\begin{subfigure}{.5\textwidth}
		\centering
		\includegraphics[width=.6\linewidth]{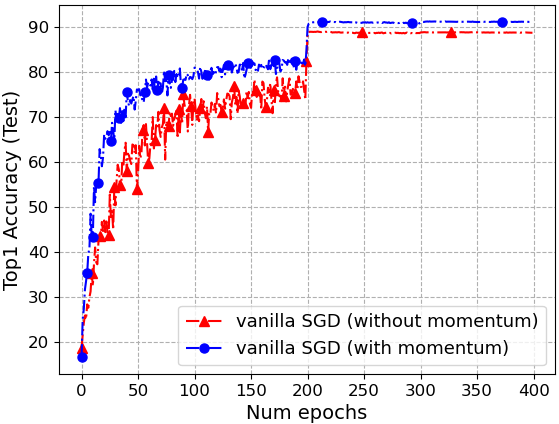}
		\caption{For vanilla SGD}
		\label{fig:1a}
	\end{subfigure}%
	\begin{subfigure}{.5\textwidth}
		\centering
		\includegraphics[width=.6\linewidth]{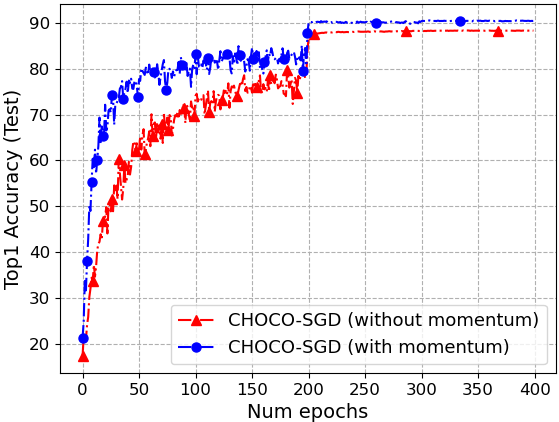}
		\caption{For CHOCO-SGD ($TopK$)}
		\label{fig:1b}
	\end{subfigure}
	\caption{Increase in test accuracy when using momentum updates.}
	\label{fig:1}
\end{figure}

%
%
%
%

\begin{figure*}[tp!]
	\begin{minipage}[t]{.5\linewidth}
		\begin{subfigure}[t]{0.488\textwidth}
			\centering
			\includegraphics[width=\textwidth]{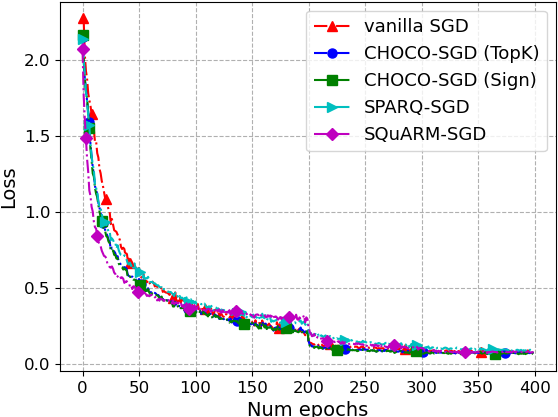}
			\centering
			\caption{\centering Comparison of training loss}\label{rev1:fig1.1}
		\end{subfigure}
		\begin{subfigure}[t]{0.488\textwidth}
			\centering
			\includegraphics[width=\textwidth]{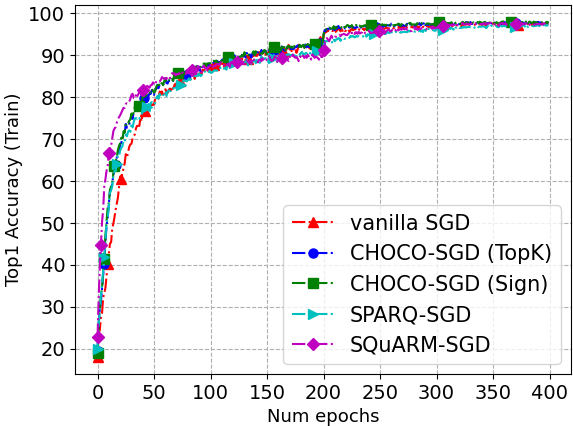}
			\centering
			\caption{ \centering Comparison train accuracy.}\label{rev1:fig1.2}
		\end{subfigure}
		\centering
		\captionsetup{justification=centering}
		\caption{Training metrics for different schemes.} \label{rev1:fig1}
	\end{minipage}
	\begin{minipage}[t]{.5\linewidth}
		\begin{subfigure}[t]{0.49\textwidth}
			\centering
			\includegraphics[width=\textwidth]{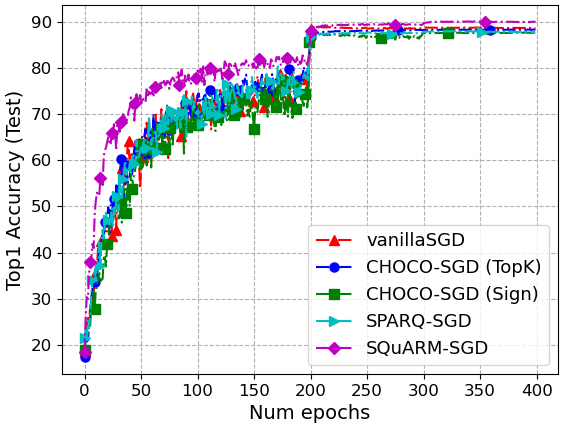}
			\caption{Comparison of test accuracy}\label{fig:2a}
		\end{subfigure}
		\begin{subfigure}[t]{0.49\textwidth}
			\centering
			\includegraphics[width=\textwidth]{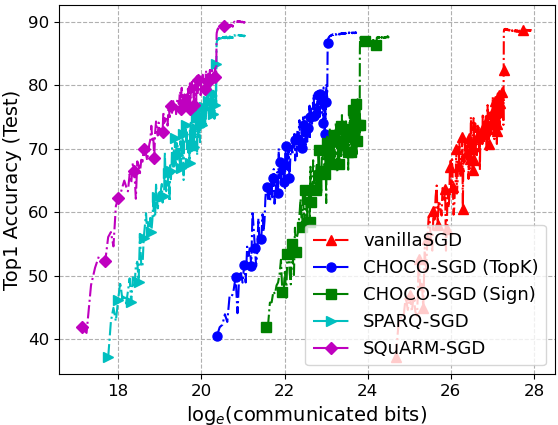}
			\caption{Test accuracy vs. no. of bits}\label{fig:2b}
		\end{subfigure}
		\centering
		\captionsetup{justification=centering}
		\caption{Test performance comparison for different schemes.} \label{fig:2}
	\end{minipage}
\end{figure*}

%
\paragraph{Setup.}
We match the setting in CHOCO-SGD, SPARQ-SGD and train ResNet20 \cite{wen2016learning} models on the CIFAR-10 \cite{cifar} dataset with $n=8$ nodes connected in a ring topology. Learning rate follows a schedule: initialized to $0.2$, warmup period of 5 epochs and has a decay of 10 at epoch 200 and 300; we stop training at epoch 400. For SQuARM-SGD, we use Nesterov momentum with a factor of $\beta = 0.9$ and mini-batch size of 256. For either SPARQ-SGD \cite{singh2019sparq} or CHOCO-SGD \cite{koloskova_decentralized_2019}, we do not use momentum.\footnote{{We note that while experimental results in \cite{singh2019sparq,koloskova_decentralized_2019} were provided with momentum, they do not consider momentum in their analysis. Thus for a fair comparison, we consider our algorithm SQuARM-SGD with momentum updates while SPARQ-SGD, CHOCO-SGD are evaluated without momentum.}} Matching \cite{singh2019sparq}, SQuARM-SGD consists of $H=5$ local iterations and we take top $1\%$ elements of each tensor and only transmit the sign and norm of the result. The triggering threshold follows a schedule piecewise constant: initialized to $2.5$ and increases by $1.5$ after every $20$ epochs till $350$ epochs are complete, while maintaining that $c_t < \nicefrac{1}{\eta}$ for all $t$. We compare performance of SQuARM-SGD against SPARQ-SGD (which uses $SignTopK$ compression, local iterations and threshold based communication), CHOCO-SGD with $Sign$, $TopK$ compression (taking top $1\%$ of elements of the tensor) and decentralized vanilla SGD \cite{lian2017can}.

\paragraph{Results.}
We first demonstrate that performing momentum updates can lead to better test performance
when training large scale machine learning models. Figure~\ref{fig:1a} and Figure~\ref{fig:1b} show test accuracy with and without momentum for vanilla SGD decentralized training and CHOCO-SGD (with $TopK$ compression), respectively. We observe that training with momentum updates improves test performance by $2$-$3\%$. Figure~\ref{rev1:fig1} shows the training loss and training accuracy performance of all the schemes, and  Figure~\ref{fig:2} compares the test performance. In our numerics, SQuARM-SGD incorporates momentum updates (also theoretically analyzed) while CHOCO-SGD ($Sign$ or $TopK$ compression) and SPARQ-SGD ($SignTopK$ compression and local iterations) do not.  From Figure~\ref{rev1:fig1}, we observe that each scheme is able to train the ResNet-20 model well over the CIFAR-10 dataset. Figure~\ref{fig:2a} shows that SQuARM-SGD has a better test performance than other methods by around $2\%$ owing to momentum updates. Moreover, SQuARM-SGD reaches a higher test accuracy in relatively fewer epochs due to speedup by momentum. As SQuARM uses $SignTopK$ compression along with local iterations and triggering, it also achieves the target test accuracy of about 90\% using significantly less communication bits\footnote{As SPARQ-SGD \cite{singh2019sparq} also uses $SignTopK$ compression with local iterations and event-triggering, it uses the same amount of communication bits as SQuARM-SGD although with an inferior test performance due to absence of momentum updates.} than either CHOCO-SGD or vanilla SGD training as demonstrated in Figure \ref{fig:2b}.

\paragraph{Wall clock comparison.}
Figure \ref{rev1:fig2.1} shows the wall-clock time for training the ResNet-20 model for all the schemes logged in at each epoch. It can be seen that performing the encoding/decoding process for CHOCO-SGD (Sign/TopK)\cite{koloskova_decentralized_2019-1} can be expensive, and takes more time than vanilla SGD. For SPARQ-SGD and SQuARM-SGD, we consider 10 local iterations, and thus the nodes only need to perform the encoding decoding process once in every 10 iterations as compared to each iteartion in vanilla SGD or CHOCO-SGD. The time take for SQuARM-SGD is a bit higher than SPARQ-SGD on account on performing more computation with the momentum updates. \\
Figure \ref{rev1:fig2.2} shows the test error performance as a function of the wall clock time elapsed during training. It can be seen that on account of using momentum and local iterations, SQuARM-SGD achieves a higher test performance while taking about $0.5 \times$ the time compared to CHOCO-SGD for training, and about $0.75 \times$ the time compared to vanilla-SGD.

\begin{figure}[tp!]
	\centering
	\begin{subfigure}{.5\textwidth}
		\centering
		\includegraphics[width=.7\linewidth]{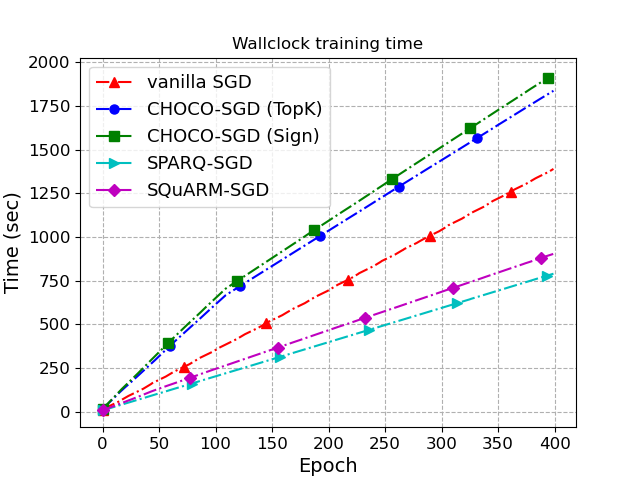}
		\caption{\centering Wall-clock training time logged at each epoch.}
		\label{rev1:fig2.1}
	\end{subfigure}%
	\begin{subfigure}{.5\textwidth}
		\centering
		\includegraphics[width=.7\linewidth]{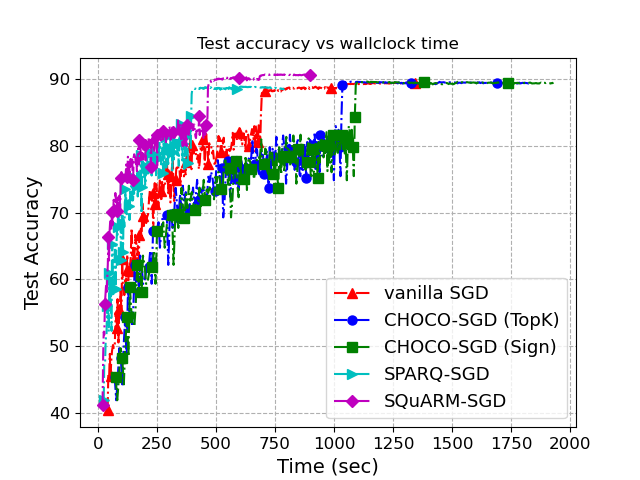}
		\caption{\centering Test accuracy vs wall-clock time.}
		\label{rev1:fig2.2}
	\end{subfigure}
	\caption{Comparing performance of schemes with wall-clock training time.}
	\label{rev1:fig2}
\end{figure}

}
\section*{Acknowledgment}\label{ack}
{\small This work was supported in part by NSF under
	Grant \#2007714 and Grant \#1955632; in part by UC-NL under Grant LFR18-548554; and in part by the Army Research Laboratory under Cooperative
	Agreement under Grant W911NF-17-2-0196. The views and conclusions contained in this document are those of the authors and should not be interpreted as representing the official policies, either expressed or implied, of the Army Research Laboratory or the U.S. Government. The U.S. Government is authorized to reproduce and distribute reprints for Government purposes notwithstanding any copyright notation here on.
}

\bibliographystyle{alpha}
\bibliography{ref}

\appendix
\newpage
\section{Preliminaries} \label{suppl_prelim}

\paragraph{Notation.}
Unless specified otherwise, for a vector $\bu$, we write $\|\bu\|$ to denote the $\ell_2$-norm $\|\bu\|_2$.

\subsection{Vector and matrix inequalities}
\begin{fact}
	Let $\mathbf{M} \in \mathbb{R}^{p \times q} $ be a matrix with entries $[m_{ij}] $, $i \in [p], j \in [q]$. The Frobenius norm of $\mathbf{M}$ is given by : $$ \verts{\mathbf{M}}_F  =  \sqrt{ \sum\limits_{i=1}^{p} \sum\limits_{j=1}^{q} \vert m_{ij} \vert^2  } $$ 
	Consider any two matrices $ \mathbf{A} \in \mathbb{R}^{d \times n}$, $\mathbf{B} \in \mathbb{R}^{n \times n}$. Then the following holds:
	\begin{align} \label{bound_frob_mult}
		\Vert \mathbf{AB} \Vert_F \leq \Vert \mathbf{A} \Vert_F \Vert \mathbf{B} \Vert_2
	\end{align}
\end{fact}
\begin{fact}
	For any set of $n$ vectors $ \{ \mathbf{a_1} ,\hdots , \mathbf{a_n}    \}$ where $\mathbf{a_i} \in \mathbb{R}^d$, we have:
	\begin{align} \label{bound_seq_sum}
		\verts{\sum_{i=1}^{n} \mathbf{a_i}}^2 \leq n \sum_{i=1}^{n} \verts{\mathbf{a_i}}^2
	\end{align}
\end{fact}
\begin{fact}
	For any two vectors $\mathbf{a},\mathbf{b} \in \mathbb{R}^d$, for all $\gamma >0$, we have:
	\begin{align} \label{bound_inner_prod}
		2 \lragnle{\mathbf{a},\mathbf{b}} \leq \gamma \verts{\mathbf{a}}^2 + \gamma^{-1} \verts{\mathbf{b}}^2
	\end{align}
\end{fact}
\begin{fact} \label{bound_l2_sum}
	For any two vectors $\mathbf{a},\mathbf{b} \in \mathbb{R}^d$, for all $\alpha >0$, we have:
	\begin{align} 
	\verts{\mathbf{a}+\mathbf{b}}^2 \leq (1+\alpha) \verts{\mathbf{a}}^2 + {(1 + \alpha^{-1})} \verts{\mathbf{b}}^2
	\end{align}
	Similar inequality holds for matrices in Frobenius norm, i.e., for any two matrices $\mathbf{A},\mathbf{B} \in \mathbb{R}^{p \times q} $ and for any $\alpha >0$ , we have
	\begin{align*}
	\verts{\mathbf{A} + \mathbf{B} }_F^2 \leq (1 + \alpha) \verts{\mathbf{A}}_F^2 + (1 + \alpha^{-1})\verts{\mathbf{B}}_F^2		
	\end{align*}
\end{fact}
\subsection{Properties of functions}
\begin{definition}[Smoothness]
	A differentiable function $f : \mathbb{R}^d \rightarrow \mathbb{R}$ is L-smooth with parameter $L \geq 0$ if
	\begin{align} \label{l_smooth}
	f(\by) \leq f(\bx) + \langle \nabla f(\bx), \by-\bx  \rangle + \frac{L}{2} \Vert \by-\bx \Vert^2, \hspace{2cm} \forall \bx,\by \in \mathbb{R}^d
	\end{align}
\end{definition}
\begin{lemma}
	Let $f$ be an $L$-smooth function with global minimizer $\bx^*$. We have
	\begin{align} \label{l_smooth_prop}
	\Vert \nabla f(\bx) \Vert^2 \leq 2L( f(\bx) - f(\bx^*) ).
	\end{align}
	
\end{lemma}
\begin{proof}
	By definition of $L$-smoothness, we have
	\begin{align*}
	f(\by) & \leq f(\bx) + \langle \nabla f(\bx), \by-\bx \rangle + \frac{L}{2}\Vert \by-\bx \Vert^2. \\
	\intertext{Taking infimum over y yields:}
	\inf_{\by} f(\by) & \leq \inf_{\by} \left( f(\bx) + \langle \nabla f(\bx), \by-\bx \rangle + \frac{L}{2}\Vert \by-\bx \Vert^2 \right) \\
	& \stackrel{\text{(a)}}{=}  \inf_{\bv: \Vert \bv \Vert = 1} \inf_t \left( f(\bx) + t \langle \nabla f(\bx), \bv \rangle + \frac{L t^2}{2} \right) \\
	& \stackrel{\text{(b)}}{=}  \inf_{\bv: \Vert \bv \Vert = 1} \left( f(\bx) - \frac{1}{2L} \langle \nabla f(\bx),\bv \rangle^2  \right) \\
	& \stackrel{\text{(c)}}{=}  \left( f(\bx) - \frac{1}{2L} \Vert  \nabla f(\bx) \Vert^2  \right) \\
	\end{align*}  
The value of $t$ that minimizes the RHS of (a) is $t=-\frac{1}{L}\langle \nabla f(\bx), \bv \rangle$, this implies (b);
(c) follows from the Cauchy-Schwartz inequality: $\langle \bu, \bv \rangle \leq \| \bu \| \| \bv \|$, where equality is achieved whenever $u=v$.
Now, substituting $\inf \limits_{\by} f(\by) = f(\bx^*)$ in the RHS of (c) yields the result.
\end{proof}

\section{Preliminaries for Convergence with Relaxed Assumptions}\label{app:relaxed-assump-prelims}
\begin{proof}[Proof of Proposition~\ref{prop:variance-reduction_relaxed}]
	This simply follows from the independence of the randomness used in sampling stochastic gradients at different workers.
\end{proof}

\begin{proof}[Proof of Proposition~\ref{prop:bound_v}]
We want to show the following bound on $\bbE\verts{\bV^{(t)}}_F^2$ for any $t$:
\begin{align*}
\bbE\verts{\bV^{(t)}}_F^2 \leq \frac{1}{(1-\beta)}\sum_{k=0}^{t} \beta^{t-k}\bbE\verts{\nabla \bF(\bX^{(k)}, \bxi^{(k)})}_F^2.
\end{align*}
For any $t$, let $\theta_t = \sum_{k=0}^t\beta^{t-k}$.
\begin{align}
\bbE\verts{\bV^{(t)}}_F^2 &= \bbE\verts{\sum_{k=0}^{t} \beta^{t-k}\nabla \bF(\bX^{(k)}, \bxi^{(k)})}_F^2 \notag \\
&= \theta_t^2\bbE\verts{\sum_{k=0}^{t} \frac{\beta^{t-k}}{\theta_t}\nabla \bF(\bX^{(k)}, \bxi^{(k)})}_F^2 \notag \\
&\leq \theta_t\sum_{k=0}^{t} \beta^{t-k}\bbE\verts{\nabla \bF(\bX^{(k)}, \bxi^{(k)})}_F^2 \notag \\
&\leq \frac{1}{1-\beta}\sum_{k=0}^{t} \beta^{t-k}\bbE\verts{\nabla \bF(\bX^{(k)}, \bxi^{(k)})}_F^2.
\end{align}
\end{proof}

\section{Omitted Details from Section~\ref{sec:relaxed-assump-results}}\label{app:relaxed-assump-results}

\subsection{Omitted Details from Section~\ref{subsec:proof-similar-lem11}}\label{app:proof-similar-lem11}

\begin{lemma}
We have the following bounds on $P_1$ and $P_2$ (which are defined in \eqref{mom_noncvx_relaxed-interim6}):
\begin{align*}
P_1 &\leq - \frac{\eta}{2(1-\beta)} \verts{\nabla f(\btx^{(t)})}^2 + \frac{\eta L^2}{2n(1-\beta)}\sum_{i=1}^n\verts{\btx^{(t)} - \bx^{(t)}_i}^2, \\
P_2 &\leq \frac{\sigma^2}{n} + \frac{2(M^2+n)L^2}{n^2}\sum_{i=1}^n\verts{\bx^{(t)}_i - \btx^{(t)}}_2^2 + \frac{2(M^2+n)}{n}\big(G^2+B^2\verts{\nabla f(\btx^{(t)})}_2^2\big).
\end{align*}
\end{lemma}
\begin{proof}
\begin{align}
P_1 &= - \lragnle{\nabla f(\btx^{(t)}) , \frac{\eta}{(1-\beta)} \frac{1}{n} \sum_{i=1}^{n}\nabla f_i(\bx^{(t)}_i)} \notag \\
&= - \lragnle{\nabla f(\btx^{(t)}) , \frac{\eta}{(1-\beta)} \frac{1}{n} \sum_{i=1}^{n} \Big(\nabla f_i(\bx^{(t)}_i) - \nabla f_i(\btx^{(t)}) + \nabla f_i(\btx^{(t)})\Big)} \notag \\
&= - \lragnle{\nabla f(\btx^{(t)}) , \frac{\eta}{(1-\beta)} \nabla f(\btx^{(t)})} + \frac{\eta}{(1-\beta)} \frac{1}{n} \sum_{i=1}^{n}\lragnle{\nabla f(\btx^{(t)}) ,  \nabla f_i(\btx^{(t)}) - \nabla f_i(\bx^{(t)}_i)} \notag \\
&\stackrel{\text{(b)}}{\leq} - \frac{\eta}{(1-\beta)} \verts{\nabla f(\btx^{(t)})}^2 + \frac{\eta}{2(1-\beta)} \verts{\nabla f(\btx^{(t)})}^2 + \frac{\eta}{2(1-\beta)}\frac{1}{n}\sum_{i=1}^n\verts{\nabla f_i(\btx^{(t)}) - \nabla f_i(\bx^{(t)}_i)}^2 \notag \\
&\stackrel{\text{(c)}}{\leq} - \frac{\eta}{2(1-\beta)} \verts{\nabla f(\btx^{(t)})}^2 + \frac{\eta L^2}{2n(1-\beta)}\sum_{i=1}^n\verts{\btx^{(t)} - \bx^{(t)}_i}^2, \notag
\end{align}
where (b) follows from $\lragnle{\ba,\bb} \leq \frac{1}{2}(\|\ba\|^2+\|\bb\|^2)$ and (c) follows from the $L$-smoothness of $f_i$.

For bounding $P_2$, we will use Proposition~\ref{prop:variance-reduction_relaxed}.
\begin{align}
P_2 &= \mathbb{E}_{\xi_{(t)}} \verts{\frac{1}{n} \sum_{i=1}^{n} \nabla F_i(\bx^{(t)}_i, \xi^{(t)}_i)}^2 \notag \\
&\stackrel{\text{(d)}}{=} \mathbb{E}_{\xi_{(t)}} \verts{\frac{1}{n} \sum_{i=1}^{n} \nabla \big(F_i(\bx^{(t)}_i, \xi^{(t)}_i) - \nabla f_i(\bx^{(t)}_i)\big)}^2 + \verts{\frac{1}{n} \sum_{i=1}^{n} \nabla f_i(\bx^{(t)}_i) }^2 \notag \\
&\stackrel{\text{(e)}}{\leq} \frac{\sigma^2}{n} + \frac{M^2}{n^2}\sum_{i=1}^n\verts{\nabla f_i(\bx^{(t)}_i)}_2^2 + \frac{1}{n} \sum_{i=1}^{n} \verts{\nabla f_i(\bx^{(t)}_i) }^2 \notag \\
&= \frac{\sigma^2}{n} + \frac{(M^2+n)}{n^2}\sum_{i=1}^n\verts{\nabla f_i(\bx^{(t)}_i)}_2^2 \label{bound_P2-interim} \\
&\leq \frac{\sigma^2}{n} + \frac{2(M^2+n)}{n^2}\sum_{i=1}^n\verts{\nabla f_i(\bx^{(t)}_i) - \nabla f_i(\btx^{(t)})}_2^2 + \frac{2(M^2+n)}{n^2}\sum_{i=1}^n\verts{\nabla f_i(\btx^{(t)})}_2^2 \notag \\
&\stackrel{\text{(f)}}{\leq} \frac{\sigma^2}{n} + \frac{2(M^2+n)L^2}{n^2}\sum_{i=1}^n\verts{\bx^{(t)}_i - \btx^{(t)}}_2^2 + \frac{2(M^2+n)}{n}\big(G^2+B^2\verts{\nabla f(\btx^{(t)})}_2^2\big) \notag
\end{align}
Here, (d) follows because the randomness used for sampling the unbiased stochastic gradients across workers is independent of each other, (e) follows from \eqref{eq:variance-reduction_relaxed}, and (f) follows from the $L$-smoothness of $f_i$ and \eqref{eq:grad-dissim}.
\end{proof}

\begin{lemma*}[Restating Lemma~\ref{glob_virt_bound}]
Consider the deviation of the global average parameter $\blx^{(t)}$  and the virtual sequence $\btx^{(t)}$ defined in \eqref{virt_seq} for constant stepsize $\eta$. Then at any time step $t$, the following holds:
\begin{align}\label{eq:glob_virt_bound}
\verts{ \blx^{(t)} -  \btx^{(t)}}^2 & \leq \frac{\beta^4 \eta^2}{(1-\beta)^3}  \sum_{\tau=0}^{t-1} \left[\beta^{t-\tau-1}\verts{\frac{1}{n} \sum_{i=1}^{n} \nabla F_i (\bx^{(\tau)}_i , \xi^{(\tau)}_i ) }^2\right]
\end{align}
\end{lemma*}
\begin{proof}
		Using the definition of $\tilde{\bx}^{(t)}$ as in \eqref{virt_seq}, we have:
		\begin{align*}
		\verts{ \bar{\bx}^{(t)} -  \tilde{\bx}^{(t)}}^2 & = \verts{ \bar{\bx}^{(t)} -  \tilde{\bx}^{(t)}}^2 =  \frac{\beta^4 \eta^2}{(1-\beta)^2} \verts{\frac{1}{n} \sum_{i=1}^{n} \bv^{(t-1)}_i}^2 
		\intertext{Define $\theta_{t-1} = \sum_{k=0}^{t-1} \beta^{1-t-k} = \frac{1-\beta^t}{1-\beta}$. Thus we can expand the term in the norm as: }
		& = \frac{\beta^4 \eta^2}{(1-\beta)^2} \theta_{t-1}^2  \verts{  \sum_{k=0}^{t-1} \frac{\beta^{t-1-k}}{\theta_{t-1}}  \frac{1}{n} \sum_{i=1}^{n} \nabla F(\bx_i^{(k)}, \xi_i^{(k)})  }^2 \\
		& \leq \frac{\beta^4 \eta^2}{(1-\beta)^2} \theta_{t-1}^2   \sum_{k=0}^{t-1} \frac{\beta^{t-1-k}}{\theta_{t-1}}  \verts{ \frac{1}{n} \sum_{i=1}^{n} \nabla F(\bx_i^{(k)}, \xi_i^{(k)})  }^2 \\
		& = \frac{\beta^4 \eta^2}{(1-\beta)^2} \theta_{t-1}   \sum_{k=0}^{t-1} \beta^{t-1-k}  \verts{ \frac{1}{n} \sum_{i=1}^{n} \nabla F(\bx_i^{(k)}, \xi_i^{(k)})  }^2 \\
		& \leq \frac{\beta^4 \eta^2}{(1-\beta)^3}  \sum_{\tau=0}^{t-1} \left[\beta^{t-\tau-1}\verts{\frac{1}{n} \sum_{i=1}^{n} \nabla F_i (\bx^{(\tau)}_i , \xi^{(\tau)}_i ) }^2\right]
		\end{align*}
		Where the first inequality follows from  Jensen's inequality and the second inequality follows from noting that $\theta_{t} \leq \frac{1}{1-\beta}$. This completes the proof.
\end{proof}

\begin{proof}[Proof of Lemma~\ref{lem:bound-decaying-grads}]
We have already bounded the expectation term in \eqref{bound_P2} -- the same bound holds when expectation is taken w.r.t.\ the entire past. Substituting that bound -- i.e., \\ $\bbE\verts{\frac{1}{n} \sum_{i=1}^{n} \nabla F_i (\bx^{(\tau)}_i , \xi^{(\tau)}_i )}^2 \leq \frac{\sigma^2}{n} + \frac{(M^2+n)}{n^2}\sum_{i=1}^n\bbE\verts{\nabla f_i(\bx^{(\tau)}_i)}_2^2$ -- from \eqref{bound_P2-interim} into \eqref{eq:bound-decaying-grads} gives
\begin{align}
&\frac{1}{T}\sum_{t=0}^{T-1} \sum_{\tau=0}^{t-1} \left[\beta^{t-\tau-1}\bbE\verts{\frac{1}{n} \sum_{i=1}^{n} \nabla F_i (\bx^{(\tau)}_i , \xi^{(\tau)}_i ) }^2\right] \leq \frac{1}{T}\sum_{t=0}^{T-1} \sum_{\tau=0}^{t-1} \beta^{t-\tau-1}\frac{\sigma^2}{n} \notag \\
&\hspace{3cm} + \frac{1}{T}\sum_{t=0}^{T-1} \sum_{\tau=0}^{t-1} \beta^{t-\tau-1}\frac{(M^2+n)}{n^2}\sum_{i=1}^n\bbE\verts{\nabla f_i(\bx^{(\tau)}_i)}_2^2 \label{bound-decaying-grads-interim1}
\end{align} 
Now we bound both the terms of \eqref{bound-decaying-grads-interim1} separately.
\begin{align}
\frac{1}{T}\sum_{t=0}^{T-1} \sum_{\tau=0}^{t-1} \beta^{t-\tau-1}\frac{\sigma^2}{n} &= \frac{\sigma^2}{n}\frac{1}{T}\sum_{t=0}^{T-1} \sum_{\tau=0}^{t-1} \beta^{t-\tau-1}\leq\frac{\sigma^2}{n(1-\beta)}. \label{bound-decaying-grads-interim2} \\
\frac{1}{T}\sum_{t=0}^{T-1} \sum_{\tau=0}^{t-1} \beta^{t-\tau-1}\frac{(M^2+n)}{n^2}\sum_{i=1}^n\bbE\verts{\nabla f_i(\bx^{(\tau)}_i)}_2^2 &= \frac{1}{T}\sum_{\tau=0}^{T-2} \sum_{t=\tau+1}^{T-1} \beta^{t-\tau-1}\frac{(M^2+n)}{n^2}\sum_{i=1}^n\bbE\verts{\nabla f_i(\bx^{(\tau)}_i)}_2^2 \notag \\
&\hspace{-7cm}= \frac{(M^2+n)}{n^2}\frac{1}{T}\sum_{\tau=0}^{T-2}\sum_{i=1}^n\bbE\verts{\nabla f_i(\bx^{(\tau)}_i)}_2^2 \sum_{t=\tau+1}^{T-1} \beta^{t-\tau-1} \notag \\
&\hspace{-7cm}\leq \frac{(M^2+n)}{n^2(1-\beta)}\frac{1}{T}\sum_{\tau=0}^{T-2}\sum_{i=1}^n\bbE\verts{\nabla f_i(\bx^{(\tau)}_i)}_2^2 \notag \\
&\hspace{-7cm}\leq \frac{2(M^2+n)}{n^2(1-\beta)}\frac{1}{T}\sum_{\tau=0}^{T-2}\sum_{i=1}^n\bbE\verts{\nabla f_i(\bx^{(\tau)}_i) - \nabla f_i(\blx^{(\tau)})}_2^2 + \frac{2(M^2+n)}{n^2(1-\beta)}\frac{1}{T}\sum_{\tau=0}^{T-2}\sum_{i=1}^n\bbE\verts{\nabla f_i(\blx^{(\tau)})}_2^2 \notag \\
&\hspace{-7cm}\leq \frac{2(M^2+n)}{n^2(1-\beta)}\frac{1}{T}\sum_{\tau=0}^{T-2}\sum_{i=1}^nL^2\bbE\verts{\bx^{(\tau)}_i - \blx^{(\tau)}}_2^2 + \frac{2(M^2+n)}{n(1-\beta)}\frac{1}{T}\sum_{\tau=0}^{T-2}\big(G^2 + B^2\bbE\|\nabla f(\blx^{(\tau)})\|_2^2\big) \notag \\
&\hspace{-7cm}\leq \frac{2(M^2+n)L^2}{n^2(1-\beta)}\frac{1}{T}\sum_{\tau=0}^{T-2}\sum_{i=1}^n\bbE\verts{\bx^{(\tau)}_i - \blx^{(\tau)}}_2^2 + \frac{2(M^2+n)G^2}{n(1-\beta)} + \frac{2(M^2+n)B^2}{n(1-\beta)}\frac{1}{T}\sum_{\tau=0}^{T-2}\bbE\|\nabla f(\blx^{(\tau)})\|_2^2 \label{bound-decaying-grads-interim3}
\end{align}
Substituting the bounds from \eqref{bound-decaying-grads-interim2}, \eqref{bound-decaying-grads-interim3} into \eqref{bound-decaying-grads-interim1} yields \eqref{eq:bound-decaying-grads}, which proves Lemma~\ref{lem:bound-decaying-grads}.
\end{proof}

\subsection{Omitted Details from Section~\ref{subsec:useful-lemmas}}\label{app:useful-lemmas}
\subsubsection{Proof of Lemma~\ref{lem:cvx_e1_relaxed}}\label{app:cvx_e1_relaxed}
In this section we will prove Lemma~\ref{lem:cvx_e1_relaxed}.
\begin{proof}
We show the following bound in Lemma~\ref{supp_lemma_consensus_part1} (provided at the end of this section):
\begin{align}
\bbE\verts{\bX^{((m+1)H)} - \blX^{((m+1)H)}}_F^2 &\leq \vartheta_1 \bbE\verts{\bX^{(mH)} - \blX^{(mH)}}_F^2 +  \vartheta_2\bbE\verts{\bX^{(mH)} - \bhX^{((m+1)H)}}_F^2 \notag \\
&\quad + \vartheta_3\eta^2\bbE\verts{\sum_{t'=mH}^{(m+1)H-1}\beta \bV^{(t')}+\nabla \bF(\bX^{(t')},\bxi^{(t')})}_F^2, \label{lem_cvx_e1_relaxed-interim1}
\end{align}
where $\vartheta_1=(1+\alpha_5^{-1})R_1$, $\vartheta_2=(1+\alpha_5^{-1}) R_2$, and $\vartheta_3=(R_1+R_2)(1+\alpha_5)$.

We want to write the second expectation term $\bbE\verts{\bX^{(mH)} - \bhX^{((m+1)H)}}_F^2$ on the RHS of \eqref{lem_cvx_e1_relaxed-interim1} in terms of $\bbE\verts{\bX^{(mH)} - \bhX^{(mH)}}_F^2$. 
For that, first we define 
\begin{align}\label{eq:sum_momentum-updates}
\bX^{((m+1/2)H)} := \bX^{(mH)} - \eta\sum_{t'=mH}^{(m+1)H-1} \(\beta \bV^{(t')}+\nabla \bF(\bX^{(t')},\bxi^{(t')})\).
\end{align} 
\begin{align}
&\bbE\verts{\bX^{(mH)} - \bhX^{((m+1)H)}}_F^2 = \bbE\verts{\bX^{(mH)} - \(\bhX^{(mH)} + \C\(\bX^{((m+1/2)H)} - \bhX^{(mH)}\)\)}_F^2 \notag \\
&\quad= \bbE\verts{\bX^{((m+1/2)H)} - \bhX^{(mH)} - \C\(\bX^{((m+1/2)H)} - \bhX^{(mH)}\) + \bX^{(mH)} - \bX^{((m+1/2)H)}}_F^2 \notag \\
&\quad\leq (1+\tau_1)(1-\omega)\bbE\verts{\bX^{((m+1/2)H)} - \bhX^{(mH)}}_F^2 + (1+\tau_1^{-1})\bbE\verts{\bX^{(mH)} - \bX^{((m+1/2)H)}}_F^2 \notag \\
&\quad= (1+\tau_1)(1-\omega)\bbE\verts{\bX^{((m+1/2)H)} - \bX^{(mH)} + \bX^{(mH)}- \bhX^{(mH)}}_F^2 \notag \\
&\hspace{3cm}+ (1+\tau_1^{-1})\bbE\verts{\bX^{(mH)} - \bX^{((m+1/2)H)}}_F^2 \notag \\
&\quad\leq (1+\tau_1)(1-\omega)(1+\tau_2)\bbE\verts{\bX^{(mH)} - \bhX^{(mH)}}_F^2 \notag \\
&\hspace{3cm} + \((1+\tau_1^{-1}) + (1+\tau_1)(1-\omega)(1+\tau_2^{-1})\)\bbE\verts{\bX^{(mH)} - \bX^{((m+1/2)H)}}_F^2 \notag \\
&\quad\leq \chi_1\bbE\verts{\bX^{(mH)} - \bhX^{(mH)}}_F^2 + \chi_2\eta^2\bbE\verts{\sum_{t'=mH}^{(m+1)H-1} \(\beta \bV^{(t')}+\nabla \bF(\bX^{(t')},\bxi^{(t')})\)}_F^2, \label{lem_cvx_e1_relaxed-interim3}
\end{align}
where $\chi_1=(1+\tau_1)(1-\omega)(1+\tau_2)$ and $\chi_2=\((1+\tau_1^{-1}) + (1+\tau_1)(1-\omega)(1+\tau_2^{-1})\)$. 

Substituting this back in \eqref{lem_cvx_e1_relaxed-interim1} yields \eqref{eq:cvx_e1_relaxed}, which proves Lemma~\ref{lem:cvx_e1_relaxed}.
\end{proof}

\begin{lemma}\label{supp_lemma_consensus_part1}
We have
	\begin{align*}
		\mathbb{E} \Vert \bX^{((m+1)H)} - \Bar{\bX}^{((m+1)H)} \Vert_F^2 & \leq R_1 (1+\alpha_5^{-1}) \mathbb{E} \left\Vert  \Bar{\bX}^{(mH)} - \bX^{(mH)} \right\Vert^2 + R_2 (1+\alpha_5^{-1}) \mathbb{E} \left\Vert \hat{\bX}^{((m+1)H)} - \bX^{(mH)} \right\Vert^2 \notag \\
		&  + (1+\alpha_5)(R_1+R_2)\eta^2\left\Vert \sum_{t' = (mH)}^{((m+1)H)-1}  ( \beta \bV^{(t')} +   \boldsymbol{\nabla F}(\bX^{(t')}, \boldsymbol{\xi}^{(t')} )) \right\Vert_F^2
	\end{align*}
\end{lemma}
\begin{proof}
	Using the update equations of $\bX^{((m+1)H)}$ in matrix form given in \eqref{mat_not_algo_1}-\eqref{mat_not_algo_2} in Section \ref{prelim}, we have:
	\begin{align}
		\Vert \bX^{((m+1)H)} - \Bar{\bX}^{((m+1)H)} \Vert_F^2 & = \Vert \bX^{((m+\nicefrac{1}{2})H)} - \Bar{\bX}^{((m+1)H)} + \gamma \hat{\bX}^{((m+1)H)} (\mathbf{W}-\mathbf{I}) \Vert_F^2 \notag 
		\intertext{Noting that $\Bar{\bX}^{((m+1)H)} = \Bar{\bX}^{((m+\nicefrac{1}{2})H)} $ (from (\ref{mean_seq_iter})) and $\Bar{\bX}^{((m+\nicefrac{1}{2})H)} (\mathbf{W}-\mathbf{I})= 0$ (from (\ref{mean_prop})), we get: }
		\Vert \bX^{((m+1)H)} - \Bar{\bX}^{((m+1)H)} \Vert_F^2 & = 
		\Vert (\bX^{((m+\nicefrac{1}{2})H)} - \Bar{\bX}^{((m+\nicefrac{1}{2})H)})((1-\gamma)\mathbf{I}\notag \\
		& \qquad + \gamma \mathbf{W})+ \gamma (\hat{\bX}^{((m+1)H)}-\bX^{((m+\nicefrac{1}{2})H)}) (\mathbf{W}-\mathbf{I}) \Vert_F^2 \notag
	\end{align}
	For any positive constant\footnote{\label{foot:frob_bound}For any two matrices $\mathbf{A},\mathbf{B} \in \mathbb{R}^{p \times q} $ and for any $\alpha >0$ , we have the following relationship for the Frobenius norm:
		\begin{align*}
			\verts{\mathbf{A} + \mathbf{B} }_F^2 \leq (1 + \alpha) \verts{\mathbf{A}}_F^2 + (1 + \alpha^{-1})\verts{\mathbf{B}}_F^2		
	\end{align*}} $\alpha_1$, we have:	
	\begin{align}
		\Vert \bX^{((m+1)H)} - \Bar{\bX}^{((m+1)H)} \Vert_F^2 & \leq (1+\alpha_1)\Vert (\bX^{((m+\nicefrac{1}{2})H)} - \Bar{\bX}^{((m+\nicefrac{1}{2})H)})((1-\gamma)\mathbf{I} + \gamma \mathbf{W})\Vert_F^2 \notag \\
		& \qquad \qquad \qquad  + (1+\alpha_1^{-1}) \Vert \gamma (\hat{\bX}^{((m+1)H)}-\bX^{((m+\nicefrac{1}{2})H)}) (\mathbf{W}-\mathbf{I}) \Vert_F^2 \notag \\
		\intertext{Using $\Vert \mathbf{A} \mathbf{B} \Vert_F \leq \Vert \mathbf{A} \Vert_F \Vert \mathbf{B} \Vert_2 $ for any matrices $\mathbf{A},\mathbf{B}$, we have: }
		\Vert \bX^{((m+1)H)} - \Bar{\bX}^{((m+1)H)} \Vert_F^2 &  \leq (1+\alpha_1)\Vert (\bX^{((m+\nicefrac{1}{2})H)} - \Bar{\bX}^{((m+\nicefrac{1}{2})H)})((1-\gamma)\mathbf{I} + \gamma \mathbf{W})\Vert_F^2 \notag \\
		& \quad  + (1+\alpha_1^{-1}) \gamma^2 \Vert  (\hat{\bX}^{((m+1)H)}-\bX^{((m+\nicefrac{1}{2})H)})\Vert_F^2 .\Vert(\mathbf{W}-\mathbf{I})\Vert_2^2 \label{suppl_cvx_li_temp_eqn}
	\end{align}
	To bound the first term in (\ref{suppl_cvx_li_temp_eqn}), we use the triangle inequality for Frobenius norm, giving us:
	\begin{align*}
		\Vert (\bX^{((m+\nicefrac{1}{2})H)} - \Bar{\bX}^{((m+\nicefrac{1}{2})H)})((1-\gamma)\mathbf{I} + \gamma \mathbf{W})\Vert_F & \leq (1-\gamma)\Vert \bX^{((m+\nicefrac{1}{2})H)} - \Bar{\bX}^{((m+\nicefrac{1}{2})H)} \Vert_F \\
		&+ \gamma \Vert (\bX^{((m+\nicefrac{1}{2})H)} - \Bar{\bX}^{((m+\nicefrac{1}{2})H)} )\mathbf{W}\Vert_F
	\end{align*}
	{Since $\left(\bX^{((m+\nicefrac{1}{2})H)} - \Bar{\bX}^{((m+\nicefrac{1}{2})H)}\right)\frac{\mathbf{1}\mathbf{1}^T}{n} = \bzero$ (from \eqref{mean_prop}), adding this inside the last term above, we get: }
	\begin{align*}
		\Vert (\bX^{((m+\nicefrac{1}{2})H)} - \Bar{\bX}^{((m+\nicefrac{1}{2})H)})((1-\gamma)\mathbf{I} + \gamma \mathbf{W})\Vert_F & \leq (1-\gamma)\Vert \bX^{((m+\nicefrac{1}{2})H)} - \Bar{\bX}^{((m+\nicefrac{1}{2})H)} \Vert_F \\
		& + \gamma \left\Vert (\bX^{((m+\nicefrac{1}{2})H)} - \Bar{\bX}^{((m+\nicefrac{1}{2})H)})\left (\mathbf{W} - \frac{\mathbf{1}\mathbf{1}^T}{n} \right) \right\Vert_F 
	\end{align*}
	Using $\Vert \mathbf{A} \mathbf{B} \Vert_F \leq \Vert \mathbf{A} \Vert_F \Vert \mathbf{B} \Vert_2 $ and then using (\ref{bound_W_mat}) from Fact 3 with $k=1$, we can simplify the above to:
	\begin{align*}
		\Vert (\bX^{((m+\nicefrac{1}{2})H)} - \Bar{\bX}^{((m+\nicefrac{1}{2})H)})((1-\gamma)\mathbf{I} + \gamma \mathbf{W})\Vert_F \leq (1-\gamma \delta) \Vert \bX^{((m+\nicefrac{1}{2})H)} - \Bar{\bX}^{((m+\nicefrac{1}{2})H)} \Vert_F
	\end{align*}
	Substituting the above in (\ref{suppl_cvx_li_temp_eqn}) and using $\lambda = \text{max}_i \{ 1- \lambda_i(\mathbf{W}) \} \Rightarrow \Vert \mathbf{W}-\mathbf{I} \Vert_2^2 \leq \lambda^2 $, we get: 
	\begin{align*}
		\Vert \bX^{((m+1)H)} - \Bar{\bX}^{((m+1)H)} \Vert_F^2 & \leq (1+\alpha_1)(1-\gamma \delta)^2  \Vert \bX^{((m+\nicefrac{1}{2})H)} - \Bar{\bX}^{((m+\nicefrac{1}{2})H)} \Vert_F^2 \\
		& \quad + (1+\alpha_1^{-1}) \gamma^2 \lambda^2 \Vert  {\bX}^{((m+\nicefrac{1}{2})H)}-\hat{\bX}^{((m+1)H)}\Vert_F^2
	\end{align*}
	Taking expectation w.r.t.\ the entire process, we have:
	\begin{align*}
		\mathbb{E} \Vert \bX^{((m+1)H)} - \Bar{\bX}^{((m+1)H)} \Vert_F^2 & \leq (1+\alpha_1)(1-\gamma \delta)^2 \mathbb{E} \Vert \bX^{((m+\nicefrac{1}{2})H)} - \Bar{\bX}^{((m+\nicefrac{1}{2})H)} \Vert_F^2 \\
		& \qquad  + (1+\alpha_1^{-1}) \gamma^2 \lambda^2 \mathbb{E} \Vert  {\bX}^{((m+\nicefrac{1}{2})H)}-\hat{\bX}^{((m+1)H)}\Vert_F^2
	\end{align*}
	Define $R_1=(1+\alpha_1)(1-\gamma \delta)^2 , R_2=(1+\alpha_1^{-1}) \gamma^2 \lambda^2 $.
	Using the update steps of algorithm given in equations \eqref{mat_not_algo_4} and \eqref{mean_seq_iter} (given in Section \ref{prelim}), we have:
	\begin{align*}
		\mathbb{E} \Vert \bX^{((m+1)H)} - \Bar{\bX}^{((m+1)H)} \Vert_F^2 & \leq R_1 \mathbb{E} \left\Vert  \Bar{\bX}^{(mH)} - \bX^{(mH)} -  \sum_{t' = mH}^{(m+1)H-1} \eta ( \beta \bV^{(t')} +  \boldsymbol{\nabla F}(\bX^{(t')}, \boldsymbol{\xi}^{(t')} ))\left(   \frac{\mathbf{1}\mathbf{1}^T}{n} - I \right) \right\Vert_F^2 \notag \\
		& \qquad + R_2 \mathbb{E} \left\Vert   \hat{\bX}^{((m+1)H)} - \bX^{(mH)} + \sum_{t' = mH}^{(m+1)H-1} \eta ( \beta \bV^{(t')} + \boldsymbol{\nabla F}(\bX^{(t')}, \boldsymbol{\xi}^{(t')} )) \right\Vert_F^2
	\end{align*} 	
	Thus, for any $\alpha_5 > 0 $ (using Footnote \ref{foot:frob_bound}), we have:
	\begin{align*}
		\mathbb{E} \Vert \bX^{((m+1)H)} - \Bar{\bX}^{((m+1)H)} \Vert_F^2 & \leq R_1 (1+\alpha_5^{-1}) \mathbb{E} \left\Vert  \Bar{\bX}^{(mH)} - \bX^{(mH)} \right\Vert^2 + R_2 (1+\alpha_5^{-1}) \mathbb{E} \left\Vert \hat{\bX}^{((m+1)H)} - \bX^{(mH)} \right\Vert^2 \notag \\
		& \qquad + R_1(1+\alpha_5)  \mathbb{E} \left\Vert \sum_{t' = (mH)}^{((m+1)H)-1} \eta ( \beta \bV^{(t')} +   \boldsymbol{\nabla F}(\bX^{(t')}, \boldsymbol{\xi}^{(t')} ))\left(   \frac{{\bf 1}{\bf 1}^T}{n} - I \right) \right\Vert_F^2 \notag \\
		& \qquad + R_2 (1+\alpha_5) \mathbb{E} \left\Vert \sum_{t' = (mH)}^{((m+1)H)-1} \eta ( \beta \bV^{(t')} +   \boldsymbol{\nabla F}(\bX^{(t')}, \boldsymbol{\xi}^{(t')} )) \right\Vert_F^2
	\end{align*}
	Using $\verts{\mathbf{AB}}_F \leq \verts{\mathbf{A}}_F \verts{\mathbf{B}}_2$ to split the third term, and then using the bound $\left\| \frac{\mathbf{1}\mathbf{1}^T}{n} - \mathbf{I} \right\|_2 =1$ (which is shown in Claim \ref{J-I_eigenvalue} in Appendix \ref{app:details-stronger-assump} in supplementary), the above can be rewritten as:  
	\begin{align*}
		\mathbb{E} \Vert \bX^{((m+1)H)} - \Bar{\bX}^{((m+1)H)} \Vert_F^2 & \leq R_1 (1+\alpha_5^{-1}) \mathbb{E} \left\Vert  \Bar{\bX}^{(mH)} - \bX^{(mH)} \right\Vert^2 + R_2 (1+\alpha_5^{-1}) \mathbb{E} \left\Vert \hat{\bX}^{((m+1)H)} - \bX^{(mH)} \right\Vert^2 \notag \\
		&  + (1+\alpha_5)(R_1+R_2)\eta^2\left\Vert \sum_{t' = (mH)}^{((m+1)H)-1}  ( \beta \bV^{(t')} +   \boldsymbol{\nabla F}(\bX^{(t')}, \boldsymbol{\xi}^{(t')} )) \right\Vert_F^2
	\end{align*}
\end{proof}

\subsubsection{Proof of Lemma~\ref{lem:cvx_e2_relaxed}}\label{app:cvx_e2_relaxed}
In this section, we prove Lemma~\ref{lem:cvx_e2_relaxed}.
\begin{proof}
\begin{align}
&\bbE\verts{\bX^{((m+1)H)} - \bhX^{((m+1)H)}}_F^2 = \bbE\verts{\bX^{((m+1)H)} - \(\bhX^{(mH)} + \C\(\bX^{((m+1/2)H)} - \bhX^{(mH)}\)\)}_F^2 \notag \\
&\qquad= \bbE\verts{\bX^{((m+1/2)H)} - \bhX^{(mH)} - \C\(\bX^{((m+1/2)H)} - \bhX^{(mH)}\) + \bX^{((m+1)H)} - \bX^{((m+1/2)H)}}_F^2 \notag \\
&\qquad\leq (1+\tau_3)(1-\omega)\underbrace{\bbE\verts{\bX^{((m+1/2)H)} - \bhX^{(mH)}}_F^2}_{=:\ T_1} + (1+\tau_3^{-1})\underbrace{\bbE\verts{\bX^{((m+1)H)} - \bX^{((m+1/2)H)}}_F^2}_{=:\ T_2} \label{cvx_e2_relaxed-interim2}
\end{align}
Now we bound $T_1$ and $T_2$.
\begin{align}
T_1 &= \bbE\verts{\bX^{((m+1/2)H)} - \bhX^{(mH)}}_F^2 \notag \\
&= \bbE\verts{\bX^{(mH)} - \eta\sum_{t'=mH}^{(m+1)H-1} \(\beta \bV^{(t')}+\nabla \bF(\bX^{(t')},\bxi^{(t')})\) - \bhX^{(mH)}}_F^2 \notag \\
&\leq (1+\tau_4)\bbE\verts{\bX^{(mH)} - \bhX^{(mH)}}_F^2 + (1+\tau_4^{-1})\eta^2\bbE\verts{\sum_{t'=mH}^{(m+1)H-1} \(\beta \bV^{(t')}+\nabla \bF(\bX^{(t')},\bxi^{(t')})\)}_F^2 \label{cvx_e2_relaxed-interim3}
\end{align}
\begin{align}
T_2 &= \bbE\verts{\bX^{((m+1)H)} - \bX^{((m+1/2)H)}}_F^2 \notag \\
&= \bbE\verts{\bX^{((m+1/2)H)} + \gamma\bhX^{((m+1)H)}(\bW-\bI) - \bX^{((m+1/2)H)}}_F^2 \notag \\
&= \gamma^2\bbE\verts{\bhX^{((m+1)H)}(\bW-\bI)}_F^2 \notag \\
&= \gamma^2\bbE\verts{\(\bhX^{((m+1)H)} - \blX^{((m+1/2)H)}\)(\bW-\bI)}_F^2 \tag{Since $\blX^{((m+1/2)H)}(\bW-\bI)={\bf 0}$} \\
&\leq \gamma^2\lambda^2\bbE\verts{\bhX^{((m+1)H)} - \blX^{((m+1/2)H)}}_F^2 \tag{Since $\|\bW-\bI\|_2=\lambda$} \\
&= \gamma^2\lambda^2\bbE\verts{\bhX^{((m+1)H)} - \(\blX^{(mH)} - \eta\sum_{t'=mH}^{(m+1)H-1} \(\beta \bV^{(t')}+\nabla \bF(\bX^{(t')},\bxi^{(t')})\)\)}_F^2 \notag \\
&\leq \phi_1\underbrace{\bbE\verts{\bhX^{((m+1)H)} - \blX^{(mH)}}_F^2}_{=:\ T_3} + \phi_2\eta^2\bbE\verts{\sum_{t'=mH}^{(m+1)H-1} \(\beta \bV^{(t')}+\nabla \bF(\bX^{(t')},\bxi^{(t')})\)}_F^2, \label{cvx_e2_relaxed-interim4}
\end{align}
where $\phi_1=\gamma^2\lambda^2(1+\tau_5)$ and $\phi_2=\gamma^2\lambda^2(1+\tau_5^{-1})$.
\begin{align}
T_3 &= \bbE\verts{\bhX^{((m+1)H)} - \blX^{(mH)}}_F^2 \notag \\
&= \bbE\verts{\bhX^{((m+1)H)} - \bX^{(mH)} + \bX^{(mH)} - \blX^{(mH)}}_F^2 \notag \\
&\leq (1+\tau_6)\bbE\verts{\bX^{(mH)} - \blX^{(mH)}}_F^2 + (1+\tau_6^{-1})\bbE\verts{\bhX^{((m+1)H)} - \bX^{(mH)}}_F^2 \notag \\
&\stackrel{\text{(a)}}{\leq} (1+\tau_6)\bbE\verts{\bX^{(mH)} - \blX^{(mH)}}_F^2 + (1+\tau_6^{-1})(1+\tau_7)(1-\omega)(1+\tau_8)\bbE\verts{\bX^{(mH)} - \bhX^{(mH)}}_F^2 \notag \\
&\hspace{1cm} + \phi\eta^2\bbE\verts{\sum_{t'=mH}^{(m+1)H-1} \(\beta \bV^{(t')}+\nabla \bF(\bX^{(t')},\bxi^{(t')})\)}_F^2, \label{lem_cvx_e1_relaxed-interim5}
\end{align}
where $\phi_3=(1+\tau_6^{-1})\((1+\tau_7^{-1}) + (1+\tau_7)(1-\omega)(1+\tau_8^{-1})\)$, (a) follows from \eqref{lem_cvx_e1_relaxed-interim3} for bounding the term $\bbE\Vert \bhX^{((m+1)H)} - \bX^{(mH)} \Vert_F^2$. Observe that since we are bounding this quantity separately for (a), we can use different coefficients here. In the above bound on $\bbE \Vert\bhX^{((m+1)H)} - \bX^{(mH)} \Vert_F^2$ from \eqref{lem_cvx_e1_relaxed-interim3}, instead of using the same $\tau_1,\tau_2$, we used $\tau_7,\tau_8$, respectively.

Substituting the above bound on $T_3$ into \eqref{cvx_e2_relaxed-interim4} and the substituting the resulting bound on $T_2$ from \eqref{cvx_e2_relaxed-interim4} and on $T_1$ from \eqref{cvx_e2_relaxed-interim3} into \eqref{cvx_e2_relaxed-interim2} gives
\begin{align}
\bbE\verts{\bX^{((m+1)H)} - \bhX^{((m+1)H)}}_F^2 &\leq b_1\bbE\verts{\bX^{(mH)} - \blX^{(mH)}}_F^2 + b_2\bbE\verts{\bX^{(mH)} - \bhX^{(mH)}}_F^2 \notag \\
&\quad + b_3\eta^2\bbE\verts{\sum_{t'=mH}^{(m+1)H-1} \(\beta \bV^{(t')}+\nabla \bF(\bX^{(t')},\bxi^{(t')})\)}_F^2,
\end{align}
where $b_1 = (1+\tau_3^{-1})\gamma^2\lambda^2(1+\tau_5)(1+\tau_6)$,
$b_2 = (1+\tau_3)(1-\omega)(1+\tau_4) + (1+\tau_3^{-1})\gamma^2\lambda^2(1+\tau_5)(1+\tau_6^{-1})(1+\tau_7)(1-\omega)(1+\tau_8)$,
$b_3 = (1+\tau_3)(1-\omega)(1+\tau_4^{-1}) + (1+\tau_3^{-1})\gamma^2\lambda^2(1+\tau_5)(1+\tau_6^{-1})\((1+\tau_7^{-1}) + (1+\tau_7)(1-\omega)(1+\tau_8^{-1})\) + (1+\tau_3^{-1})\gamma^2\lambda^2(1+\tau_5^{-1})$.
\end{proof}

\subsection{Setting up parameters}\label{app:setting-params}
We need to set the parameters such that we get $(1+\nu_1)\max\{a_1+b_1,a_2+b_2\}<1$, this will give a contractive recursion in \eqref{equi-16-interim5} and will lead to our convergence results. Recall the definitions of $a_1,a_2$ and $b_1,b_2$ from Lemma~\ref{lem:cvx_e1_relaxed} and Lemma~\ref{lem:cvx_e2_relaxed}, respectively.
\begin{align}
a_1 &= (1+\alpha_5^{-1})(1+\alpha_1)(1-\gamma \delta)^2, \label{params-setup-a1} \\
a_2 &= (1+\alpha_5^{-1}) (1+\alpha_1^{-1}) \gamma^2 \lambda^2 (1+\tau_1)(1-\omega)(1+\tau_2), \label{params-setup-a2} \\
b_1 &= (1+\tau_3^{-1})\gamma^2\lambda^2(1+\tau_5)(1+\tau_6), \label{params-setup-b1} \\
b_2 &= (1+\tau_3)(1-\omega)(1+\tau_4) + (1+\tau_3^{-1})\gamma^2\lambda^2(1+\tau_5)(1+\tau_6^{-1})(1+\tau_7)(1-\omega)(1+\tau_8). \label{params-setup-b2} 
\end{align}
Here, $\omega,\delta,\lambda$ are fixed parameters and are given to us. Among the rest, there is no trade-off when choosing $\alpha_5,\tau_1,\tau_2,\tau_4,\tau_5,\tau_7,\tau_8$, and we can chose them without any constraints. We need to carefully choose the remaining parameters $\alpha_1,\tau_3,\tau_6,\gamma$ as they contribute differently to different terms in the above equations. We will set all these parameters as follows:
\begin{align}
&\tau_i = \frac{\omega}{4}, \text{ for } i=1,2,3,4,5,7,8; \quad \tau_6 = \frac{4}{\omega}; \label{params-setting-taus} \\
&\alpha_1 = \frac{\gamma\delta}{2}; \quad \alpha_5^{-1} = \frac{\gamma\delta}{2}; \quad \gamma^* = \frac{2\delta\omega^3}{(128\lambda^2 + 24\lambda^2\omega^2 + 4\delta^2\omega^2)}. \label{params-setting-alphas-gamma}
\end{align}
Now we substitute these values into \eqref{params-setup-a1}-\eqref{params-setup-b2}. 
\begin{itemize}
\item For $a_1$, we will use $\alpha_5^{-1} \leq \frac{\gamma\delta}{2}$ and $(1+\frac{\gamma\delta}{2})(1-\gamma \delta) \leq (1-\frac{\gamma\delta}{2})$ (since $\gamma\delta\leq1$ which is true for $\gamma=\gamma^*$).
\begin{align}\label{params-setup-a1-bound}
a_1 &\leq (1+\frac{\gamma\delta}{2})^2(1-\gamma \delta)^2 \leq (1-\frac{\gamma\delta}{2})^2.
\end{align}
\item For $a_2$, we will use $\alpha_5^{-1} \leq \frac{\omega}{4}$ (which holds because $\frac{\gamma\delta}{2}\leq\frac{\omega}{4}$ for $\gamma=\gamma^*$), $(1+\frac{\omega}{4})^3(1-\omega) \leq (1-\frac{\omega}{4})$, and $\frac{1}{\gamma\delta}\geq1$.
\begin{align}\label{params-setup-a2-bound}
a_2 &\leq (1+\frac{\omega}{4})(1+\frac{2}{\gamma\delta})\gamma^2 \lambda^2 (1+\frac{\omega}{4})(1-\omega)(1+\frac{\omega}{4}) \leq \frac{3\gamma\lambda^2}{\delta}(1-\frac{\omega}{4}).
\end{align}
\item For $b_1$, we will use $(1+\frac{4}{\omega})\leq\frac{5}{\omega}$, $(1+\frac{\omega}{4})\leq\frac{5}{4}$, and $\frac{125}{4}\leq 32$.
\begin{align}\label{params-setup-b1-bound}
b_1 &= (1+\frac{4}{\omega})\gamma^2\lambda^2(1+\frac{\omega}{4})(1+\frac{4}{\omega}) \leq \gamma^2\lambda^2\frac{25}{\omega^2}\frac{5}{4} \leq \gamma^2\lambda^2\frac{32}{\omega^2}.
\end{align}
\item For $b_2$, we will use $(1+\frac{\omega}{4})^2(1-\omega) \leq (1+\frac{\omega}{4})^3(1-\omega) \leq (1-\frac{\omega}{4})$ in the first inequality, and $(1+\frac{4}{\omega})\leq\frac{5}{\omega}$ and $(1+\frac{\omega}{4})\leq\frac{5}{4}$ in the second inequality.
\begin{align}\label{params-setup-b2-bound}
b_2 &= (1+\frac{\omega}{4})^2(1-\omega) + (1+\frac{4}{\omega})\gamma^2\lambda^2(1+\frac{\omega}{4})^4(1-\omega) \notag \\
&\leq (1-\frac{\omega}{4}) + (1+\frac{4}{\omega})\gamma^2\lambda^2(1+\frac{\omega}{4})(1-\frac{\omega}{4}) \notag \\
&\leq (1-\frac{\omega}{4})\(1+\frac{5}{\omega}\gamma^2\lambda^2\frac{5}{4}\) \notag\\
&= (1-\frac{\omega}{4})\(1+\gamma^2\lambda^2\frac{25}{4\omega}\).
\end{align}
\end{itemize}
\paragraph{Bounding $(a_1+b_1)$.} Adding the bounds in \eqref{params-setup-a1-bound} and \eqref{params-setup-b1-bound}, we get
\begin{align}\label{params-setup-a1+b1-bound-interim1}
a_1+b_1 &\leq \underbrace{(1-\frac{\gamma\delta}{2})^2 + \gamma^2\lambda^2\frac{32}{\omega^2}}_{=:\ h_1(\gamma)}.
\end{align}
It can be verified that $h_1(\gamma)$ is a convex function in $\gamma$ and attains minima at $\gamma'=\frac{2\delta\omega^2}{128\lambda^2+\delta^2\omega^2}$ with value $h_1(\gamma')=\frac{128\lambda^2}{128\lambda^2+\delta^2\omega^2}<1$. 

Putting this $\gamma'$ in the expression for $a_2+b_2$ will not give a quantity that is less than one. In the following, we will derive a value of $\gamma^*$ that works for both $a_1+b_1$ and $a_2+b_2$. Let $\gamma^*=s\gamma'$ for some $s\in[0,1]$. We will derive the value of $s$ (and of $\gamma^*$). 

By the convexity of $h$, we have
\begin{align}\label{params-setup-a1+b1-bound-interim2}
h_1(\gamma^*) &= h_1(s\gamma') = h_1((1-s)0+s\gamma') \notag \\
&\leq (1-s)h_1(0) + sh_1(\gamma') \notag \\
&\leq (1-s) + s\frac{128\lambda^2}{128\lambda^2+\delta^2\omega^2} \notag \\
&= 1 - s\frac{\delta^2\omega^2}{128\lambda^2+\delta^2\omega^2}.
\end{align}
\paragraph{Bounding $(a_2+b_2)$.} Adding the bounds in \eqref{params-setup-a2-bound} and \eqref{params-setup-b2-bound} gives:
\begin{align}\label{params-setup-a2+b2-bound-interim1}
a_2+b_2 &\leq (1-\frac{\omega}{4})\(1+\frac{3\gamma\lambda^2}{\delta} + \gamma^2\lambda^2\frac{25}{4\omega}\) \notag \\
&\leq \underbrace{(1-\frac{\omega}{4})+ \(\frac{3\gamma\lambda^2}{\delta} + \gamma^2\lambda^2\frac{25}{4\omega}\)}_{=:\ h_2(\gamma)}.
\end{align}
Putting $\gamma=\gamma^*=s\gamma' = \frac{2\delta\omega^2s}{D}$, where $D=(128\lambda^2+\delta^2\omega^2)$, we get
\begin{align}\label{params-setup-a2+b2-bound-interim2}
h_2(\gamma^*) &\leq (1-\frac{\omega}{4}) + \(3\lambda^2\frac{2\omega^2s}{D} + \frac{25\lambda^2}{4\omega}\frac{4\delta^2\omega^4s^2}{D^2}\) \notag \\
&\leq (1-\frac{\omega}{4}) + \frac{s}{D}\(6\lambda^2\omega^2 + \frac{25\lambda^2\delta^2\omega^3s}{D}\) \notag \\
&\leq (1-\frac{\omega}{4}) + \frac{s}{D}\(6\lambda^2\omega^2 + 25\lambda^2\) \tag{Since $D\geq\delta^2\omega^2\geq\delta^2\omega^3s$ because $\omega,s\leq1$} \\
&\leq (1-\frac{\omega}{4}) + \frac{s}{D}\(6\lambda^2\omega^2 + 32\lambda^2\).
\end{align}
Equating the upper bounds on $h_1(\gamma^*)$ and $h_2(\gamma^*)$, we get
\begin{align}\label{value-of-s}
1 - s\frac{\delta^2\omega^2}{D} &= (1-\frac{\omega}{4}) + \frac{s}{D}\(6\lambda^2\omega^2 + 32\lambda^2\) \notag \\
\iff \frac{\omega}{4} &= \frac{s}{D}(32\lambda^2 + 6\lambda^2\omega^2 + \delta^2\omega^2) \notag \\
\iff s &= \frac{\omega D}{(128\lambda^2 + 24\lambda^2\omega^2 + 4\delta^2\omega^2)} <1.
\end{align}
With this, we have $\gamma^*=s\gamma'=\frac{2\delta\omega^2s}{D}=\frac{2\delta\omega^3}{(128\lambda^2 + 24\lambda^2\omega^2 + 4\delta^2\omega^2)}$.

Substituting the value of $s$ from \eqref{value-of-s} into \eqref{params-setup-a1+b1-bound-interim2}, we get
\begin{align}\label{final-bound-h_1}
h_1(\gamma^*) \leq 1 - \frac{\delta^2\omega^3}{(128\lambda^2 + 24\lambda^2\omega^2 + 4\delta^2\omega^2)} = 1-\frac{\gamma^*\delta}{2}.
\end{align}
%
Thus we have
\begin{align*}
\max\{a_1+b_1,a_2+b_2\} \leq \max\{h_1(\gamma^*),h_2(\gamma^*)\} \leq 1-\frac{\gamma^*\delta}{2}.
\end{align*}
Taking $\nu_1=\frac{\gamma^*\delta}{4}$ and using the inequality $(1+\nicefrac{x}{2})(1-x)\leq(1-\nicefrac{x}{2})$ (for $x=\frac{\gamma^*\delta}{2}\leq1$), we get
\begin{align}\label{final-bound_coeff_St}
(1+\nu_1)\max\{a_1+b_1,a_2+b_2\} \leq 1-\frac{\gamma^*\delta}{4} \leq 1- \frac{\delta^2\omega^3}{1224},
\end{align}
where the last inequality follows by substituting the trivial upper bounds of
$\lambda\leq2$ and $\delta,\omega\leq1$ in the denominator of the expression of $\gamma^*$.

\paragraph{Bounding $c_2+c_4$ in \eqref{equi-16-interim5}.}
\begin{align}
c_2 &= 2(1+\nu_1)(a_{31}+a_{32}) + 2(1+\nu_1^{-1}), \label{bound-on-c2} \\
c_4 &= 2(1+\nu_1)(b_{31}+b_{32}+b_{33}) + 2(1+\nu_1^{-1}), \label{bound-on-c4}
\end{align}
where
\begin{align}
a_{31} &= (1+\alpha_1)(1-\gamma \delta)^2(1+\alpha_5) + (1+\alpha_1^{-1}) \gamma^2 \lambda^2(1+\alpha_5), \label{bound-on-a31} \\
a_{32} &= (1+\alpha_5^{-1}) (1+\alpha_1^{-1}) \gamma^2 \lambda^2 \((1+\tau_1^{-1}) + (1+\tau_1)(1-\omega)(1+\tau_2^{-1})\), \label{bound-on-a32} \\
b_{31} &= (1+\tau_3)(1-\omega)(1+\tau_4^{-1}), \label{bound-on-b31} \\
b_{32} &= (1+\tau_3^{-1})\gamma^2\lambda^2(1+\tau_5)(1+\tau_6^{-1})\((1+\tau_7^{-1}) + (1+\tau_7)(1-\omega)(1+\tau_8^{-1})\), \label{bound-on-b32} \\
b_{33} &= (1+\tau_3^{-1})\gamma^2\lambda^2(1+\tau_5^{-1}). \label{bound-on-b33}
\end{align}
Now we substituting the parameter setting from \eqref{params-setting-taus}, \eqref{params-setting-alphas-gamma} into the above equations.
\begin{itemize}
\item For $a_{31}$, we will use $(1+\frac{\gamma\delta}{2})(1-\gamma\delta)^2\leq(1-\frac{\gamma\delta}{2})(1-\gamma\delta)\leq1$ and $(1+\frac{2}{\gamma\delta})\leq\frac{3}{\gamma\delta}$ (both follow from $\gamma\delta\leq1$).
\begin{align}\label{bound-on-a31}
a_{31} &= (1+\frac{\gamma\delta}{2})(1-\gamma\delta)^2(1+\frac{2}{\gamma\delta}) + (1+\frac{2}{\gamma\delta})^2\gamma^2\lambda^2 \notag \\
&\leq \frac{3}{\gamma\delta} + (\frac{3}{\gamma\delta})^2\gamma^2\lambda^2 = \frac{3}{\gamma\delta}\(1+\frac{3\gamma\lambda^2}{\delta}\)
\end{align}
\item For $a_{32}$, we will use $(1+\frac{\gamma\delta}{2})\leq\frac{3}{2}$, $(1+\frac{2}{\gamma\delta})\leq\frac{3}{\gamma\delta}$, and $(1+\frac{\omega}{4})(1-\omega)\leq(1-\frac{3\omega}{4})\leq1$ and $(1+\frac{4}{\omega})\leq\frac{5}{\omega}$.
\begin{align}\label{bound-on-a32}
a_{32} &= (1+\frac{\gamma\delta}{2}) (1+\frac{2}{\gamma\delta}) \gamma^2 \lambda^2 \((1+\frac{4}{\omega}) + (1+\frac{\omega}{4})(1-\omega)(1+\frac{4}{\omega})\) \notag \\
&\leq \frac{3}{2} \frac{3}{\gamma\delta} \gamma^2 \lambda^2 \frac{10}{\omega} =\frac{45\gamma\lambda^2}{\delta\omega}.
\end{align}
\item For $b_{31}$, we will use $(1+\frac{\omega}{4})(1-\omega)\leq(1-\frac{3\omega}{4})$.
\begin{align}\label{bound-on-b31}
b_{31} &= (1+\frac{\omega}{4})(1-\omega)(1+\frac{4}{\omega}) \leq (1-\frac{3\omega}{4})(1+\frac{4}{\omega}) \leq \frac{4}{\omega} - 2.
\end{align}
\item For $b_{32}$, we will use $(1+\frac{4}{\omega})\leq \frac{5}{\omega}$, $(1+\frac{\omega}{4})\leq\frac{5}{4}$, and $\((1+\frac{4}{\omega}) + (1+\frac{\omega}{4})(1-\omega)(1+\frac{4}{\omega})\)\leq\frac{10}{\omega}$ as in $a_{32}$.
\begin{align}\label{bound-on-b32}
b_{32} &= (1+\frac{4}{\omega})\gamma^2\lambda^2(1+\frac{\omega}{4})(1+\frac{\omega}{4})\((1+\frac{4}{\omega}) + (1+\frac{\omega}{4})(1-\omega)(1+\frac{4}{\omega})\) \notag \\
&\leq \frac{5}{\omega}\gamma^2\lambda^2(\frac{5}{4})^2\frac{10}{\omega} = \frac{625}{8}\frac{\gamma^2\lambda^2}{\omega^2} \leq \frac{79\gamma^2\lambda^2}{\omega^2}.
\end{align}
\item For $b_{33}$, we will use 
\begin{align}\label{bound-on-b33}
b_{33} &= (1+\frac{4}{\omega})\gamma^2\lambda^2(1+\frac{4}{\omega}) \leq \frac{25\gamma^2\lambda^2}{\omega^2}.
\end{align}
\end{itemize}
Substituting the bounds on $a_{31},a_{32}$ from \eqref{bound-on-a31}, \eqref{bound-on-a32}, respectively, and $\nu_1=\frac{\gamma\delta}{4}$ (where $\gamma=\gamma^*$ is defined in \eqref{params-setting-alphas-gamma}) into \eqref{bound-on-c2}, we get:
\begin{align}\label{final-bound-c2}
c_2 &\leq 2(1+\frac{\gamma\delta}{4})\(\frac{3}{\gamma\delta}\(1+\frac{3\gamma\lambda^2}{\delta}\) + \frac{45\gamma\lambda^2}{\delta\omega}\) + 2(1+\frac{4}{\gamma\delta}).
\end{align}
Similarly, substituting the bounds on $b_{31},b_{32},b_{33}$ from \eqref{bound-on-b31}, \eqref{bound-on-b32}, \eqref{bound-on-b33}, respectively, and $\nu_1=\frac{\gamma\delta}{4}$ (where $\gamma=\gamma^*$ is defined in \eqref{params-setting-alphas-gamma}) into \eqref{bound-on-c4}, we get:
\begin{align}\label{final-bound-c4}
c_4 &\leq 2(1+\frac{\gamma\delta}{4})\(\frac{4}{\omega} - 2 + \frac{104\gamma^2\lambda^2}{\omega^2}\) + 2(1+\frac{4}{\gamma\delta}).
\end{align}
Adding the bounds on $c_2$ and $c_4$ gives
\begin{align}\label{final-bound-c2+c4}
c_2+c_4 &\leq 2(1+\frac{\gamma\delta}{4})\(\frac{3}{\gamma\delta} + \frac{9\lambda^2}{\delta^2} + \frac{45\gamma\lambda^2}{\delta\omega} + \frac{104\gamma^2\lambda^2}{\omega^2} + \frac{4}{\omega} - 2\) + 4(1+\frac{4}{\gamma\delta}).
\end{align}
Putting the bounds from \eqref{final-bound_coeff_St} and \eqref{final-bound-c2+c4} back into \eqref{equi-16-interim5}, we get
\begin{align}
S^{(t)} &\leq \(1-\frac{\gamma\delta}{4}\)S^{(mH)} + 2c_1\eta^2H^2n\(2(M^2+1)G^2 + \sigma^2\) + c_1\eta^2H\beta^2\sum_{t'=mH}^{t-1}\bbE\verts{\bV^{(t')}}_F^2 \notag \\
&\quad + 2c_1\eta^2H(M^2+1)L^2\sum_{t'=mH}^{t-1}S^{(t')} + 2c_1\eta^2H(M^2+1)nB^2\sum_{t'=mH}^{t-1}\bbE\verts{\nabla f(\blx^{(t')})}_2^2, \label{equi-16-final-bound}
\end{align} 
where $c_1=c_2+c_4$ and the bound on $c_2+c_4$ is given in \eqref{final-bound-c2+c4}, and $\gamma=\gamma^*$ is defined in \eqref{params-setting-alphas-gamma}.

\subsection{Omitted Details from Section~\ref{subsec:proof-similar-eq16}}\label{app:proof-similar-eq16}

\begin{proof}[Proof of Proposition~\ref{prop:mom-update-norm}]
\begin{align}
&\bbE\verts{\nabla \bF(\bX^{(t')},\bxi^{(t')})}_F^2 = \bbE\verts{\nabla f(\bX^{(t')})}_F^2 + \bbE\verts{\nabla \bF(\bX^{(t')},\bxi^{(t')}) - \nabla f(\bX^{(t')})}_F^2 \notag \\
&= \bbE\verts{\nabla f(\bX^{(t')})}_F^2 + \bbE\sum_{i=1}^n\verts{\nabla F(\bx_i^{(t')},\xi_i^{(t')}) - \nabla f(\bx_i^{(t')})}_2^2 \notag \\
&\stackrel{\text{(a)}}{\leq} \bbE\verts{\nabla f(\bX^{(t')})}_F^2 + n\sigma^2 + M^2\bbE\verts{\nabla f(\bX^{(t')})}_F^2 \notag \\
&= (M^2+1)\bbE\verts{\nabla f(\bX^{(t')})}_F^2 + n\sigma^2 \notag \\
&= (M^2+1)\bbE\verts{\nabla f(\bX^{(t')}) - \nabla f(\blX^{(t')}) + \nabla f(\blX^{(t')})}_F^2 + n\sigma^2 \tag{Where $\nabla f(\blX^{(t')})=[\nabla f_1(\blx^{(t')})\hdots \nabla f_n(\blx^{(t')})]$} \\
&\leq 2(M^2+1)\(\bbE\verts{\nabla f(\bX^{(t')}) - \nabla f(\blX^{(t')})}_F^2 + \bbE\verts{\nabla f(\blX^{(t')})}_F^2\) + n\sigma^2 \notag \\
&\stackrel{\text{(b)}}{\leq} 2(M^2+1)\(L^2\bbE\verts{\bX^{(t')} - \blX^{(t')}}_F^2 + \bbE\sum_{i=1}^n\verts{\nabla f_i(\blx^{(t')})}_2^2\) + n\sigma^2 \notag \\
&\stackrel{\text{(c)}}{\leq} 2(M^2+1)\(L^2\bbE\verts{\bX^{(t')} - \blX^{(t')}}_F^2 + nG^2+nB^2\bbE\verts{\nabla f(\blx^{(t')})}_2^2\) + n\sigma^2 \notag \\
&= 2(M^2+1)\(L^2\Xi^{(t')} + nG^2+nB^2\bbE\verts{\nabla f(\blx^{(t')})}_2^2\) + n\sigma^2 \notag
\end{align}
where (a) follows from Assumption~\ref{assump:variance}, (b) follows from the $L$-smoothness of $f$, and (c) follows from Assumption~\ref{assump:grad-dissim}.
\end{proof}

\subsection{Omitted Details from Section~\ref{subsec:proof-similar-lem14}}\label{app:proof-similar-lem14}
	\begin{claim}\label{claim:lem14-interim-claim}
	We have $\left(1 - \frac{\alpha}{4} \right)^{\lfloor \frac{t-j}{H} \rfloor} \leq  2 \left( 1 - \frac{\alpha}{8H} \right)^{t-j}$.
	\end{claim}
	\begin{proof}
	First note that $(1-\frac{\alpha}{4})^{1/H}\leq \exp(-\frac{\alpha}{4H})\leq 1-\frac{\alpha}{8H}$ and also that $\lfloor \frac{t-j}{H}\rfloor \geq \frac{t-j}{H} - 1$.
	\begin{align*}
	\(1-\frac{\alpha}{4}\)^{\lfloor \frac{t-j}{H} \rfloor} &= \left[\(1-\frac{\alpha}{4}\)^{1/H}\right]^{H\lfloor \frac{t-j}{H}\rfloor} \leq \(1-\frac{\alpha}{8H}\)^{H\lfloor \frac{t-j}{H}\rfloor} \\
	&\leq \(1-\frac{\alpha}{8H}\)^{t-j}\(1-\frac{\alpha}{8H}\)^{-H} \leq 2\(1-\frac{\alpha}{8H}\)^{t-j}.
	\end{align*}
	In the last inequality we used $\(1-\frac{\alpha}{8H}\)^{-H} \leq 2$, which can be shown as follows:
	\begin{align*}
	\(1-\frac{\alpha}{8H}\)^{-H} &= \(\frac{1}{1-\frac{\alpha}{8H}}\)^{H} \stackrel{\text{(a)}}{\leq} \(1+\frac{\alpha}{4H}\)^H \leq \exp(\frac{\alpha}{4}) \leq 2,
	\end{align*}
	where (a) holds because $\frac{\alpha}{8H}\leq\frac{1}{2}$.
	\end{proof}
	
\subsection{Completing the Convergence Proof}\label{app:completing-proof}
Note that $\Xi^{(t)}\sum_{i=1}^n\bbE\verts{\bx^{(t)}_i - \blx^{(t)}}_2^2 \leq S^{(t)}$ for any $t\in[T]$. Substituting this and the bound from \eqref{upper-bound-avg-St} in the last term of \eqref{mom_noncvx_interim13}, we get
\begin{align}
\frac{1}{T}\sum_{t=0}^{T-1}\bbE\verts{\nabla f(\blx^{(t)})}_2^2 &\leq \frac{16(1-\beta)(f(\blx^{(0)}) - f^*)}{\eta T} + \frac{16\eta L}{(1-\beta)}\Big(\frac{\sigma^2+2(M^2+n)G^2}{n}\Big) \notag \\
&\quad + \eta^2\frac{128 L^2J_1}{n} + \eta^2\frac{128 L^2J_2}{n} \frac{1}{T}\sum_{t=0}^{T-1}  \bbE \verts{ \nabla f (\blx^{(t)}) }^2.
\end{align}
where $J_1= \left(\frac{8A\eta^2}{\alpha} + \left( \frac{32 D H}{\alpha} \right)  \left( \frac{2(M^2+1)nG^2 + n \sigma^2}{(1-\beta)}   \right)   \right)$ and $J_2=\left(\frac{32C  H}{\alpha}   + \left( \frac{32 D H }{\alpha} \right) \frac{2(M^2+1)nB^2}{(1-\beta)}\right)$, \\ $A=2c_1H^2n\(2(M^2+1)G^2 + \sigma^2\)$, $C=2c_1H(M^2+1)nB^2$, and $D=\frac{c_1H\beta^2}{(1-\beta)}$ and $c_1$ defined below. 
If $\eta\leq\sqrt{\frac{n}{256L^2J_2}}$, then taking the last term on the LHS gives
\begin{align}
\frac{1}{T}\sum_{t=0}^{T-1}\bbE\verts{\nabla f(\blx^{(t)})}_2^2 &\leq \frac{32(1-\beta)(f(\blx^{(0)}) - f^*)}{\eta T} + \frac{32\eta L}{(1-\beta)}\Big(\frac{\sigma^2+2(M^2+n)G^2}{n}\Big) \notag \\
&\quad + \eta^2\frac{256 L^2J_1}{n}.
\end{align}
Choosing $\eta = (1-\beta)\sqrt{\frac{n}{T}}$ and running the algorithm for $T \geq \max \{U_1, U_2 , U_3 , U_4 , U_5   \}  $ iterations
completes the proof of Theorem~\ref{convergence_relaxed_assump}. \\
Here, $U_1 =  \frac{81n\beta^8}{4(1-\beta)^4}$, $U_2 =\frac{9(M^2+n)\beta^4 L^2}{4(1-\beta^2)} $, $U_3 = \frac{72 (M^2+n) \beta^2 L^2 B^2}{(1-\beta)^2}$, $U_4 = 256L^2 J_2 (1-\beta)^2$ and $U_5 = \frac{512DH (M^2+1) L^2 (1-\beta)n}{\delta \gamma}$, with
$J_2  = \frac{128CH}{\gamma \delta} + \left(\frac{128DH}{\gamma \delta}\right)\left( \frac{2(M^2+1)nB^2}{1-\beta}  \right) $, $D = \frac{c_1H \beta^2}{(1-\beta)}$,\\ $C = 2c_1 H(M^2+1)nB^2$ and $c_1 = 2 ( 1 + \frac{\gamma \delta}{4} ) \left( \frac{3}{\gamma \delta} + \frac{9 \lambda^2}{\delta^2} + \frac{45 \gamma \lambda^2}{\delta \omega} + \frac{104 \gamma^2 \lambda^2 }{\omega^2} + \frac{4}{\omega} -2 \right) $.

\section{Preliminaries for Convergence with Relaxed Assumptions}\label{app:details-stronger-assump}


{\allowdisplaybreaks

	\begin{fact}\label{bound_var}
		Consider the variance bound on the stochastic gradient for nodes $i \in [n]$: $$\mathbb{E}_{\xi_i} \verts{\nabla F_i (\bx, \xi_i) - \nabla f_i (\bx) }^2 \leq \sigma_i^2,$$ where $\mathbb{E}_{\xi_i} [\nabla F_i(\bx, \xi_i)] = \nabla f_i(\bx) $, then:
		\begin{align}\label{eq:bound_var}
		\mathbb{E}_{\boldsymbol{\xi}^{(t)}} \left\Vert \frac{1}{n}\sum_{j=1}^n \left(\nabla f_j(\bx_j^{(t)}) - \nabla F_j(\bx_j^{(t)},\xi_j^{(t)}) \right)\right\Vert^2 \leq \frac{\Bar{\sigma}^2}{n}
		\end{align}
		where ${\boldsymbol{\xi}^{(t)}} = \{ \xi_1^{(t)}, \xi_2^{(t)}, \hdots, \xi_n^{(t)} \}$ denotes the stochastic sample for the nodes at any timestep $t$ and $\frac{ \sum_{j=1}^n \sigma_j^2 }{n}= {\Bar{\sigma}^2}$
	\end{fact}
		\begin{proof}
		\begin{align*}
		&\mathbb{E}_{\xi^{(t)}}\left\Vert \frac{1}{n}\sum_{j=1}^n \nabla f_j(\bx_j^{(t)}) - \frac{1}{n}\sum_{j=1}^n \nabla F_j(\bx_j^{(t)},\xi_j^{(t)}) \right\Vert^2  = 
		\frac{1}{n^2} \sum_{j=1}^n \mathbb{E}_{\xi^{(t)}} \Vert \nabla f_j(\bx_j^{(t)}) - \nabla F_j(\bx_j^{(t)},\xi_j^{(t)})\Vert^2 \\
		& \hspace{4.1cm} + \frac{1}{n^2}\sum_{i \neq j} \mathbb{E}_{\xi^{(t)}} \left\langle \nabla f_i(\bx_i^{(t)}) - \nabla F_i(\bx_i^{(t)},\xi_j^{(t)}), \nabla f_j(\bx_j^{(t)}) - \nabla F_j(\bx_j^{(t)},\xi_j^{(t)}) \right\rangle 
		\end{align*}
		Since $\xi_i$ is independent of $\xi_j$, the second term is zero in expectation, thus the above reduces to:
		\begin{align*}
		\mathbb{E}_{\xi^{(t)}}\left\Vert \frac{1}{n}\sum_{j=1}^n \nabla f_j(\bx_j^{(t)}) - \frac{1}{n}\sum_{j=1}^n \nabla F_j(\bx_j^{(t)},\xi_j^{(t)}) \right\Vert^2 & = \frac{1}{n^2} \sum_{j=1}^n \mathbb{E}_{\xi^{(t)}} \Vert \nabla f_j(\bx_j^{(t)}) - \nabla F_j(\bx_j^{(t)},\xi_j^{(t)})\Vert^2 \\
		& \leq \frac{1}{n^2} \sum_{j=1}^n \sigma_j^2 = \frac{\Bar{\sigma}^2}{n}
		\end{align*}
	\end{proof}

	\begin{fact}\label{fact-bound}
	Consider the set of synchronization indices $\{ I_{(1)},I_{(2)}, \hdots, I_{(k)}, \hdots  \}  \in \mathcal{I}_T $. We assume that the maximum gap between any two consecitive elements in  $\mathcal{I}_T$ is bounded by $H$. Let ${\xi^{(t)}} = \{ \xi_1^{(t)}, \xi_2^{(t)}, \hdots, \xi_n^{(t)} \}$ denote the stochastic samples for the nodes at any timestep $t$. Consider any two consecutive synchronization indices $I_{(k)}$ and $I_{(k+1)}$, then for learning rate $\eta$, we have:
	\begin{align}\label{bound_gap_grad}
	\mathbb{E} \left[\left\Vert \sum_{t' = I_{(k)}}^{I_{(k+1)}-1}\eta ( \beta \bV^{(t')} + \boldsymbol{\nabla F} (\bX^{(t')}, \boldsymbol{\xi}^{(t')})) \right\Vert_F^2\right] \leq 2 n H^2 G^2 \eta^2 \left( 1 + \frac{\beta^2}{(1-\beta)^2} \right).
	\end{align}
	\end{fact}
	\begin{proof}
		Using the fact that the sequence gap is bounded by $H$, we have $I_{(t+1)}- I_{(t)} \leq H$ for all synchronization indices $I_{(t)} \in \mathcal{I}_T$. Thus we have:
		\begin{align}
		\mathbb{E} \left[\left\Vert \sum_{t' = I_{(k)}}^{I_{(k+1)}-1}\eta ( \beta \bV^{(t')} + \boldsymbol{\nabla F} (\bX^{(t')}, \boldsymbol{\xi}^{(t')})) \right\Vert_F^2\right] & \leq H \eta^2 \sum_{t' = I_{(k)}}^{I_{(k+1)}-1} \mathbb{E} \verts{\beta \bV^{(t')} + \boldsymbol{\nabla F} (\bX^{(t')}, \boldsymbol{\xi}^{(t')})}_F^2 \notag \\
		& \leq 2 H \eta^2 \sum_{t' = I_{(k)}}^{I_{(k+1)}-1} \left[\mathbb{E} \verts{\beta \bV^{(t')}}_F^2 + \mathbb{E}\verts{\boldsymbol{\nabla F} (\bX^{(t')}, \boldsymbol{\xi}^{(t')})}_F^2\right] \notag
		\intertext{Using the bounded gradient assumption and definition of gap $H$, we can bound the above as:}
		\mathbb{E} \left[\left\Vert \sum_{t' = I_{(k)}}^{I_{(k+1)}-1}\eta ( \beta \bV^{(t')} + \boldsymbol{\nabla F} (\bX^{(t')}, \boldsymbol{\xi}^{(t')})) \right\Vert_F^2\right] & \leq 2 H \eta^2 \beta^2 \sum_{t' = I_{(k)}}^{I_{(k+1)}-1} \mathbb{E} \verts{\bV^{(t')}}_F^2 + 2nH^2G^ 2 \eta^2 \notag \\
		=& 2 H \eta^2 \beta^2 \sum_{t' = I_{(k)}}^{I_{(k+1)}-1} \sum_{i=1}^n\mathbb{E} \verts{\bv_i^{(t')}}^2 + 2nH^2G^ 2 \eta^2 \label{fact-bound_interim2}
		\end{align}
		Now we show that $\mathbb{E} \verts{\bv_i^{(t)}}^2 \leq \frac{G^2}{(1-\beta)^2}$ for all $i\in[n]$ and for every $t\geq0$. Fix an arbitrary $i\in[n]$ and $t\geq0$. 
		Define $\theta_{t} = \sum_{k=0}^{t} \beta^k$, we then have:
		\begin{align}
		\mathbb{E} \verts{\bv^{(t)}_i}^2 & =\theta_{t}^2  \mathbb{E}\verts{  \sum_{k=0}^{t} \frac{\beta^{t-k}}{\theta_{t}}  \nabla F(\bx_i^{(k)}, \xi_i^{(k)})  }^2 \notag \\
		& \leq \theta_{t} \sum_{k=0}^{t} \beta^{t-k}  \mathbb{E} \verts{ \nabla F(\bx_i^{(k)}, \xi_i^{(k)})  }^2 \notag \\
		& \leq \theta_t  \sum_{k=0}^{t} \left[\beta^{t-k} G^2\right] \notag \\
		& = G^2 \theta_t^2 \notag 
		\end{align}
		Here the first inequality follows from the Jensen's inequality and the second inequality follows from the bounded gradient assumption. We now note the following bound for $\theta_t$:
		\begin{align*}
		\theta_t & = \sum_{k=0}^{t} \beta^k  \leq  \sum_{k=0}^{\infty} \beta^k  \leq \frac{1}{(1-\beta)} 
		\end{align*}
		Thus, for all $t$ and all $i \in [n]$, we have:
		\begin{align} \label{bound_v_it}
		\mathbb{E} \verts{\bv^{(t)}_i}^2 \leq \frac{G^2}{(1-\beta)^2} 
		\end{align}

		Substituting the bound $\bbE\|\bv_i^{(t)}\|^2 \leq \frac{G^2}{(1-\beta)^2}$ in \eqref{fact-bound_interim2} gives 
		\begin{align*}
		\mathbb{E} \left[\left\Vert \sum_{t' = I_{(k)}}^{I_{(k+1)}-1}\eta ( \beta \bV^{(t')} + \boldsymbol{\nabla F} (\bX^{(t')}, \boldsymbol{\xi}^{(t')})) \right\Vert_F^2\right] & \leq 2 H^2 \eta^2 \beta^2 n \frac{G^2}{(1-\beta)^2} + 2nH^2G^ 2 \eta^2.
		\end{align*}
		This completes the proof of Fact \ref{fact-bound}.
	\end{proof}
	
\begin{fact} [Triggering rule, \cite{singh2019sparq}]
Consider the set of nodes $ {\Gamma}^{(t)}  $ which do not communicate at time $t$. 
For a threshold sequence $\{c_t\}_{t=0}^{T-1}$, the triggering rule in Algorithm \ref{alg_dec_sgd_li} dictates that $$\Vert \bx_i^{(t+\frac{1}{2})} - \hat{\bx}_i^{(t)} \Vert^2 \leq  {c_t \eta^2}  \hspace{1cm} \forall i \in \Gamma^{(t)}.$$ 
Using the matrix notation, this implies that:
	\begin{align} \label{bound_trig}
	\verts{ (\bX^{(t+\frac{1}{2})}  - \hat{\bX}^{(t)}) (\mathbf{I} - \mathbf{P}^{(t)}) }_F^2 \leq nc_t \eta^2.
	\end{align}
\end{fact}
\begin{fact} [Lemma 16, \cite{koloskova_decentralized_2019-1}] 
	For doubly stochastic matrix $\mathbf{W}$ with second largest eigenvalue $1- \delta = |\lambda_2(\mathbf{W}) | <1 $, we have:
	\begin{align} \label{bound_W_mat}
	\left\Vert \mathbf{W}^k - \frac{1}{n} \mathbf{\mathbf{1}}\mathbf{\mathbf{1}}^T  \right\Vert = (1-\delta)^k
	\end{align}
	for any non-negative integer $k$.
\end{fact}

	\begin{claim}\label{J-I_eigenvalue}
		For any $n\in\mathbb{N}$, we have $\left\| \frac{{\bf 1}{\bf 1}^T}{n} - \mathbf{I} \right\|_2 = 1$ where ${\bf 1} = [ 1 \, 1 \hdots 1]^T_{1 \times n} $
	\end{claim}
	\begin{proof}
		Note that $ \frac{{\bf 1}{\bf 1}^T}{n} $ is a symmetric doubly stochastic matrix with eigenvalues 1 and 0 (with algebraic multiplicity $n-1$). Thus, it has the eigen-decomposition $ \frac{{\bf 1}{\bf 1}^T}{n} =\mathbf{ U D U}^T  $ where columns of $\mathbf{U}$ are orthogonal and $\mathbf{D}$ = $diag( [ 1 \, 0 \hdots 0]  )$, which gives us:
		\begin{align*}
		\left\| \frac{{\bf 1}{\bf 1}^T}{n} - \mathbf{I} \right\|_2 = \left\| \mathbf{UDU}^T - \mathbf{UU}^T  \right\|_2 = \left\| \mathbf{D}-\mathbf{I} \right\|_2=
		\left\|
		\begin{bmatrix}
		1 & 0 & \hdots  & 0 \\
		0 & 0 & \hdots & 0 \\
		\vdots & \vdots & \ddots & 0 \\
		0 & 0 & \hdots & 0 
		\end{bmatrix}
		- 
		\begin{bmatrix}
		1 & 0 & \hdots  & 0 \\
		0 & 1 & \hdots & 0 \\
		\vdots & \vdots & \ddots & 0 \\
		0 & 0 & \hdots & 1 
		\end{bmatrix} \right\|_2 = 1
		\end{align*}
\end{proof}

}

\section{Proof of Theorem \ref{convergence_results} (Non-convex objective)} \label{app:proof_thm_noncvx}

{\allowdisplaybreaks
From the recurrence relation of the virtual sequence \eqref{virt_seq_rec}, we have:
\begin{align} \label{mom_noncvx_interim6}
	\mathbb{E}_{\xi_{(t)}} [f(\btx^{(t+1)})] &= \mathbb{E}_{\xi_{(t)}} f \left(\btx^{(t)}  - \frac{\eta}{(1-\beta)} \frac{1}{n} \sum_{i=1}^{n} \nabla F_i(\bx^{(t)}_i, \xi^{(t)}_i) \right) \notag \\
	& \leq f(\btx^{(t)}) - \lragnle{\nabla f(\btx^{(t)}) , \frac{\eta}{(1-\beta)} \frac{1}{n} \sum_{i=1}^{n} \mathbb{E}_{\xi_{(t)}} [\nabla F_i(\bx^{(t)}_i, \xi^{(t)}_i)] } \notag \\
	& \qquad + \frac{L}{2} \frac{\eta^2}{(1-\beta)^2} \mathbb{E}_{\xi_{(t)}} \verts{\frac{1}{n} \sum_{i=1}^{n} \nabla F_i(\bx^{(t)}_i, \xi^{(t)}_i)}^2 \notag \\
	& \leq f(\btx^{(t)}) - \lragnle{\nabla f(\btx^{(t)}) , \frac{\eta}{(1-\beta)} \frac{1}{n} \sum_{i=1}^{n} \nabla f_i(\bx^{(t)}_i) } + \frac{L}{2} \frac{\eta^2}{(1-\beta)^2} \verts{\frac{1}{n} \sum_{i=1}^{n} \nabla f_i(\bx^{(t)}_i)}^2 \notag \\ 
	 & \qquad + \frac{L}{2} \frac{\eta^2}{(1-\beta)^2} \mathbb{E}_{\xi_{(t)}}\verts{\frac{1}{n} \sum_{i=1}^{n}  (  \nabla f_i(\bx^{(t)}_i) - \nabla F_i(\bx^{(t)}_i, \xi^{(t)}_i) }^2 \notag \\
	& \leq f(\btx^{(t)}) - \lragnle{\nabla f(\btx^{(t)}) , \frac{\eta}{(1-\beta)} \frac{1}{n} \sum_{i=1}^{n} \nabla f_i(\bx^{(t)}_i) } + \frac{L}{2} \frac{\eta^2}{(1-\beta)^2} \verts{\frac{1}{n} \sum_{i=1}^{n} \nabla f_i(\bx^{(t)}_i)}^2 \notag \\ 
	 & \hspace{1cm } + \frac{L\eta^2 \bar{\sigma}^2}{2n(1-\beta)^2}
\end{align}
We now focus on bounding the second term in \eqref{mom_noncvx_interim6}. First, note the following:
\begin{align} \label{mom_noncvx_dotprod-1}
	\lragnle{\nabla f(\btx^{(t)}) , \frac{1}{n} \sum_{i=1}^{n} \nabla f_i(\bx^{(t)}_i) } & = \verts{\frac{1}{n} \sum_{i=1}^{n}  \nabla f_i (\bx^{(t)}_i) }^2 - \lragnle{  \frac{1}{n} \sum_{i=1}^{n} \nabla f_i(\bx^{(t)}_i) {-} \nabla f(\btx^{(t)}) , \frac{1}{n} \sum_{i=1}^{n} \nabla f_i(\bx^{(t)}_i) } \notag \\
	& = \verts{\frac{1}{n} \sum_{i=1}^{n}  \nabla f_i (\bx^{(t)}_i) }^2 {-} \lragnle{  \frac{1}{n} \sum_{i=1}^{n}( \nabla f_i(\bx^{(t)}_i) {-} \nabla f_i(\btx^{(t)})) , \frac{1}{n} \sum_{i=1}^{n} \nabla f_i(\bx^{(t)}_i) } \notag \\
	& \geq   \frac{1}{2} \verts{\frac{1}{n} \sum_{i=1}^{n}  \nabla f_i (\bx^{(t)}_i) }^2 - \frac{L^2}{2n} \sum_{i=1}^{n} \verts{ \bx^{(t)}_i -  \btx^{(t)}}^2		
\end{align}
where in the last inequality, we've used the fact that $2\langle \mathbf{a},\mathbf{b} \rangle \leq \Vert \mathbf{a} \Vert^2 + \Vert \mathbf{b}\Vert^2$ for any $\mathbf{a},\mathbf{b} \in \mathbb{R}^d$ and the $L-$smoothness assumption for objectives $\{f_i\}_{i=1}^n$. We now state how to bound the last term on R.H.S. of \eqref{mom_noncvx_dotprod-1}. First, note the bound:
\begin{align} \label{mom_noncvx_interim4}
	\sum_{i=1}^{n} \verts{ \bx^{(t)}_i -  \btx^{(t)}}^2 \leq 2\sum_{i=1}^{n} \verts{ \bx^{(t)}_i -  \blx^{(t)}}^2 + 2\sum_{i=1}^{n} \verts{ \blx^{(t)} -  \btx^{(t)}}^2
\end{align}
\noindent Using Lemma \ref{glob_virt_bound} to bound the second term in \eqref{mom_noncvx_interim4}, we get:
\begin{align} \label{mom_noncvx_interim5}
\sum_{i=1}^{n} \verts{ \bx^{(t)}_i {-}  \btx^{(t)}}^2 \leq 2\sum_{i=1}^{n} \verts{ \bx^{(t)}_i {-}  \blx^{(t)}}^2 + \frac{2n \beta^4 \eta^2}{(1{-}\beta)^3}  \sum_{\tau=0}^{t-1} \left[\beta^{t-\tau-1}\verts{\frac{1}{n} \sum_{i=1}^{n} \nabla F_i (\bx^{(\tau)}_i , \xi^{(\tau)}_i ) }^2\right]
\end{align}
Using the bound \eqref{mom_noncvx_interim5} in \eqref{mom_noncvx_dotprod-1} and substituting it in \eqref{mom_noncvx_interim6}, we have the following bound:
\begin{align*}
	\mathbb{E}_{\xi_{(t)}}[f(&\btx^{(t+1)})] \leq f(\btx^{(t)}) + \frac{L\eta^2 \bar{\sigma}^2}{2n(1{-}\beta)^2} +  \frac{L\eta^2}{2(1{-}\beta)^2} \verts{\frac{1}{n} \sum_{i=1}^{n} \nabla f_i(\bx^{(t)}_i)}^2 {-} \frac{\eta}{2(1{-}\beta)}  \verts{\frac{1}{n} \sum_{i=1}^{n}  \nabla f_i (\bx^{(t)}_i) }^2 \\
	& + \frac{\eta}{(1{-}\beta)} \frac{L^2}{n}\sum_{i=1}^{n} \verts{ \bx^{(t)}_i -  \blx^{(t)}}^2 + \frac{L^2 \eta^3 \beta^4}{ (1-\beta)^4 } \sum_{\tau=0}^{t-1} \left[\beta^{t-\tau-1}  \mathbb{E}_{\xi_{(t)}} \verts{\frac{1}{n} \sum_{i=1}^{n} \nabla F_i (\bx^{(\tau)}_i , \xi^{(\tau)}_i ) }^2\right]
\end{align*}
Rearranging the terms, we can write:
\begin{align*}
& \left(  \frac{\eta }{2(1-\beta) } - \frac{L\eta^2}{2(1-\beta)^2} \right) \verts{\frac{1}{n} \sum_{i=1}^{n} \nabla f_i(\bx^{(t)}_i)}^2  \leq f(\btx^{(t)}) - 	\mathbb{E}_{\xi_{(t)}}f(\btx^{(t+1)}) + \frac{L\eta^2 \bar{\sigma}^2}{2n(1-\beta)^2} \\
& \qquad  + \frac{L^2 \eta}{ (1-\beta) n}\sum_{i=1}^{n} \verts{ \bx^{(t)}_i -  \blx^{(t)}}^2   + \frac{L^2 \eta^3 \beta^4}{ (1-\beta)^4 } \sum_{\tau=0}^{t-1} \left[\beta^{t-\tau-1} \mathbb{E}_{\xi_{(t)}} \verts{\frac{1}{n} \sum_{i=1}^{n} \nabla F_i (\bx^{(\tau)}_i , \xi^{(\tau)}_i ) }^2\right]
\end{align*}
Summing from $t=0$ to $T$ gives us:
\begin{align} 
	&\left(\frac{\eta }{2(1-\beta) } - \frac{L\eta^2}{2(1-\beta)^2} \right) \sum_{t=0}^{T-1}  \verts{\frac{1}{n} \sum_{i=1}^{n} \nabla f_i(\bx^{(t)}_i)}^2 \notag \\
	  &  \leq   f(\btx^{(0)}) - 	\mathbb{E}_{\xi_{(t)}}   f(\btx^{(T)}) + \frac{L\eta^2 \bar{\sigma}^2T}{2n(1-\beta)^2}+ \frac{L^2 \eta}{ (1-\beta) n} \sum_{t=0}^{T-1} \sum_{i=1}^{n} \mathbb{E} \verts{ \bx^{(t)}_i -  \blx^{(t)}}^2  \notag \\
	& \hspace{3.5cm} + \frac{L^2 \eta^3 \beta^4}{ (1-\beta)^4 } \sum_{t=0}^{T-1} \sum_{\tau=0}^{t-1} \left[\beta^{t-\tau-1}\mathbb{E}_{\xi_{(t)}}  \verts{\frac{1}{n} \sum_{i=1}^{n} \nabla F_i (\bx^{(\tau)}_i , \xi^{(\tau)}_i ) }^2\right] \notag \\
	\intertext{Using the fact that $\mathbb{E}_{\xi_{(t)}} [ \nabla F_i(\bx_i^{(t)},\xi_i^{(t)}) ]  = \nabla f_i(\bx_i^{(t)}) $ for all $i \in [n]$ and for all $t \in [T]$, we have: \linebreak {\small $\mathbb{E}_{\xi_{(t)}} \verts{\frac{1}{n} \sum_{i=1}^{n} \nabla F_i(\bx_i^{(t)},\xi_i^{(t)})  }^2 = \mathbb{E}_{\xi_{(t)}}  \verts{\frac{1}{n}  \sum_{i=1}^{n} \nabla f_i(\bx_i^{(t)})  }^2 + \mathbb{E}_{\xi_{(t)}}  \verts{\frac{1}{n} \sum_{i=1}^{n} (\nabla f_i(\bx^{(t)}) - \nabla F_i(\bx_i^{(t)},\xi_i^{(t)}))  }^2  $}. Using this equation along with the variance bound \eqref{eq:bound_var} from Fact \ref{bound_var}, the fact that $ \sum_{t=0 }^{T-1}  \sum_{\tau = 0}^{t-1}  \beta^{t-\tau - 1} \leq  \nicefrac{T}{1-\beta} $ for $\beta \in (0,1)$ and taking expectation w.r.t. the entire process: }
	& \leq  f(\btx^{(0)}) -  \mathbb{E} 	f(\btx^{(T)}) + \frac{L\eta^2 \bar{\sigma}^2T}{2n(1-\beta)^2}+ \frac{L^2 \eta}{ (1-\beta) n} \sum_{t=0}^{T-1} \sum_{i=1}^{n} \mathbb{E} \verts{ \bx^{(t)}_i -  \blx^{(t)}}^2 \notag \\
	& \hspace{3.5cm}  + \frac{L^2 \eta^3 \beta^4 \bar{\sigma}^2 T }{n (1-\beta)^5 } + \frac{L^2 \eta^3 \beta^4}{ (1-\beta)^4 } \sum_{t=0}^{T-1} \sum_{\tau=0}^{t-1} \left[\beta^{t-\tau-1} \mathbb{E} \verts{\frac{1}{n} \sum_{i=1}^{n} \nabla f_i (\bx^{(\tau)}_i ) }^2\right] \label{mom_noncvx_interim9}
\end{align}
To bound the last term in \eqref{mom_noncvx_interim9}, we note that:
\begin{align*}
	\sum_{t=0}^{T-1} \sum_{\tau=0}^{t-1} \beta^{t-\tau-1} \mathbb{E} \verts{\frac{1}{n} \sum_{i=1}^{n} \nabla f_i (\bx^{(\tau)}_i ) }^2
	& = \sum_{\tau=0}^{T-2} \sum_{t=\tau+1}^{T-1} \beta^{t-\tau-1} \mathbb{E} \verts{\frac{1}{n} \sum_{i=1}^{n} \nabla f_i (\bx^{(\tau)}_i ) }^2 \\
	& \leq \frac{1}{(1-\beta)} \sum_{\tau = 0}^{T-2} \mathbb{E} \verts{\frac{1}{n} \sum_{i=1}^{n} \nabla f_i (\bx^{(\tau)}_i ) }^2 \\
	& \leq \frac{1}{(1-\beta)} \sum_{t=0}^{T-1}  \mathbb{E} \verts{\frac{1}{n} \sum_{i=1}^{n} \nabla f_i (\bx^{(t)}_i ) }^2
\end{align*}
Substituting the above bound in \eqref{mom_noncvx_interim9} and rearranging terms, we finally get:
\begin{align}
	& \left(\frac{\eta }{2(1-\beta) }  -  \frac{L\eta^2}{2(1-\beta)^2} - \frac{L^2 \eta^3 \beta^4}{ (1-\beta)^5 } \right)  \sum_{t=0}^{T-1} \mathbb{E} \verts{\frac{1}{n} \sum_{i=1}^{n} \nabla f_i(\bx^{(t)}_i)}^2 \notag \\
	& \quad \leq f(\btx^{(0)}) {-}  \mathbb{E} 	f(\btx^{(T)}) + \frac{L\eta^2 \bar{\sigma}^2T}{2n(1{-}\beta)^2}+ \frac{L^2 \eta}{ (1-\beta) n} \sum_{t=0}^{T-1} \sum_{i=1}^{n} \mathbb{E} \verts{ \bx^{(t)}_i -  \blx^{(t)}}^2 + \frac{L^2 \eta^3 \beta^4 \bar{\sigma}^2 T }{n (1{-}\beta)^5 } \label{eq_noncvx_techdiff}
\end{align}
If we select $\eta \leq \min \left\{ \frac{(1-\beta)}{4L} , \frac{(1-\beta)^2}{2\sqrt{2} L \beta^2} \right\}  $, it can be shown that $\left(\frac{\eta }{2(1-\beta) } - \frac{L\eta^2}{2(1-\beta)^2} -  \frac{L^2 \eta^3 \beta^4}{ (1-\beta)^5 } \right)  \geq \frac{\eta }{4(1-\beta) } $. This gives:
\begin{align*}
\frac{\eta }{4(1-\beta) }    \sum_{t=0}^{T-1} \mathbb{E} \verts{\frac{1}{n} \sum_{i=1}^{n} \nabla f_i(\bx^{(t)}_i)}^2 
& \leq f(\btx^{(0)}) - 	\mathbb{E}[f(\btx^{(T)})] + \frac{L\eta^2 \bar{\sigma}^2T}{2n(1-\beta)^2}+ + \frac{L^2 \eta^3 \beta^4 \bar{\sigma}^2 T }{n (1-\beta)^5 } \\
& \qquad + \frac{L^2 \eta}{ (1-\beta) n} \sum_{t=0}^{T-1} \sum_{i=1}^{n} \mathbb{E} \verts{ \bx^{(t)}_i -  \blx^{(t)}}^2  
\end{align*}
Multiplying both sides by $\frac{4(1-\beta)}{\eta T}$ and noting that $ \mathbb{E}[f(\btx^{(T)})]  \geq f^*$, we have:
\begin{align} \label{mom_noncvx_interim1}
	\frac{1}{T} \sum_{t=0}^{T-1} \mathbb{E} \verts{\frac{1}{n} \sum_{i=1}^{n} \nabla f_i(\bx^{(t)}_i)}^2 & \leq \frac{4(1-\beta)}{\eta }\frac{(f(\bx^{(0)}) - 	f^*)}{T} + \frac{2L\eta \bar{\sigma}^2}{n   (1-\beta)} \notag \\
	& \qquad + \frac{4L^2 }{ n T} \sum_{t=0}^{T-1} \sum_{i=1}^{n} \mathbb{E} \verts{ \bx^{(t)}_i -  \blx^{(t)}}^2   + \frac{4L^2 \eta^2 \beta^4 \bar{\sigma}^2 }{n (1-\beta)^4 }
\end{align}
Now consider the time average of gradients evaluated at the global average $\blx^{(t)}$:
\begin{align} \label{mom_noncvx_interim10}
\frac{1}{T} \sum_{t=0}^{T-1} \mathbb{E} \verts{ \nabla f(\blx^{(t)})}^2 & =
	\frac{1}{T} \sum_{t=0}^{T-1} \mathbb{E} \verts{\frac{1}{n} \sum_{i=1}^{n} \nabla f_i(\blx^{(t)})}^2 \notag \\
	& = \frac{1}{T} \sum_{t=0}^{T-1} \mathbb{E} \verts{\frac{1}{n} \sum_{i=1}^{n} (\nabla f_i(\blx^{(t)}) - \nabla f_i(\bx^{(t)}_i) )+ \frac{1}{n} \sum_{i=1}^{n}  \nabla f_i(\bx^{(t)}_i)}^2 \notag \\
	& \leq \frac{2}{T} \sum_{t=0}^{T-1} \mathbb{E} \verts{\frac{1}{n} \sum_{i=1}^{n} (\nabla f_i(\blx^{(t)}) {-} \nabla f_i(\bx^{(t)}_i) )}^2 +  \frac{2}{T} \sum_{t=0}^{T-1} \mathbb{E}  \verts{\frac{1}{n} \sum_{i=1}^{n}  \nabla f_i(\bx^{(t)}_i)}^2 \notag \\
	& \leq \frac{2L^2}{nT} \sum_{t=0}^{T-1}  \sum_{i=1}^{n} \mathbb{E} \verts{\blx^{(t)} - \bx^{(t)}_i}^2  +  \frac{2}{T} \sum_{t=0}^{T-1} \mathbb{E}  \verts{\frac{1}{n} \sum_{i=1}^{n}  \nabla f_i(\bx^{(t)}_i)}^2	
\end{align}
where in the first inequality follows from Jensen's inequality and the second inequality follows from the $L-$smoothness assumption. We can bound the last term in \eqref{mom_noncvx_interim10} using \eqref{mom_noncvx_interim1} which gives us:
\begin{align} \label{mom_noncvx_interim2}
	\frac{1}{T} \sum_{t=0}^{T-1} \mathbb{E} \verts{ \nabla f(\blx^{(t)})}^2 & \leq \frac{8(1-\beta)}{\eta }\frac{(f({\bx}^{(0)}) - f^*)}{T} + \frac{4L\eta \bar{\sigma}^2}{n  (1-\beta)} \notag \\
	& \quad + \left(\frac{8L^2 }{ n T} + \frac{2L^2}{nT}  \right) \sum_{t=0}^{T-1} \sum_{i=1}^{n} \mathbb{E}\verts{ \bx^{(t)}_i -  \blx^{(t)}}^2  + \frac{8L^2 \eta^2 \beta^4 \bar{\sigma}^2 }{n (1-\beta)^4 }
\end{align}
Note that in our matrix form, $ \mathbb{E}\verts{\bar{\bX}^{(t)} - \bX^{(t)} }_F^2 = \sum_{i=1}^{n} \mathbb{E}\verts{ \bx^{(t)}_i -  \blx^{(t)}}^2$. Let $I_{(t+1)_0} \in \mathcal{I}_T$ denote the latest synchronization step before or equal to $(t+1)$. Then we have:
\begin{align*}
\bX^{(t+1)} & = \bX^{I_{(t+1)_0}} - \textstyle \sum_{t' = I_{(t+1)_0}}^{t}\eta ( \beta \bV^{(t')} +\boldsymbol{\nabla F}(\bX^{(t')}, \boldsymbol{\xi}^{(t')} )) \\
\Bar{\bX}^{(t+1)} & = \bar{\bX}^{I_{(t+1)_0}} - \textstyle \sum_{t' = I_{(t+1)_0}}^{t}\eta ( \beta \bV^{(t')} +\boldsymbol{\nabla F}(\bX^{(t')}, \boldsymbol{\xi}^{(t')} )) \frac{\mathbf{1}\mathbf{1}^T}{n} 
\end{align*} 
Thus the following holds:
\begin{align*} 
\mathbb{E} & \Vert \bX^{(t+1)} {-}  \Bar{\bX}^{(t+1)} \Vert_F^2 = \mathbb{E} \left\Vert \bX^{I_{(t+1)_0}} {-} \Bar{\bX}^{I_{(t+1)_0}} {-}\textstyle \sum_{t' = I_{(t+1)_0}}^{t} \eta ( \beta \bV^{(t')} +\boldsymbol{\nabla F}(\bX^{(t')}, \boldsymbol{\xi}^{(t')} )) \left( \mathbf{I} {-} \frac{1}{n} \mathbf{1}\mathbf{1}^T  \right) \right\Vert_F^2 \\
& \hspace{2cm}  \leq 2 \mathbb{E}\Vert \bX^{I_{(t+1)_0}} {-} \Bar{\bX}^{I_{(t+1)}}  \Vert_F^2 {+} 2\mathbb{E} \left\Vert \textstyle \sum_{t' = I_{(t+1)_0}}^{t} \eta ( \beta \bV^{(t')} +\boldsymbol{\nabla F}(\bX^{(t')}, \boldsymbol{\xi}^{(t')} )) \left( \mathbf{I} {-} \frac{1}{n} \mathbf{1}\mathbf{1}^T  \right) \right\Vert_F^2 \notag 
\end{align*}
Using $\verts{\mathbf{AB}}_F \leq \verts{\mathbf{A}}_F \verts{\mathbf{B}}_2$ to split the second term in R.H.S. of above along with (\ref{bound_W_mat}) from Fact 3 (with $k=0$) and further using the bound \eqref{bound_gap_grad}, we get: 
\begin{align}
\mathbb{E} \Vert \bX^{(t+1)} -  \Bar{\bX}^{(t+1)} \Vert_F^2  \leq 2 \mathbb{E}\Vert \bX^{I_{(t+1)_0}} - \Bar{\bX}^{I_{(t+1)_0}}  \Vert_F^2 + 4\eta^2 H^2 n G^2 \left( 1 + \frac{\beta^2}{(1-\beta)^2} \right) \label{eq:noncvx_cvx_localitr_diff}
\end{align}
 We bound the first term in R.H.S. of (\ref{eq:noncvx_cvx_localitr_diff}) by Lemma \ref{lem_noncvx} stated below and proved in Appendix \ref{proof_lem_consensus}.

\begin{lemma} \label{lem_noncvx} (\textbf{Consensus})
	Let $ \{ \bx_t ^{(i)}  \}_{t=0}^{T-1} $ be generated according to Algorithm \ref{alg_dec_sgd_li} under assumptions of Theorem \ref{convergence_results} with constant stepsize $\eta$, a threshold sequence $c_t \leq \frac{c_0}{\eta^{(1-\epsilon)}}$ for all $t$ where $\epsilon \in (0,1)$ and $c_0$ is constant, and define $\blx_t := \frac{1}{n} \sum_{i=1}^{n} \bx_t^{(i)} $.
	Consider the set of synchronization indices $\mathcal{I}_T$ =  $\{ I_{(1)},I_{(2)}, \hdots, I_{(t)}, \hdots  \}$. Then for any $I_{(t)} \in \mathcal{I}_T$, we have:
	\begin{align*}
	\mathbb{E}\sum_{j=1}^n \left\Vert \blx^{I_{(t)}} - \bx^{I_{(t)}}_{j} \right\Vert^2 = \mathbb{E} \Vert \bX^{I_{(t)}} - \Bar{\bX}^{I_{(t)}} \Vert_F^2  \leq \frac{4nA  \eta^2}{p^2}
	\end{align*}
	for constant $A =  \frac{p}{2}\left(  2 H^2 G^2 \left( 1 + \frac{\beta^2}{(1-\beta)^2} \right) \left( \frac{16}{\omega} + \frac{4}{p} \right) + \frac{2c_0\omega }{\eta^{(1-\epsilon)}} \right)$ where $p = \frac{\delta \gamma}{8} $, $\delta := 1 - | \lambda_2(\mathbf{W})|$, $\omega$ is compression parameter for operator $\C$.
\end{lemma}
Substituting the bound from Lemma \ref{lem_noncvx} in (\ref{eq:noncvx_cvx_localitr_diff}) and using the fact that $p \leq 1$, we have:
\begin{align} \label{mom_noncvx_interim8}
\mathbb{E} \Vert \bX^{(t+1)} -  \Bar{\bX}^{(t+1)} \Vert_F^2 \leq  \frac{2\eta^2}{p} \left(  2 H^2 n G^2  \left( 1 + \frac{\beta^2}{(1-\beta)^2} \right)  \left( \frac{16}{\omega} + \frac{8}{p} \right) + \frac{2c_0\omega n}{\eta^{(1-\epsilon)}} \right) 
\end{align}
for the same constant $\epsilon>0$ as in Lemma \ref{lem_noncvx}. Note that the above bound holds for all values of $t$. \\
Define $\Lambda := \frac{2}{p} \left(  2 H^2 n G^2  \left( 1 + \frac{\beta^2}{(1-\beta)^2} \right) \left( \frac{16}{\omega} + \frac{8}{p} \right) + \frac{2\omega c_0n}{\eta^{(1-\epsilon)}} \right) $. Substituting \eqref{mom_noncvx_interim8} in  \eqref{mom_noncvx_interim2} gives us:
\begin{align*}
		\frac{1}{T} \sum_{t=0}^{T-1} \mathbb{E} \verts{ \nabla f(\blx^{(t)})}^2 & \leq \frac{8(1{-}\beta)}{\eta }\frac{(f({\bx}^{(0)}) - f^*)}{T} + \frac{4L\eta \bar{\sigma}^2}{n   (1-\beta)} + \frac{10L^2 \Lambda \eta^2 }{ n }   + \frac{8L^2 \eta^2 \beta^4 \bar{\sigma}^2 }{n (1{-}\beta)^4 }
\end{align*}
Expanding on the value of $\Lambda$, we have:
\begin{align*}
	\frac{1}{T} \sum_{t=0}^{T-1} \mathbb{E} \verts{ \nabla f(\blx^{(t)})}^2 & \leq \frac{8(1-\beta)}{\eta }\frac{(f({\bx}^{(0)}) - f^*)}{T} + \frac{4L\eta \bar{\sigma}^2}{n   (1-\beta)} \\
	& \qquad +   \frac{20\eta^2L^2}{pn} \left(  2 H^2 n G^2  \left( 1 + \frac{\beta^2}{(1-\beta)^2} \right)  \left( \frac{16}{\omega} + \frac{8}{p} \right) \right)  \\
	& \qquad  + \frac{40L^2\omega nc_0\eta^{(1+\epsilon)}}{pn}   + \frac{8L^2 \eta^2 \beta^4 \bar{\sigma}^2 }{n (1-\beta)^4 }
\end{align*}
Substituting the value of $\eta = (1-\beta)  \sqrt{\frac{n}{T}}$, we get:
\begin{align*}
\frac{1}{T} \sum_{t=0}^{T-1} \mathbb{E} \verts{ \nabla f(\blx^{(t)})}^2 & \leq    \frac{1}{\sqrt{nT}} \left(8 (f({\bx}^{(0)}) - f^*) + 4L \bar{\sigma}^2  \right) + \frac{40L^2 (1-\beta)^{(1+\epsilon)} \omega c_0 n^{\nicefrac{(1+\epsilon)}{2}}}{ pT^{\nicefrac{(1+\epsilon)}{2}} }  \\
& \quad + \frac{20 (1-\beta)^2 L^2}{T p}   \left(  2 H^2 n G^2  \left( 1 + \frac{\beta^2}{(1-\beta)^2} \right)  \left( \frac{16}{\omega} + \frac{8}{p} \right) \right)   + \frac{8L^2  \beta^4 \bar{\sigma}^2 }{T (1-\beta)^2 } \\
& \leq    \frac{1}{\sqrt{nT}} \left(8 (f({\bx}^{(0)}) - f^*) + 4L \bar{\sigma}^2  \right) + \frac{40L^2 \omega c_0 n^{\nicefrac{(1+\epsilon)}{2}}  (1-\beta)^{(1+\epsilon)} }{ pT^{\nicefrac{(1+\epsilon)}{2}} }  \\
& \quad + \frac{80n L^2H^2 G^2 }{T p }  \left( \frac{16}{\omega} + \frac{8}{p} \right)   + \frac{8L^2  \beta^4 \bar{\sigma}^2 }{T (1-\beta)^2 }
\end{align*}
where in the last inequality, we've used the fact that $(1-\beta)^{r} \leq 1$ , $\beta^r \leq 1$ for $r >0$.
Note that we require $\eta \leq \min \left\{ \frac{(1-\beta)}{4L} , \frac{(1-\beta)^2}{2\sqrt{2} L \beta^2} \right\} $, thus for $\eta = (1-\beta) \sqrt{\frac{n}{T}}$, we need  to run our algorithm for $ T \geq \max \left\{ 16L^2n, \frac{8L^2 \beta^4 n}{(1-\beta)^2}  \right\} $ for the above rate expression to hold.
We finally use the fact that $p \leq \omega$ (as $\delta \leq 1$ and $p:= \frac{\gamma^* \delta}{8}$ with $\gamma^* \leq \omega $). This completes proof of the non-convex part of Theorem \ref{convergence_results}. We can further use the fact that $p \geq \frac{\delta^2 \omega}{644}$ (proved in Lemma \ref{suppl_lem_coeff_calc}) to get the expression given in the theorem statement.

}

\section{Proof of Theorem \ref{convergence_results} (Convex objective)} \label{app:proof_thm_cvx}
We start with the same virtual sequence defined in \eqref{virt_seq_rec}. Consider the quantity $\mathbb{E}_{\boldsymbol{\xi}^{(t)}}\|\btx^{(t+1)} - \bx^*\|^2$, where expectation is taken over sampling across all the nodes at the $t$'th iteration:
\begin{align}
&\mathbb{E}_{\boldsymbol{\xi}^{(t)}} \Vert \btx^{(t+1)} - \bx^* \Vert^ 2 
=  \mathbb{E}_{\boldsymbol{\xi}^{(t)}} \left\Vert \btx^{(t)} - \frac{\eta }{(1-\beta)  n}\sum_{j=1}^n \nabla F_j(\bx_j^{(t)},\xi_j^{(t)}) - \bx^* \right\Vert^2  \notag\\
& = \mathbb{E}_{\boldsymbol{\xi}^{(t)}} \left\Vert \btx^{(t)} {-} \bx^* {-} \frac{\eta}{(1-\beta) n}\sum_{j=1}^n \nabla f_j(\bx_j^{(t)}) {+} \frac{\eta }{ (1{-}\beta) n}\sum_{j=1}^n \nabla f_j(\bx_j^{(t)}) {-} \frac{\eta }{n (1{-}\beta) }\sum_{j=1}^n \nabla F_j(\bx_j^{(t)},\xi_j^{(t)})  \right\Vert^2  \notag\\
& =  \left\Vert \btx^{(t)} {-} \bx^* {-} \frac{\eta }{ (1-\beta) n}\sum_{j=1}^n \nabla f_j(\bx_j^{(t)}) \right\Vert^2 {+} \frac{\eta ^2}{(1-\beta)^2}  \mathbb{E}_{\boldsymbol{\xi}^{(t)}} \left\Vert \frac{1}{n}\sum_{j=1}^n \nabla f_j(\bx_j^{(t)}) {-} \frac{1}{n}\sum_{j=1}^n \nabla F_j(\bx_j^{(t)},\xi_j^{(t)}) \right\Vert^2 \notag\\
& \quad  + \frac{2\eta }{ (1-\beta) n}\mathbb{E}_{\boldsymbol{\xi}^{(t)}} \left\langle \btx^{(t)} - \bx^* - \frac{\eta }{ (1-\beta) n}\sum_{j=1}^n \nabla f_j(\bx_j^{(t)}) , \sum_{j=1}^n \nabla f_j(\bx_j^{(t)}) - \sum_{j=1}^n \nabla F_j(\bx_j^{(t)},\xi_j^{(t)}) \right\rangle \notag \\
&\leq \left\Vert \btx^{(t)} - \bx^* - \frac{\eta }{ (1-\beta) n}\sum_{j=1}^n \nabla f_j(\bx_j^{(t)}) \right\Vert^2 + \frac{\eta ^2\bar{\sigma}^2}{(1-\beta)^2n}
\label{suppl_cvx_lemm_total}
\end{align}

Where to get the last inequality we used the fact that $\mathbb{E}_{\xi_i^{(t)}} [\nabla F_i(\bx_i^{(t)}, \xi_i^{(t)})] = \nabla f_i(\bx_i^{(t)}) $ for all $i \in [n]$ and the variance bound \eqref{eq:bound_var} from Fact \ref{bound_var}. Now we thus consider the first term in \eqref{suppl_cvx_lemm_total}:
\begin{align} \label{suppl_cvx_lemm_firstterm}
\left\Vert \btx^{(t)} - \bx^* - \frac{\eta }{(1-\beta)n}\sum_{j=1}^n \nabla f_j(\bx_j^{(t)}) \right\Vert^2 = \Vert \btx^{(t)} - \bx^* \Vert^2 + \frac{\eta ^2}{(1-\beta)^2}  \underbrace{\left\Vert \frac{1}{n}\sum_{j=1}^n \nabla f_j(\bx_j^{(t)}) \right\Vert^2}_{T_1} \notag \\
- \frac{2\eta }{(1-\beta)} \underbrace{ \left\langle \btx^{(t)} - \bx^*,\frac{1}{n}\sum_{j=1}^n \nabla f_j(\bx_j^{(t)}) \right\rangle}_{T_2}
\end{align}
To bound $T_1$ in (\ref{suppl_cvx_lemm_firstterm}), note that:
\begin{align}
T_1 &= \left\Vert \frac{1}{n}\sum_{j=1}^n (\nabla f_j(\bx_j^{(t)}) - \nabla f_j(\blx^{(t)}) + \nabla f_j(\blx^{(t)}) - \nabla f_j(\bx^*) ) \right\Vert^2  \notag \\
& \leq \frac{2}{n}\sum_{j=1}^n \Vert \nabla f_j(\bx_j^{(t)}) - \nabla f_j(\blx^{(t)}) \Vert^2 + 2 \left\Vert \frac{1}{n}\sum_{j=1}^n \nabla f_j(\blx^{(t)}) - \frac{1}{n}\sum_{j=1}^n \nabla f_j(\bx^*) \right\Vert^2  \notag \\
& \leq \frac{2L^2}{n}\sum_{j=1}^n \Vert \bx_j^{(t)} - \blx^{(t)} \Vert^2 + 4L  (f(\blx^{(t)}) - f^*)  \label{suppl_cvx_t1_bound}
\end{align}
where in the last inequality, we used $L-$Lipschitz gradient property of objectives $\{f_j\}_{j=1}^{n}$ to bound the first term and optimality of $\bx^*$ for $f$ (i.e., $\nabla f(\bx^*) = 0$) and $L-$smoothness property of $f$  to bound the second term as:
$ \left\Vert \frac{1}{n}\sum_{j=1}^n \nabla f_j(\blx^{(t)}) - \frac{1}{n}\sum_{j=1}^n \nabla f_j(\bx^*) \right\Vert^2 = \verts{ \nabla f (\blx^{(t)} ) - \nabla f (\bx^*)  }^2   \leq 2L \left( f(\blx^{(t)} ) -f^*  \right) $.\\
To bound $T_2$ in (\ref{suppl_cvx_lemm_firstterm}), note that: 
\begin{align}
-2T_2 & = - 2 \lragnle{\btx^{(t)}  -\blx^{(t)} , \frac{1}{n}\sum_{j=1}^n \nabla f_j(\bx_j^{(t)})    } - \frac{2}{n} \sum_{j=1}^n \lragnle{\blx^{(t)}  - \bx^* ,  \nabla f_j(\bx_j^{(t)})} \notag \\
& = 2\frac{\beta^2}{(1-\beta)} \lragnle{ \frac{\eta}{n} \sum_{i=1}^{n} \bv_i^{(t-1)}   , \frac{1}{n}\sum_{j=1}^n \nabla f_j(\bx_j^{(t)})} - \frac{2}{n} \sum_{j=1}^n \lragnle{\blx^{(t)}  - \bx^*,  \nabla f_j(\bx_j^{(t)})    }  \label{eq:mom_cvx_interim2}
\end{align}
In \eqref{eq:mom_cvx_interim2}, we used the definition of $\btx^{(t)}$ from \eqref{virt_seq} to write $\btx^{(t)} - \blx^{(t)} = - \frac{\eta \beta^2}{(1-\beta)}\frac{1}{n}\sum_{i=1}^n\bv_i^{(t-1)}$. 
Now we note a simple trick for inner-products: 
\begin{align}
\lragnle{\frac{\eta }{n} \sum_{i=1}^{n} \bv_i^{(t-1)}, \frac{1}{n}\sum_{j=1}^n \nabla f_j(\bx_j^{(t)})} = \lragnle{ \frac{(\eta )^{\nicefrac{3}{4}}}{n} \sum_{i=1}^{n} \bv_i^{(t-1)}, \frac{(\eta )^{\nicefrac{1}{4}}}{n}\sum_{j=1}^n \nabla f_j(\bx_j^{(t)})}. \label{eq:mom_cvx_interim_3}
\end{align}
This trick is crucial to getting a speedup of $n$ -- the number of worker nodes -- in our final convergence rate.
Using $2\langle {\bf a}, {\bf b} \rangle \leq \|{\bf a}\|^2 + \|{\bf b}\|^2$ for bounding \eqref{eq:mom_cvx_interim_3} and then substituting that in \eqref{eq:mom_cvx_interim2} gives
\begin{align}
{-}2T_2 &\leq 
\frac{\beta^2}{(1{-}\beta)}\left[(\eta )^{\nicefrac{3}{2}}\left\|\frac{1}{n} \sum_{i=1}^{n} \bv_i^{(t{-}1)}\right\|^2 {+} (\eta )^{\nicefrac{1}{2}} \left\|\frac{1}{n}\sum_{j=1}^n \nabla f_j(\bx_j^{(t)})\right\|^2\right]  {-} \frac{2}{n} \sum_{j=1}^n \lragnle{\blx^{(t)} { -} \bx^* ,  \nabla f_j(\bx_j^{(t)})} \label{eq:mom_cvx_interim_4}
\end{align}
Note that the second term of \eqref{eq:mom_cvx_interim_4} is the same as $T_1$ from \eqref{suppl_cvx_lemm_firstterm} and we have already bounded that in \eqref{suppl_cvx_t1_bound}.
We now focus on bounding the last term of \eqref{eq:mom_cvx_interim_4}. Using expression for convexity and $L$-smoothness for $f_j , \, j \in [n] $ respectively, we can bound this as follows: 
\begin{align} 
-\frac{2}{n} \sum_{j=1}^n \langle \blx^{(t)}  - & \bx^*,   \nabla f_j(\bx_j^{(t)})  \rangle  = -\frac{2}{n} \sum_{j=1}^n \left[ \left\langle \blx^{(t)} - \bx_j^{(t)}, \nabla f_j(\bx_j^{(t)}) \right\rangle + \left\langle \bx_j^{(t)} - \bx^*, \nabla f_j(\bx_j^{(t)}) \right\rangle \right] \notag  \\
& \leq -\frac{2}{n} \sum_{j=1}^n \left[  f_j(\blx^{(t)}) - f_j(\bx_j^{(t)}) - \frac{L}{2}\Vert \blx^{(t)} - \bx_j^{(t)}  \Vert^2  + f_j(\bx_j^{(t)}) - f_j(\bx^*)  \right] \notag \\
& =  -2  (f(\blx^{(t)}) - f(\bx^*)) + \frac{L}{n}\sum_{j=1}^n \Vert \blx^{(t)} - \bx_j^{(t)}  \Vert^2   \label{suppl_cvx_t2part_bound}
\end{align}
Substituting the bounds for the second and the last terms of \eqref{eq:mom_cvx_interim_4} from \eqref{suppl_cvx_t1_bound} and \eqref{suppl_cvx_t2part_bound}, respectively, we get
\begin{align*} 
-2T_2  & \leq \frac{(\eta )^{\nicefrac{3}{2}}\beta^2}{(1-\beta)} \left\|\frac{1}{n} \sum_{i=1}^{n} \bv_i^{(t-1)}\right\|^2 + \frac{(\eta)^{\nicefrac{1}{2}}\beta^2}{(1-\beta)}\left( \frac{2L^2}{n}\sum_{j=1}^n \Vert \bx_j^{(t)} - \blx^{(t)} \Vert^2 + 4L  (f(\blx^{(t)}) - f^*)\right) \notag \\
& \hspace{1cm} - 2  (f(\blx^{(t)}) - f(\bx^*)) + \frac{L}{n}\sum_{j=1}^n \Vert \blx^{(t)} - \bx_j^{(t)}  \Vert^2  
\end{align*}
Thus we finally have:
\begin{align}
-\frac{2\eta }{(1-\beta)}T_2 & \leq  \frac{\eta^{\nicefrac{5}{2}}\beta^2}{(1-\beta)^2} \left\|\frac{1}{n} \sum_{i=1}^{n} \bv_i^{(t-1)}\right\|^2 + \left(\frac{2\eta^{\nicefrac{3}{2}}\beta^2L^2}{(1-\beta)^2} + \frac{\eta L }{(1-\beta)} \right)\frac{1}{n}\sum_{j=1}^n \Vert \bx_j^{(t)} - \blx^{(t)} \Vert^2 \notag \\
& \quad + \left(\frac{4\eta^{\nicefrac{3}{2}}\beta^2L}{(1-\beta)^2} - \frac{2\eta }{(1-\beta)}\right)\left(f(\blx^{(t)}) - f^*\right)  \label{eq:mom_cvx_interim_5}
\end{align}
Substituting \eqref{suppl_cvx_t1_bound}, \eqref{eq:mom_cvx_interim_5} in \eqref{suppl_cvx_lemm_firstterm} and using the resulting bound back in \eqref{suppl_cvx_lemm_total}, and then taking expectation w.r.t.\ the entire process, we get:
\begin{align} 
\mathbb{E}\Vert \btx^{(t+1)} - \bx^* \Vert^2 &\leq \bbE\|\btx^{(t)} - \bx^*\|^2 + \frac{\eta^{\nicefrac{5}{2}}\beta^2}{(1-\beta)^2} \bbE \left\|\frac{1}{n} \sum_{i=1}^{n} \bv_i^{(t-1)}\right\|^2 + \frac{\eta ^2\bar{\sigma}^2}{(1-\beta)^2n}  \notag \\ 
& + \left(\frac{2\eta ^2L^2}{(1-\beta)^2} + \frac{2\eta^{\nicefrac{3}{2}}\beta^2L^2}{(1-\beta)^2} + \frac{\eta L }{(1-\beta)} \right)\frac{1}{n}\sum_{j=1}^n \bbE\Vert \bx_j^{(t)} - \blx^{(t)} \Vert^2 \notag \\
&\hspace{2cm} + \left(\frac{4\eta ^2L}{(1-\beta)^2} + \frac{4\eta^{\nicefrac{3}{2}}\beta^2L}{(1-\beta)^2} - \frac{2\eta }{(1-\beta)}\right)\left(\bbE f(\blx^{(t)}) - f^*\right)  \label{eq:mom_cvx_interim1}
\end{align}
Using the fact that $\mathbb{E} \verts{\frac{1}{n} \sum_{j=1}^{n} \bv_j^{(t)}}^2 \leq \frac{G^2}{(1-\beta)^2}$ for all $t\geq 1$ (see proof of Fact \ref{fact-bound}), we have: 
\begin{align} 
\mathbb{E}\Vert \btx^{(t+1)} - \bx^* \Vert^2 &\leq \bbE\|\btx^{(t)} - \bx^*\|^2 + \frac{\eta^{\nicefrac{5}{2}}\beta^2G^2}{(1-\beta)^4} + \frac{\eta ^2\bar{\sigma}^2}{(1-\beta)^2n}  \notag \\ 
& + \left(\frac{2\eta ^2L^2}{(1-\beta)^2} + \frac{2\eta^{\nicefrac{3}{2}}\beta^2L^2}{(1-\beta)^2} + \frac{\eta L }{(1-\beta)} \right)\frac{1}{n}\sum_{j=1}^n \bbE\Vert \bx_j^{(t)} - \blx^{(t)} \Vert^2 \notag \\
&\hspace{2cm} + \left(\frac{4\eta ^2L}{(1-\beta)^2} + \frac{4\eta^{\nicefrac{3}{2}}\beta^2L}{(1-\beta)^2} - \frac{2\eta }{(1-\beta)}\right)\left(\bbE f(\blx^{(t)}) - f^*\right)  \label{mom_cvx_interim_5}
\end{align}

If we take $\eta  \leq \min \left\{ \frac{(1-\beta)}{8L} , \frac{(1-\beta)^2}{(8L\beta^2)^2}   \right\}$, then we have:
\begin{align}
\left(\frac{2\eta ^2 L^2}{(1-\beta)^2} + \frac{2\eta^{\nicefrac{3}{2}}\beta^2L^2}{(1-\beta)^2} + \frac{\eta L }{(1-\beta)} \right) &\leq \frac{3\eta L}{2(1-\beta)} \label{mom_cvx_bound1} \\
\left(\frac{4\eta ^2L}{(1-\beta)^2} + \frac{4\eta^{\nicefrac{3}{2}}\beta^2L}{(1-\beta)^2} - \frac{2\eta }{(1-\beta)}\right) &\leq -\frac{\eta }{(1-\beta)} \label{mom_cvx_bound2}
\end{align}

Substituting the bounds from \eqref{mom_cvx_bound1} and \eqref{mom_cvx_bound2} to \eqref{mom_cvx_interim_5} gives
\begin{align} 
\mathbb{E}\Vert \btx^{(t+1)} - \bx^* \Vert^2 &\leq \bbE\|\btx^{(t)} - \bx^*\|^2 + \frac{\eta^{\nicefrac{5}{2}}\beta^2G^2}{(1-\beta)^4} + \frac{\eta ^2\bar{\sigma}^2}{(1-\beta)^2n} + \frac{3\eta L}{2(1-\beta)}\frac{1}{n}\sum_{j=1}^n \bbE\Vert \bx_j^{(t)} - \blx^{(t)} \Vert^2 \notag \\ 
&\hspace{1cm}   - \frac{\eta }{(1-\beta)} \left(\bbE f(\blx^{(t)}) - f^*\right)  \label{mom_cvx_interim6}
\end{align}

We can now bound the second last term in R.H.S. of \eqref{mom_cvx_interim6} similar to \eqref{mom_noncvx_interim8} in the proof of non-convex part of Theorem \ref{convergence_results} given in Appendix~\ref{app:proof_thm_noncvx}. This gives us the bound:
\begin{align*}
\mathbb{E} \Vert \bX^{(t+1)} -  \Bar{\bX}^{(t+1)} \Vert_F^2 \leq  \frac{2\eta^2}{p} \left(  2 H^2 n G^2  \left( 1 + \frac{\beta^2}{(1-\beta)^2} \right)  \left( \frac{16}{\omega} + \frac{8}{p} \right) + \frac{2c_0\omega n}{\eta^{(1-\epsilon)}} \right) 
\end{align*}

Using above bound for the term $\sum_{j=1}^n \bbE\Vert \bx_j^{(t)} - \blx^{(t)} \Vert^2$ in \eqref{mom_cvx_interim6} we get:
\begin{align} 
\mathbb{E}\Vert \btx^{(t+1)} - \bx^* \Vert^2 &\leq \bbE\|\btx^{(t)} - \bx^*\|^2 + \frac{\eta^{\nicefrac{5}{2}}\beta^2G^2}{(1-\beta)^4} + \frac{\eta ^2\bar{\sigma}^2}{(1-\beta)^2n}  - \frac{\eta }{(1-\beta)} \left(\bbE f(\blx^{(t)}) - f^*\right)  \notag \\ 
&\hspace{1cm} + \frac{3\eta^3 L}{p(1-\beta)} \left(  2 H^2  G^2  \left( 1 + \frac{\beta^2}{(1-\beta)^2} \right)  \left( \frac{16}{\omega} + \frac{8}{p} \right) + \frac{2c_0\omega }{\eta^{(1-\epsilon)}} \right)    \label{mom_cvx_interim8}
\end{align}
By rearranging terms in \eqref{mom_cvx_interim8} and noting that $p \leq \omega$ (as $\delta \leq 1$ and $p:= \frac{\gamma^* \delta}{8}$ with $\gamma^* \leq \omega $) and the fact that $\left( 1 + \frac{\beta^2}{(1-\beta)^2} \right) \leq \frac{2}{(1-\beta)^2}$ (because $\beta < 1$), we get:
\begin{align} 
\mathbb{E}\Vert \btx^{(t+1)} - \bx^* \Vert^2 &\leq \bbE\|\btx^{(t)} - \bx^*\|^2 + \frac{\eta^{\nicefrac{5}{2}}\beta^2G^2}{(1-\beta)^4} + \frac{\eta ^2\bar{\sigma}^2}{(1-\beta)^2n}  - \frac{\eta }{(1-\beta)} \left(\bbE f(\blx^{(t)}) - f^*\right)  \notag \\ 
&\hspace{1cm} + \frac{288\eta^3 LH^2G^2}{p^2(1-\beta)^3}    + \frac{6c_0 \omega L \eta^{(2+\epsilon)}}{p(1-\beta)}  \label{mom_cvx_interim81} 
\end{align}

Summing \eqref{mom_cvx_interim81} from $t=0$ to $T-1$, rearranging terms and diving by $T$ both sides gives us:
\begin{align*}
	\sum_{t=0}^{T-1}\frac{\left(\bbE f(\blx^{(t)}) - f^*\right)}{T}  &\leq \frac{(1-\beta)}{\eta} \sum_{t=0}^{T-1}\frac{\left(\bbE\|\btx^{(t)} - \bx^*\|^2 - \mathbb{E}\Vert \btx^{(t+1)} - \bx^* \Vert^2\right)}{T} + \frac{\eta^{\nicefrac{3}{2}}\beta^2G^2}{(1-\beta)^3} + \frac{\eta \bar{\sigma}^2}{(1-\beta)n}   \notag \\ 
	&\hspace{1cm} + \frac{288\eta^2 LH^2G^2}{p^2(1-\beta)^2}    + \frac{6c_0 \omega L \eta^{(1+\epsilon)}}{p}
\end{align*}
Using Jensen's inequality for convex function $f$ on the L.H.S. and setting $\eta = (1-\beta)\sqrt{\frac{n}{T}}$ for $T \geq \max \{ (8L)^2 n, \frac{(8\beta^2L)^4 n}{(1-\beta)^2}  \}$, for  $\blx^{(T)}_{avg} := \frac{1}{T} \sum_{t=0}^{T-1} \bar{\bx}^{(t)}$ we have that:
\begin{align*}
\bbE f(\blx^{(T)}_{avg}) - f^* &\leq \frac{\left(\bbE\|\btx^{(0)} - \bx^*\|^2 - \mathbb{E}\Vert \btx^{(T)} - \bx^* \Vert^2\right)}{\sqrt{nT}} + \frac{n^{\nicefrac{3}{4}}\beta^2G^2}{(1-\beta)^{\nicefrac{3}{2}} T^{\nicefrac{3}{4}}} + \frac{\bar{\sigma}^2}{\sqrt{nT}}   \notag \\ 
&\hspace{1cm} + \frac{288 LH^2G^2}{p^2T}    + \frac{6c_0 \omega L (1-\beta)^{(1+\epsilon)} n^{\nicefrac{(1+\epsilon)}{2}}}{p T^{\nicefrac{(1+\epsilon)}{2}}}
\end{align*}
Using the fact that $\btx^{(0)} = \blx^{(0)}$ and $\epsilon, \beta \in (0,1)$ we have:
\begin{align*}
\bbE f(\blx^{(T)}_{avg}) - f^* &\leq \frac{\|\blx^{(0)} - \bx^*\|^2 + \bar{\sigma}^2} {\sqrt{nT}} + \frac{n^{\nicefrac{3}{4}} \beta^2G^2}{(1-\beta)^{\nicefrac{3}{2}} T^{\nicefrac{3}{4}}}   + \frac{384n LH^2G^2}{p^2T}    + \frac{6c_0 \omega L n^{\nicefrac{(1+\epsilon)}{2}}}{p T^{\nicefrac{(1+\epsilon)}{2}}}
\end{align*}
This completes proof of convex part of Theorem \ref{convergence_results}. We can further use the fact that $p \geq \frac{\delta^2 \omega}{644}$ to get the expression given in the theorem statement.

\section{Proof of Lemma \ref{lem_noncvx} (Consensus)} \label{proof_lem_consensus}
\allowdisplaybreaks
{


 \TODO{}

\noindent In this section, we provide a proof of Lemma \ref{lem_noncvx}, which states that 
$\sum_{j=1}^n \mathbb{E} \left\Vert \bar{\bx}^{I_{(t)}} - \bx^{I_{(t)}}_{j} \right\Vert^2$ 
-- the difference between the local and the average iterates at the synchronization indices -- 
is bounded by a constant times the learning rate $\eta$, which can effectively be made small by running the algorithm for larger number of iterations $T$ as we choose $\eta = (1-\beta)\sqrt{\frac{n}{T}}$. Thus, this result shows that the nodes achieve a consensus towards the average parameter vector as the algorithm progresses. \\
We first provide a high level idea of the proof to aid the reader. Our interest is in providing a bound for $e_{I_{(t)}}^{(1)}:=\sum_{j=1}^n \mathbb{E} \left\Vert \bar{\bx}^{I_{(t)}} - \bx^{I_{(t)}}_{j} \right\Vert^2$.
We show this by setting up a contracting recursion for $e_{I_{(t)}}^{(1)}$. First we prove that 
\begin{align}
e_{I_{(t+1)}}^{(1)} \leq (1-\alpha_1)e_{I_{(t)}}^{(1)} + (1-\alpha_1)e_{I_{(t)}}^{(2)} + c_1\eta^2, \label{eq:dec_avg_local-interim1}
\end{align}
where $e_{I_{(t)}}^{(2)}:=\sum_{j=1}^n \mathbb{E} \left\Vert \hat{\bx}^{I_{(t+1)}} - \bx^{I_{(t)}}_{j} \right\Vert^2$, $\alpha_1\in(0,1)$, and $c_1$ is a constant that depends on $n,\delta,\beta,H,G$. The quantity $e_{I_{(t)}}^{(2)}$ relates to the expected deviation of local node parameters and their copies.
Note that \eqref{eq:dec_avg_local-interim1} gives a contracting recursion in $e_{I_{(t)}}^{(1)}$, but it also gives the other term $e_{I_{(t)}}^{(2)}$, which we have to bound. 
It turns out that we can prove a similar inequality for $e_{I_{(t)}}^{(2)}$:
\begin{align}
e_{I_{(t+1)}}^{(2)} \leq (1-\alpha_2)e_{I_{(t)}}^{(1)} + (1-\alpha_2)e_{I_{(t)}}^{(2)} + c_2\eta^2, \label{eq:dec_avg_local-interim2}
\end{align}
where $\alpha_2\in(0,1)$; furthermore, we can choose $\alpha_1,\alpha_2$ such that $\alpha_1+\alpha_2 > 1$.

Define $e_{I_{(t)}}:=e_{I_{(t)}}^{(1)} + e_{I_{(t)}}^{(2)}$. Adding \eqref{eq:dec_avg_local-interim1} and \eqref{eq:dec_avg_local-interim2} gives the following recursion with $\alpha\in(0,1)$:
\begin{align}
e_{I_{(t+1)}} \leq (1-\alpha)e_{I_{(t)}} + c_3\eta^2. \label{eq:dec_avg_local-interim3}
\end{align}
From \eqref{eq:dec_avg_local-interim3}, we can show that $e_{I_{(t)}}\leq C\eta^2$ for some $C$ that depends on $n,\delta,\beta,H,G,\omega,c_0$. The result of Lemma \ref{lem_noncvx} follows from this because $\sum_{j=1}^n \mathbb{E} \left\Vert \bar{\bx}^{I_{(t)}} - \bx^{I_{(t)}}_{j} \right\Vert^2 = e_{I_{(t)}}^{(1)} \leq e_{I_{(t)}}$.

We first state the above-mentioned recursion results for $e^{(1)}_{I_{(t+1)}}$ and $e^{(2)}_{I_{(t+1)}}$ below in Lemma \ref{lem_cvx_e1} and Lemma \ref{lem_cvx_e2}, respectively, and then using that we prove Lemma \ref{lem_noncvx}. The proofs of Lemma \ref{lem_cvx_e1} and Lemma \ref{lem_cvx_e2} are provided in Appendix \ref{suppl_recursion_lemm_proofs}.

\begin{lemma}\label{lem_cvx_e1}
Under the setting of Theorem \ref{convergence_results}, $e_{I_{(t+1)}}^{(1)}:=\sum_{j=1}^n \mathbb{E} \left\Vert \bar{\bx}^{I_{(t+1)}} - \bx^{I_{(t+1)}}_{j} \right\Vert^2$ satisfies:
\begin{align*}
	e^{(1)}_{I_{(t+1)}}  & \leq  (1+\alpha_5^{-1})R_1 e^{(1)}_{I_{(t)}}  + (1+\alpha_5^{-1}) R_2 e^{(2)}_{I_{(t)}} + Q_1\eta^2,
\end{align*}
where $R_1=(1+\alpha_1)(1-\gamma \delta)^2 , R_2=(1+\alpha_1^{-1}) \gamma^2 \lambda^2 $ and $Q_1 = 2 H^2 n G^2 \left(1 + \frac{\beta^2}{(1-\beta)^2} \right)(1+\alpha_5)(R_1+R_2) $. Here $\alpha_1, \alpha_5 >0$, $\delta$ is the spectral gap, $H$ is the synchronization gap, $\gamma$ is the consensus stepsize, and $\lambda := \verts{\mathbf{W}-\mathbf{I}}_2$ where $\mathbf{W}$ is a doubly stochastic mixing matrix. 
\end{lemma}

\begin{lemma}\label{lem_cvx_e2}
	Under the setting of Theorem \ref{convergence_results}, $e_{I_{(t+1)}}^{(2)}:=\sum_{j=1}^n \mathbb{E} \left\Vert \hat{\bx}^{I_{(t+2)}} - \bx^{I_{(t+1)}}_{j} \right\Vert^2$ satisfies:
	\begin{align*}
	e^{(2)}_{I_{(t+1)}} \leq (1+\alpha_5^{-1}) R_3 e^{(2)}_{I_{(t)}} + (1+\alpha_5^{-1}) R_4 e^{(1)}_{I_{(t)}}  +  \eta^2Q_2,
	\end{align*}
	where $R_3 = (1+\gamma \lambda )^2 (1+\alpha_4) (1+\alpha_3) (1+\alpha_2) (1-\omega)$ , $R_4 = \gamma^2 \lambda^2 (1+\alpha_4^{-1}) (1+\alpha_3) (1+\alpha_2) (1-\omega) $ and $Q_2 = 2H^2 n G^2 \left( 1 + \frac{\beta^2}{(1-\beta)^2} \right)((1+\alpha_5)(R_3+R_4)+( 1+\alpha_2^{-1}) + (1+\alpha_3^{-1}) (1+\alpha_2) (1-\omega)) + (1+\alpha_2)\omega n \frac{c_0}{\eta^{(1-\epsilon)}}  $. Note that $Q_2$ depends on $t$ (as captured by $c_{I_{(t)}}$ in the expression) as we allow for our triggering threshold to change with time.  Here $\alpha_2, \alpha_3, \alpha_4 >0, \alpha_5 >0$  are the same as those used in Lemma \ref{lem_cvx_e1}, $\delta$ is the spectral gap, $H$ is the synchronization gap, $\gamma$ is the consensus stepsize, and $\lambda = \verts{\mathbf{W}-\mathbf{I}}_2$ where $\mathbf{W}$ is a doubly stochastic mixing matrix.
\end{lemma} 

\begin{proof} [Proof of Lemma \ref{lem_noncvx}]
	Having established the bounds on $e_{I_{(t+1)}}^{(1)}$ and $e_{I_{(t+1)}}^{(2)}$, we are now ready to prove Lemma \ref{lem_noncvx}.
	Consider the following expression:
	\begin{align} \label{suppl_et_eqn_li}
	e_{I_{(t+1)}} = \underbrace{\mathbb{E} \Vert \bX^{I_{(t+1)}} - \Bar{\bX}^{I_{(t+1)}} \Vert_F^2}_{e^{(1)}_{I_{(t+1)}}} + \underbrace{ \mathbb{E} \Vert \bX^{I_{(t+1)}} - \hat{\bX}^{I_{(t+2)}} \Vert_F^2}_{e^{(2)}_{I_{(t+1)}}}
	\end{align}
	We note that Lemma \ref{lem_cvx_e1} and Lemma \ref{lem_cvx_e2} provide bounds for the first and the second term in the RHS of (\ref{suppl_et_eqn_li}). Substituting them in (\ref{suppl_et_eqn_li}) gives:
	\begin{align} \label{suppl_et_eqn_total_li}
	&e_{I_{(t+1)}}  
	 \leq  R_1 (1+\alpha_5^{-1}) \mathbb{E} \left\Vert  \Bar{\bX}^{I_{(t)}} - \bX^{I_{(t)}} \right\Vert^2 + R_2 (1+\alpha_5^{-1}) \mathbb{E} \left\Vert \hat{\bX}^{I_{(t+1)}} - \bX^{I_{(t)}} \right\Vert^2 \notag \\
	& \hspace{0.3cm}  + R_4 (1+\alpha_5^{-1}) \mathbb{E} \left\Vert  \Bar{\bX}^{I_{(t)}} - \bX^{I_{(t)}} \right\Vert^2 + R_3 (1+\alpha_5^{-1}) \mathbb{E} \left\Vert   \hat{\bX}^{I_{(t+1)}} - \bX^{I_{(t)}} \right\Vert^2 + (Q_1 + Q_2)\eta ^2
	\end{align}	
	Define the following:
	\begin{align}
	\pi_1 (\gamma) & := R_2 + R_3 = \gamma^2 \lambda^2 (1+\alpha_1^{-1}) + (1+\gamma \lambda )^2 (1+\alpha_4) (1+\alpha_3) (1+\alpha_2) (1-\omega) \label{cvx_pi1}   \\
	\pi_2 (\gamma) & := R_1 + R_4 = (1-\delta \gamma)^2(1+\alpha_1) + \gamma^2 \lambda^2 (1+\alpha_4^{-1}) (1+\alpha_3) (1+\alpha_2) (1-\omega) \label{cvx_pi2} \\
	\pi_{0}  := &  Q_1 + Q_2 \leq 2 H^2 n G^2 \left( 1 + \frac{\beta^2}{(1-\beta)^2} \right)(1+\alpha_5)(R_1 + R_2 +R_3+R_4) \notag \\
	 + 2 & H^2 n G^2 \left( 1 + \frac{\beta^2}{(1-\beta)^2} \right)((1+\alpha_2^{-1}) + (1-\omega)(1+\alpha_3^{-1})(1+\alpha_2) ) + (1+\alpha_2)\frac{\omega n c_0}{\eta^{(1-\epsilon)}} \label{cvx_pi3}
	\end{align}
	The bound on $e_{I_{(t+1)}}$ in (\ref{suppl_et_eqn_total_li}) can be rewritten as:
	\begin{align} \label{lemm_mom_cvx_interim}
	&e_{I_{(t+1)}} \leq (1+\alpha_5^{-1}) \left[\pi_1 (\gamma) \mathbb{E} \Vert \bX^{I_{(t)}} - \hat{\bX}^{I_{(t+1)}} \Vert_F^2 + \pi_2 (\gamma) \mathbb{E} \Vert \bX^{I_{(t)}} - \Bar{\bX}^{I_{(t)}} \Vert_F^2\right] + \pi_0 \eta^2 \notag \\
	&\quad \leq (1+\alpha_5^{-1}) \max \{ \pi_1 (\gamma) , \pi_2 (\gamma) \} \, \mathbb{E} \left[ \Vert {\bX}^{I_{(t+\frac{1}{2})}} - \hat{\bX}^{I_{(t+1)}} \Vert_F^2 + \Vert \bX^{I_{(t+\frac{1}{2})}} - \Bar{\bX}^{I_{(t+\frac{1}{2})}} \Vert_F^2 \right] + \pi_0  \eta ^2
	\end{align}
	Calculation of $ \max \{ \pi_1 (\gamma) , \pi_2 (\gamma) \}$ and $\pi_0$ is given in Lemma \ref{suppl_lem_coeff_calc} in Appendix~\ref{app:supporting-lemmas-for-consensus}, where we show that:\\
	$\max \{ \pi_1 (\gamma) , \pi_2 (\gamma) \} \leq   \left(  1 - p \right)  $ and $\pi_0 \leq \left(  2 H^2 n G^2 \left( 1 + \frac{\beta^2}{(1-\beta)^2} \right) \left( \frac{16}{\omega} + \frac{4}{p} \right) + 2\omega n \frac{c_0}{\eta^{(1-\epsilon)}} \right)$, where $p:= \frac{\gamma^* \delta}{8} $. Here $\gamma^* = \frac{2 \delta \omega}{64 \delta + \delta^2 + 16 \lambda^2 + 8 \delta \lambda^2 - 16\delta \omega}$ is the consensus step-size. Substituting these bounds and $\alpha_5 = \frac{2}{p}$ in \eqref{lemm_mom_cvx_interim} gives:
	\begin{align}
	e_{I_{(t+1)}} & \leq (1+\frac{p}{2}) \left( 1 - p \right)\mathbb{E}  \left[ \Vert \bX^{I_{(t)}} - \hat{\bX}^{I_{(t+1)}} \Vert_F^2 + \Vert \bX^{I_{(t)}} - \Bar{\bX}^{I_{(t)}} \Vert_F^2 \right] \notag \\ 
	& \qquad + \left(  2 H^2 n G^2 \left( 1 + \frac{\beta^2}{(1-\beta)^2} \right) \left( \frac{16}{\omega} + \frac{4}{p} \right) + 2\omega n \frac{c_0}{\eta^{(1-\epsilon)}} \right) \eta ^2. \label{eq:rec_lemma_final}
	\end{align}
	Note that $e_{I_{(t)}} = \mathbb{E} \left[ \Vert \bX^{I_{(t)}} - \Bar{\bX}^{I_{(t)}} \Vert_F^2 + \Vert \bX^{I_{(t)}} - \hat{\bX}^{I_{(t+1)}} \Vert_F^2 \right] $.
	We can write \eqref{eq:rec_lemma_final} as a recurrence relation for $e_{I_{(t)}}$ as: 
	\begin{align}  \label{suppl_lemm_rec_rel_li}
	e_{I_{(t+1)}} \leq \left( 1-\frac{p}{2} \right)e_{I_{(t)}} + \frac{2nA}{p}\eta^2.
	\end{align} 
	where $A := \frac{p}{2n}\left(  2 H^2 n G^2 \left( 1 + \frac{\beta^2}{(1-\beta)^2} \right) \left( \frac{16}{\omega} + \frac{4}{p} \right) + 2\omega n\frac{c_0}{\eta^{(1-\epsilon)}} \right) $.
Using \eqref{suppl_lemm_rec_rel_li}, it can be shown (proved in Lemma \ref{lem_noncvx_seq_bound} in Appendix~\ref{app:supporting-lemmas-for-consensus} below) that for all $I_{(t)} \in \mathcal{I}_T $, we have:   
\begin{align*}
e_{I_{(t)}} \leq \frac{4nA \eta^2}{p^2}
\end{align*}
Note that we also have: $ \mathbb{E}\Vert\Bar{\bX}^{I_{(t)}} - \bX^{I_{(t)}}\Vert_F^2 \leq \mathbb{E} \left[ \Vert\Bar{\bX}^{I_{(t)}} - \bX^{I_{(t)}}\Vert_F^2  + \Vert \hat{\bX}^{I_{(t+1)}} - \bX^{I_{(t)}} \Vert_F^2 \right] = e_{I_{(t)}}$. Thus, we get the following result for any synchronization index $I_{(t)} \in \mathcal{I}_T$:
\begin{align*}
\mathbb{E}\Vert\Bar{\bX}^{I_{(t)}} - \bX^{I_{(t)}}\Vert_F^2 \leq \frac{4nA\eta^2}{p^2},
\end{align*}
where $A =  \frac{p}{2}\left(  2 H^2 G^2 \left( 1 + \frac{\beta^2}{(1-\beta)^2} \right) \left( \frac{16}{\omega} + \frac{4}{p} \right) + 2\omega \frac{c_0}{\eta^{1-\epsilon}} \right) $ for $p = \frac{\delta \gamma^*}{8}$ , $\epsilon >0$ and 	$\gamma^* =  \frac{2 \delta \omega}{64 \delta + \delta^2 + 16 \beta^2 + 8 \delta \beta^2 - 16\delta \omega} $ is the chosen consensus step size.  This completes the proof for Lemma \ref{lem_noncvx}

\end{proof}

\subsection{Supporting Lemmas for Proving Lemma \ref{lem_noncvx}}\label{app:supporting-lemmas-for-consensus}
\begin{lemma}
	\label{suppl_lem_coeff_calc}
	Consider the following variables:
	\begin{align*}
	\pi_1 (\gamma) & := \gamma^2 \lambda^2 (1+\alpha_1^{-1}) + (1+\gamma \lambda)^2 (1+\alpha_4) (1+\alpha_3) (1+\alpha_2) (1-\omega) \\
	\pi_2 (\gamma) & := (1-\delta \gamma)^2(1+\alpha_1) + \gamma^2 \lambda^2 (1+\alpha_4^{-1}) (1+\alpha_3) (1+\alpha_2) (1-\omega) \\
	\pi_0 & := 2 H^2 n G^2 \left( 1 + \frac{\beta^2}{(1-\beta)^2} \right)(1+\alpha_5)(\pi_1(\gamma) + \pi_2(\gamma) ) \\
	& + 2 H^2 n G^2 \left( 1 + \frac{\beta^2}{(1- \beta)^2} \right)((1+\alpha_2^{-1}) + (1-\omega)(1+\alpha_3^{-1})(1+\alpha_2) ) + (1+\alpha_2)\omega n\frac{c_0}{\eta^{(1-\epsilon)}}
	\end{align*}
	and the following choice of variables:
	\begin{align*}
	\alpha_1 := \frac{\gamma \delta}{2}, \, 
	\alpha_2 := \frac{\omega}{4}, \,
	\alpha_3 := \frac{\omega}{4}, \,
	\alpha_4 := \frac{\omega}{4}, \,
	\alpha_5 := \frac{2}{p} \\
	p := \frac{\delta \gamma^*}{8}, \,
	\gamma^* := \frac{2 \delta \omega}{64 \delta + \delta^2 + 16 \lambda^2 + 8 \delta \lambda^2 - 16\delta \omega} \\
	\end{align*}
	Then, it can be shown that:
	\begin{align*}
	\max \{ \pi_1 (\gamma^*) , \pi_2 (\gamma^*) \}  \leq 1 - \frac{\delta^2 \omega}{644} \hspace{0.3cm}, \hspace{0.3cm}
	\pi_0 \leq 2 H^2 n G^2 \left( 1 + \frac{\beta^2}{(1-\beta)^2} \right) \left( \frac{16}{\omega} + \frac{4}{p}  \right) + 2 \omega n\frac{c_0}{\eta^{(1-\epsilon)}}
	\end{align*}
\end{lemma}
\begin{proof}
	We adapt a part of the proof of [Theorem 1]\cite{singh2019sparq} to prove Lemma \ref{suppl_lem_coeff_calc}.	
	Consider:
	\begin{align*}
	(1+\alpha_4) (1+\alpha_3) (1+\alpha_2) (1-\omega) & = (1+\frac{\omega}{4})^3 (1-\omega) \\
	& = \left( 1 - \frac{\omega^4}{64} - \frac{11 \omega^3}{64} - \frac{9 \omega^2}{16} - \frac{\omega}{4} \right) \\
	& \leq \left( 1 - \frac{\omega}{4} \right)
	\end{align*}
	This gives us:
	\begin{align*}
	\pi_1(\gamma) \leq \gamma^2\lambda^2\left(1+\frac{2}{\gamma \delta} \right) + (1+\gamma\lambda)^2 \left( 1 - \frac{\omega}{4} \right)
	\end{align*}
	Noting that $\gamma^2 \leq \gamma $ (for $\gamma \leq 1 $ which is true for $\gamma^*$ ) and $\lambda \leq 2$, we have:
	\begin{align*}
	\pi_1(\gamma) \leq \lambda^2\left(\gamma +\frac{2\gamma}{\delta} \right) + (1+8\gamma) \left( 1 - \frac{\omega}{4} \right)
	\end{align*}
	Substituting value of $\gamma^*$ in above, it can be shown that:
	\begin{align*}
	\pi_1(\gamma^*) \leq 1 - \frac{\delta^2\omega}{4(64 \delta + \delta^2 + 16 \lambda^2 + 8 \delta \lambda^2 - 16\delta \omega)}
	\end{align*}
	Now we note that:
	\begin{align*}
	\pi_2 (\gamma) & =  (1-\delta \gamma)^2 \left(1+\frac{\delta \gamma}{2} \right) + \gamma^2 \lambda^2 \left( 1+\frac{4}{\omega} \right) \left(1+\frac{\omega}{4} \right)^2 (1-\omega) 
	\intertext{Noting the fact that for $x = \delta \gamma \leq 1 $, we have $(1 - x)^2 \left(1 + \frac{x}{2}\right) \leq (1-x)\left( 1 - \frac{x}{2}  \right) $,}
	\pi_2 (\gamma) & \leq \left( 1-\frac{\gamma \delta}{2} \right)^2 + \gamma^2 \lambda^2 \left( 1+\frac{4}{\omega} \right) \left(1+\frac{\omega}{4} \right)^2 (1-\omega)\\
	& = \left( 1-\frac{\gamma \delta}{2} \right)^2 + \gamma^2 \lambda^2 \left(3+\frac{3\omega}{4}+\frac{\omega^2}{16} + \frac{4}{\omega} \right) (1-\omega) \\
	& \leq \left( 1-\frac{\gamma \delta}{2} \right)^2 + \gamma^2 \lambda^2 \frac{4}{\omega} \, =: \zeta(\gamma)
	\end{align*}
	Note that $\zeta(\gamma)$ is convex and quadratic in $\gamma$, and attains minima at $\gamma' = \frac{2 \delta \omega}{16 \lambda^2 + \delta^2 \omega}$ with value $\zeta(\gamma') = \frac{16\lambda^2}{16\lambda^2 + \omega \delta^2}$.
	
	By the Jensen's inequality, we note that for any $s \in [0,1]$
	\begin{align*}
	\zeta(s \gamma') \leq (1-s) \zeta(0) + s \zeta(\gamma') = 1 - s \frac{\delta^2 \omega}{16 \lambda^2 + \delta^2 \omega}
	\end{align*}
	For the choice $s = \frac{16\lambda^2 + \omega \delta^2}{64 \delta + \delta^2 + 16 \lambda^2 + 8 \delta \lambda^2 - 16\delta \omega}$, it can be seen that $s \gamma' = \gamma^*$. Thus we get:
	\begin{align*}
	\pi_2(\gamma^*) \leq  \zeta(s \gamma')  & \leq 1 - \frac{\delta^2\omega}{(64 \delta + \delta^2 + 16 \lambda^2 + 8 \delta \lambda^2 - 16\delta \omega)} \\
	& \leq 1 - \frac{\delta^2\omega}{4(64 \delta + \delta^2 + 16 \lambda^2 + 8 \delta \lambda^2 - 16\delta \omega)}
	\end{align*}
	Thus we have:
	\begin{align*}
	\max \{ \pi_1 (\gamma^*) , \pi_2 (\gamma^*) \} & \leq 1 - \frac{\delta^2\omega}{4(64 \delta + \delta^2 + 16 \lambda^2 + 8 \delta \lambda^2 - 16\delta \omega)}.
	\intertext{Using the value of $\gamma^*$ given in the lemma statement, we have $ \frac{\delta^2\omega}{4(64 \delta + \delta^2 + 16 \lambda^2 + 8 \delta \lambda^2 - 16\delta \omega)} = \frac{\delta \gamma^*}{8}$. Define $p : = \frac{\gamma^* \delta}{8}$. 
		Using crude estimates $\delta \leq 1, \omega \geq 0, \lambda \leq2$, we can lower-bound $p$ as $p \geq \frac{\delta^2 \omega}{644}$. Thus we have }
	\max \{ \pi_1 (\gamma^*) , \pi_2 (\gamma^*) \} & \leq 1 - \frac{\delta^2 \omega}{644}.
	\end{align*}
	Now we upper-bound the value of $\pi_0$:
	\begin{align*}
	\pi_0 & := 2 H^2 n G^2 \left( 1 + \frac{\beta^2}{(1-\beta)^2} \right)(1+\alpha_5)(\pi_1(\gamma) + \pi_2(\gamma) )  + (1+\alpha_2)\omega n\frac{c_0}{\eta^{(1-\epsilon)}} \\
	& \qquad + 2 H^2 n G^2 \left( 1 + \frac{\beta^2}{(1-\beta)^2} \right)((1+\alpha_2^{-1}) + (1-\omega)(1+\alpha_3^{-1})(1+\alpha_2) ) \\
	& \leq 4 H^2 n G^2 \left( 1 + \frac{\beta^2}{(1-\beta)^2} \right)(1+\frac{2}{p})(1-p) + (1+\frac{\omega}{4})\omega n\frac{c_0}{\eta^{(1-\epsilon)}} \\
	& \qquad + 2 H^2 n G^2 \left( 1 + \frac{\beta^2}{(1-\beta)^2} \right)((1+\frac{4}{\omega}) + (1-\omega)(1+\frac{4}{\omega})(1+\frac{\omega}{4}) )  \\
	& \leq 4 H^2 n G^2 \left( 1 + \frac{\beta^2}{(1-\beta)^2} \right)\frac{2}{p} + (1+\frac{\omega}{4})\omega n\frac{c_0}{\eta^{(1-\epsilon)}} + 2 H^2 n G^2 \left( 1 + \frac{\beta^2}{(1-\beta)^2}\right)(1+\frac{8}{\omega} )  \\
	\end{align*}
	Where in the first inequality we have used the fact that $\pi_1(\gamma) + \pi_2(\gamma) \leq 2(1-p) $. In the second inequality, we use the fact that $(1+\frac{2}{p})(1-p) \leq \frac{2}{p} $ and $(1-\omega)(1+\frac{4}{\omega})(1+\frac{\omega}{4}) \leq \frac{4}{\omega}$. Noting that for $\omega \leq 1$, we have $(1+\frac{\omega}{4}) \leq 2$ and $\left(1 + \frac{8}{\omega}\right) \leq \frac{16}{\omega}$. Using these, we have:
	\begin{align*}
	\pi_0 & \leq 2 H^2 n G^2 \left( 1 + \frac{\beta^2}{(1-\beta)^2} \right) \left(\frac{4}{p} + \frac{16}{\omega} \right) + 2\omega nH c_t.
	\end{align*}
	This completes the proof of Lemma \ref{suppl_lem_coeff_calc}.
\end{proof}

\begin{lemma} \label{lem_noncvx_seq_bound}
	Consider the sequence \{$e_{I_{(t)}}$\} given by
	\begin{align*}
	e_{I_{(t+1)}} \leq \left( 1-\frac{p}{2} \right)e_{I_{(t)}} + \frac{2nA}{p}\eta^2,
	\end{align*}
	where $ \mathcal{I}_T = \{ I_{(1)}, I_{(2)} , \hdots , I_{(t)}, \hdots  \} \in [T] $ denotes the set of synchronization indices. For a parameter $p >0$, positive constants $A$ and $\eta$ , we have:
	\begin{align*}
	e_{I_{(t)}} \leq \frac{4nA}{p^2}\eta^2
	\end{align*}
\end{lemma}
\begin{proof}
	The proof uses an induction argument. Note that the base case is satisfied as $e_0 = 0$. Assuming the bound holds for $e_{I_{(t)}}$, for $e_{I_{(t+1)}}$, we have: 
	\begin{align*}
	e_{I_{(t+1)}} & \leq (1-\frac{p}{2}) \frac{4nA \eta^2}{p^2} + \frac{2nA \eta^2}{p} \\
	& = \frac{4nA \eta^2}{p^2} 
	\end{align*}
	Thus $	e_{I_{(t)}} \leq \frac{4nA}{p^2}\eta^2$ for all $I_{(t)} \in \mathcal{I}_T$ from induction argument, which completes the proof.
\end{proof}

\section{Supporting Lemmas for Proof of Lemma \ref{lem_noncvx}} \label{suppl_recursion_lemm_proofs}

As discussed in Appendix \ref{proof_lem_consensus}, the proof for Lemma \ref{lem_noncvx} relies on establishing a recurrence relation between two quantities of interest:
$e_{I_{(t)}}^{(1)} :=\sum_{j=1}^n \mathbb{E} \left\Vert \bar{\bx}^{I_{(t)}} - \bx^{I_{(t)}}_{j} \right\Vert^2$ -- the average deviation of local parameter copies and the global parameter -- and  $e_{I_{(t)}}^{(2)}:=\sum_{j=1}^n \mathbb{E} \left\Vert \hat{\bx}^{I_{(t+1)}} - \bx^{I_{(t)}}_{j} \right\Vert^2$ -- the average deviation of the local parameter and their copies. In this section, we provide a recursion relation for both $e_{I_{(t+1)}}^{(1)}$ and $e_{I_{(t+1)}}^{(2)}$, each in terms of $e_{I_{(t)}}^{(1)}$ and $e_{I_{(t)}}^{(2)}$. These results are stated in Lemma \ref{lem_cvx_e1} and \ref{lem_cvx_e2}, respectively, which we prove below. In order to prove these lemmas we use some techniques from proof of Lemma 1 and Lemma 2 in \cite{singh2019sparq}. \\
 In matrix notation, these quantities are given by:
\begin{align*}
e_{I_{(t+1)}}^{(1)} &= \mathbb{E} \Vert \bX^{I_{(t+1)}} - \Bar{\bX}^{I_{(t+1)}} \Vert_F^2 \\
e_{I_{(t+1)}}^{(2)} &= \mathbb{E} \Vert \bX^{I_{(t+1)}} -\hat{\bX}^{I_{(t+2)}}\Vert_F^2
\end{align*}

\subsection{Proof of Lemma \ref{lem_cvx_e1}}

	Using the update equations of $\bX^{I_{(t+1)}}$ in matrix form given in \eqref{mat_not_algo_1}-\eqref{mat_not_algo_2} in Section \ref{prelim}, we have:
	\begin{align}
	\Vert \bX^{I_{(t+1)}} - \Bar{\bX}^{I_{(t+1)}} \Vert_F^2 & = \Vert \bX^{I_{(t+\frac{1}{2})}} - \Bar{\bX}^{I_{(t+1)}} + \gamma \hat{\bX}^{I_{(t+1)}} (\mathbf{W}-\mathbf{I}) \Vert_F^2 \notag 
	\intertext{Noting that $\Bar{\bX}^{I_{(t+1)}} = \Bar{\bX}^{I_{(t+\frac{1}{2})}} $ (from (\ref{mean_seq_iter})) and $\Bar{\bX}^{I_{(t+\frac{1}{2})}} (\mathbf{W}-\mathbf{I})= 0$ (from (\ref{mean_prop})), we get: }
	\Vert \bX^{I_{(t+1)}} - \Bar{\bX}^{I_{(t+1)}} \Vert_F^2 & = 
	\Vert (\bX^{I_{(t+\frac{1}{2})}} - \Bar{\bX}^{I_{(t+\frac{1}{2})}})((1-\gamma)\mathbf{I} + \gamma \mathbf{W})+ \gamma (\hat{\bX}^{I_{(t+1)}}-\bX^{I_{(t+\frac{1}{2})}}) (\mathbf{W}-\mathbf{I}) \Vert_F^2 \notag
	\end{align}
	For any positive constant\footnote{\label{foot:frob_bound}For any two matrices $\mathbf{A},\mathbf{B} \in \mathbb{R}^{p \times q} $ and for any $\alpha >0$ , we have the following relationship for the Frobenius norm:
		\begin{align*}
		\verts{\mathbf{A} + \mathbf{B} }_F^2 \leq (1 + \alpha) \verts{\mathbf{A}}_F^2 + (1 + \alpha^{-1})\verts{\mathbf{B}}_F^2		
		\end{align*}} $\alpha_1$, we have:	
	\begin{align}
	\Vert \bX^{I_{(t+1)}} - \Bar{\bX}^{I_{(t+1)}} \Vert_F^2 & \leq (1+\alpha_1)\Vert (\bX^{I_{(t+\frac{1}{2})}} - \Bar{\bX}^{I_{(t+\frac{1}{2})}})((1-\gamma)\mathbf{I} + \gamma \mathbf{W})\Vert_F^2 \notag \\
	& \qquad \qquad \qquad  + (1+\alpha_1^{-1}) \Vert \gamma (\hat{\bX}^{I_{(t+1)}}-\bX^{I_{(t+\frac{1}{2})}}) (\mathbf{W}-\mathbf{I}) \Vert_F^2 \notag \\
	\intertext{Using $\Vert \mathbf{A} \mathbf{B} \Vert_F \leq \Vert \mathbf{A} \Vert_F \Vert \mathbf{B} \Vert_2 $ for any matrices $\mathbf{A},\mathbf{B}$, we have: }
	\Vert \bX^{I_{(t+1)}} - \Bar{\bX}^{I_{(t+1)}} \Vert_F^2 &  \leq (1+\alpha_1)\Vert (\bX^{I_{(t+\frac{1}{2})}} - \Bar{\bX}^{I_{(t+\frac{1}{2})}})((1-\gamma)\mathbf{I} + \gamma \mathbf{W})\Vert_F^2 \notag \\
	& \qquad \qquad \qquad + (1+\alpha_1^{-1}) \gamma^2 \Vert  (\hat{\bX}^{I_{(t+1)}}-\bX^{I_{(t+\frac{1}{2})}})\Vert_F^2 .\Vert(\mathbf{W}-\mathbf{I})\Vert_2^2 \label{suppl_cvx_li_temp_eqn}
	\end{align}
	To bound the first term in (\ref{suppl_cvx_li_temp_eqn}), we use the triangle inequality for Frobenius norm, giving us:
	\begin{align*}
	\Vert (\bX^{I_{(t+\frac{1}{2})}} - \Bar{\bX}^{I_{(t+\frac{1}{2})}})((1-\gamma)\mathbf{I} + \gamma \mathbf{W})\Vert_F & \leq (1-\gamma)\Vert \bX^{I_{(t+\frac{1}{2})}} - \Bar{\bX}^{I_{(t+\frac{1}{2})}} \Vert_F+ \gamma \Vert (\bX^{I_{(t+\frac{1}{2})}} - \Bar{\bX}^{I_{(t+\frac{1}{2})}} )\mathbf{W}\Vert_F
	\end{align*}
	{Since $\left(\bX^{I_{(t+\frac{1}{2})}} - \Bar{\bX}^{I_{(t+\frac{1}{2})}}\right)\frac{\mathbf{1}\mathbf{1}^T}{n} = \bzero$ (from \eqref{mean_prop}), adding this inside the last term above, we get: }
	\begin{align*}
	\Vert (\bX^{I_{(t+\frac{1}{2})}} - \Bar{\bX}^{I_{(t+\frac{1}{2})}})((1-\gamma)\mathbf{I} + \gamma \mathbf{W})\Vert_F & \leq (1-\gamma)\Vert \bX^{I_{(t+\frac{1}{2})}} - \Bar{\bX}^{I_{(t+\frac{1}{2})}} \Vert_F \\
	& + \gamma \left\Vert (\bX^{I_{(t+\frac{1}{2})}} - \Bar{\bX}^{I_{(t+\frac{1}{2})}})\left (\mathbf{W} - \frac{\mathbf{1}\mathbf{1}^T}{n} \right) \right\Vert_F 
	\end{align*}
	Using $\Vert \mathbf{A} \mathbf{B} \Vert_F \leq \Vert \mathbf{A} \Vert_F \Vert \mathbf{B} \Vert_2 $ and then using (\ref{bound_W_mat}) from Fact 3 with $k=1$, we can simplify the above to:
	\begin{align*}
	\Vert (\bX^{I_{(t+\frac{1}{2})}} - \Bar{\bX}^{I_{(t+\frac{1}{2})}})((1-\gamma)\mathbf{I} + \gamma \mathbf{W})\Vert_F \leq (1-\gamma \delta) \Vert \bX^{I_{(t+\frac{1}{2})}} - \Bar{\bX}^{I_{(t+\frac{1}{2})}} \Vert_F
	\end{align*}
	Substituting the above in (\ref{suppl_cvx_li_temp_eqn}) and using $\lambda = \text{max}_i \{ 1- \lambda_i(\mathbf{W}) \} \Rightarrow \Vert \mathbf{W}-\mathbf{I} \Vert_2^2 \leq \lambda^2 $, we get: 
	\begin{align*}
	\Vert \bX^{I_{(t+1)}} - \Bar{\bX}^{I_{(t+1)}} \Vert_F^2 \leq (1+\alpha_1)(1-\gamma \delta)^2  \Vert \bX^{I_{(t+\frac{1}{2})}} - \Bar{\bX}^{I_{(t+\frac{1}{2})}} \Vert_F^2 + (1+\alpha_1^{-1}) \gamma^2 \lambda^2 \Vert  {\bX}^{I_{(t+\frac{1}{2})}}-\hat{\bX}^{I_{(t+1)}}\Vert_F^2
	\end{align*}
	Taking expectation w.r.t.\ the entire process, we have:
	\begin{align*}
	\mathbb{E} \Vert \bX^{I_{(t+1)}} - \Bar{\bX}^{I_{(t+1)}} \Vert_F^2 & \leq (1+\alpha_1)(1-\gamma \delta)^2 \mathbb{E} \Vert \bX^{I_{(t+\frac{1}{2})}} - \Bar{\bX}^{I_{(t+\frac{1}{2})}} \Vert_F^2 + (1+\alpha_1^{-1}) \gamma^2 \lambda^2 \mathbb{E} \Vert  {\bX}^{I_{(t+\frac{1}{2})}}-\hat{\bX}^{I_{(t+1)}}\Vert_F^2
	\end{align*}
	Define $R_1=(1+\alpha_1)(1-\gamma \delta)^2 , R_2=(1+\alpha_1^{-1}) \gamma^2 \lambda^2 $.
	Using the update steps of algorithm given in equations \eqref{mat_not_algo_4} and \eqref{mean_seq_iter} (given in Section \ref{prelim}), we have:
	\begin{align*}
	\mathbb{E} \Vert \bX^{I_{(t+1)}} - \Bar{\bX}^{I_{(t+1)}} \Vert_F^2 & \leq R_1 \mathbb{E} \left\Vert  \Bar{\bX}^{I_{(t)}} - \bX^{I_{(t)}} -  \sum_{t' = I_{(t)}}^{I_{(t+1)}-1} \eta ( \beta \bV^{(t')} +  \boldsymbol{\nabla F}(\bX^{(t')}, \boldsymbol{\xi}^{(t')} ))\left(   \frac{\mathbf{1}\mathbf{1}^T}{n} - I \right) \right\Vert_F^2 \notag \\
	& \qquad + R_2 \mathbb{E} \left\Vert   \hat{\bX}^{I_{(t+1)}} - \bX^{I_{(t)}} + \sum_{t' = I_{(t)}}^{I_{(t+1)}-1} \eta ( \beta \bV^{(t')} + \boldsymbol{\nabla F}(\bX^{(t')}, \boldsymbol{\xi}^{(t')} )) \right\Vert_F^2
	\end{align*} 	
	Thus, for any $\alpha_5 > 0 $ (using Footnote \ref{foot:frob_bound}), we have:
	\begin{align*}
	\mathbb{E} \Vert \bX^{I_{(t+1)}} - \Bar{\bX}^{I_{(t+1)}} \Vert_F^2 & \leq R_1 (1+\alpha_5^{-1}) \mathbb{E} \left\Vert  \Bar{\bX}^{I_{(t)}} - \bX^{I_{(t)}} \right\Vert^2 + R_2 (1+\alpha_5^{-1}) \mathbb{E} \left\Vert \hat{\bX}^{I_{(t+1)}} - \bX^{I_{(t)}} \right\Vert^2 \notag \\
	 & \qquad + R_1(1+\alpha_5)  \mathbb{E} \left\Vert \sum_{t' = I_{(t)}}^{I_{(t+1)}-1} \eta ( \beta \bV^{(t')} +   \boldsymbol{\nabla F}(\bX^{(t')}, \boldsymbol{\xi}^{(t')} ))\left(   \frac{{\bf 1}{\bf 1}^T}{n} - I \right) \right\Vert_F^2 \notag \\
	 & \qquad + R_2 (1+\alpha_5) \mathbb{E} \left\Vert \sum_{t' = I_{(t)}}^{I_{(t+1)}-1} \eta ( \beta \bV^{(t')} +   \boldsymbol{\nabla F}(\bX^{(t')}, \boldsymbol{\xi}^{(t')} )) \right\Vert_F^2
	\end{align*}
	Using $\verts{\mathbf{AB}}_F \leq \verts{\mathbf{A}}_F \verts{\mathbf{B}}_2$ to split the third term, and then using the bound $\left\| \frac{\mathbf{1}\mathbf{1}^T}{n} - \mathbf{I} \right\|_2 =1$ (which is shown in Claim \ref{J-I_eigenvalue} in Appendix \ref{app:details-stronger-assump}), and further using the bound in \eqref{bound_gap_grad} for the third and the fourth terms, the above can be rewritten as:  
	\begin{align*}
	\mathbb{E} \Vert \bX^{I_{(t+1)}} - \Bar{\bX}^{I_{(t+1)}} \Vert_F^2 & \leq R_1 (1+\alpha_5^{-1}) \mathbb{E} \left\Vert  \Bar{\bX}^{I_{(t)}} - \bX^{I_{(t)}} \right\Vert^2 + R_2 (1+\alpha_5^{-1}) \mathbb{E} \left\Vert \hat{\bX}^{I_{(t+1)}} - \bX^{I_{(t)}} \right\Vert^2 \notag \\
	& \qquad + 2\eta^2 H^2 n G^2 \left( 1 + \frac{\beta^2}{(1-\beta)^2} \right)(1+\alpha_5)(R_1+R_2)
	\end{align*}
	Defining $Q_1 = 2 H^2 n G^2 \left( 1 + \frac{\beta^2}{(1-\beta)^2} \right)(1+\alpha_5)(R_1+R_2) $ completes the proof of Lemma \ref{lem_cvx_e1}.

\subsection{Proof of Lemma \ref{lem_cvx_e2}}
	Since $\hat{\bX}^{I_{(t+2)}} = \hat{\bX}^{I_{(t+1)}} + \C ( (\bX^{I_{(t+\frac{3}{2})}} - \hat{\bX}^{I_{(t+1)}}) \mathbf{P}^{(I_{(t+2)}-1)} )$ (from \eqref{mat_not_algo_3} in Section \ref{prelim}),
	we have:
	\begin{align*}
	& e_{I_{(t+1)}}^{(2)} = \mathbb{E} \Vert \bX^{I_{(t+1)}} - \hat{\bX}^{I_{(t+2)}}\Vert_F^2  = \mathbb{E} \Vert \bX^{I_{(t+1)}} - \hat{\bX}^{I_{(t+1)}} - \C ( (\bX^{I_{(t+\frac{3}{2})}} - \hat{\bX}^{I_{(t+1)}}) \mathbf{P}^{(I_{(t+2)}-1)} )\Vert_F^2 \\
	&  \hspace{2.4cm} = \mathbb{E} \Vert \bX^{I_{(t+\frac{3}{2})}} - \hat{\bX}^{I_{(t+1)}}+ \bX^{I_{(t+1)}} - \bX^{I_{(t+\frac{3}{2})}} - \C ( (\bX^{I_{(t+\frac{3}{2})}} - \hat{\bX}^{I_{(t+1)}}) \mathbf{P}^{(I_{(t+2)}-1)} )\Vert_F^2 
	\end{align*}
	For any $\alpha_2 >0$, using result from Footnote \ref{foot:frob_bound}, we have:
	\begin{align}
	\mathbb{E} \Vert \bX^{I_{(t+1)}} - \hat{\bX}^{I_{(t+2)}}\Vert_F^2 & \leq (1+\alpha_2) \mathbb{E} \Vert \bX^{I_{(t+\frac{3}{2})}} - \hat{\bX}^{I_{(t+1)}} - \C ( (\bX^{I_{(t+\frac{3}{2})}}- \hat{\bX}^{I_{(t+1)}}) \mathbf{P}^{(I_{(t+2)}-1)} )\Vert_F^2 \notag \\
	&\hspace{1cm} + (1+\alpha_2^{-1})\mathbb{E}\Vert \bX^{I_{(t+1)}} - \bX^{I_{(t+\frac{3}{2})}} \Vert_F^2 \label{eq:lem_copy_param_bound}
	\end{align} 
	The last term in R.H.S. of \eqref{eq:lem_copy_param_bound} can be bounded by using the update step \eqref{mat_not_algo_4} and then using \eqref{bound_gap_grad} from Fact \ref{fact-bound}, which gives:
	\begin{align}
	\mathbb{E}\Vert \bX^{I_{(t+1)}} - \bX^{I_{(t+\frac{3}{2})}} \Vert_F^2 \leq 2\eta ^2 H^2 n G^2 \left( 1 + \frac{\beta^2}{(1-\beta)^2} \right) \label{eq:lem_copy_bound}
	\end{align}	
	Using the bound \eqref{eq:lem_copy_bound} in \eqref{eq:lem_copy_param_bound}, we get:
	\begin{align*}
	\mathbb{E} \Vert \bX^{I_{(t+1)}} - \hat{\bX}^{I_{(t+2)}}\Vert_F^2 & \leq (1+\alpha_2) \mathbb{E} \Vert \bX^{I_{(t+\frac{3}{2})}} - \hat{\bX}^{I_{(t+1)}} - \C ( (\bX^{I_{(t+\frac{3}{2})}} - \hat{\bX}^{I_{(t+1)}}) \mathbf{P}^{(I_{(t+2)}-1)} )\Vert_F^2 \\
	& \qquad + (1+\alpha_2^{-1}) 2\eta ^2 H^2 n G^2 \left( 1 + \frac{\beta^2}{(1-\beta)^2} \right)
	\end{align*}
	Note that both $\mathbf{P}^{(I_{(t+2)}-1)} $ and $\mathbf{I} - \mathbf{P}^{(I_{(t+2)}-1)} $ are diagonal matrices, with disjoint support on the diagonal entries, which implies that $\mathbb{E} \Vert \bX^{I_{(t+\frac{3}{2})}} - \hat{\bX}^{I_{(t+1)}}\|_F^2 = \mathbb{E} \Vert (\bX^{I_{(t+\frac{3}{2})}} - \hat{\bX}^{I_{(t+1)}})\mathbf{P}^{(I_{(t+2)}-1)}\|_F^2 + \mathbb{E} \Vert (\bX^{I_{(t+\frac{3}{2})}} - \hat{\bX}^{I_{(t+1)}})({\bf I} - \mathbf{P}^{(I_{(t+2)}-1)})\|_F^2$. We get:
	\begin{align*}
	\mathbb{E} \Vert \bX^{I_{(t+1)}} - & \hat{\bX}^{I_{(t+2)}}\Vert_F^2  \leq  (1+\alpha_2)  \mathbb{E} \Vert (\bX^{I_{(t+\frac{3}{2})}} - \hat{\bX}^{I_{(t+1)}})\mathbf{P}^{(I_{(t+2)}-1)} - \C ( (\bX^{I_{(t+\frac{3}{2})}} - \hat{\bX}^{I_{(t+1)}}) \mathbf{P}^{(I_{(t+2)}-1)} )\Vert_F^2  \\
	&   + (1+\alpha_2) \mathbb{E}\Vert (\bX^{I_{(t+\frac{3}{2})}} - \hat{\bX}^{I_{(t+1)}})(\mathbf{I} - \mathbf{P}^{(I_{(t+2)}-1)})\Vert_F^2  + 2 (1+\alpha_2^{-1}) \eta ^2 H^2 n G^2 \left( 1 + \frac{\beta^2}{(1-\beta)^2} \right)
	\end{align*}
	Using the compression property \eqref{eq:compression} of operator $\C$, we have:
	\begin{align*}
	& \mathbb{E} \Vert \bX^{I_{(t+1)}} - \hat{\bX}^{I_{(t+2)}}\Vert_F^2 \leq (1+\alpha_2) (1-\omega) \mathbb{E}  \Vert (\bX^{I_{(t+\frac{3}{2})}} - \hat{\bX}^{I_{(t+1)}})\mathbf{P}^{(I_{(t+2)}-1)} \Vert_F^2 \\
	& \hspace{1cm} +  (1+\alpha_2) \mathbb{E} \Vert (\bX^{I_{(t+\frac{3}{2})}} - \hat{\bX}^{I_{(t+1)}})(\mathbf{I} - \mathbf{P}^{(I_{(t+2)}-1)})\Vert_F^2  + 2 (1+\alpha_2^{-1}) \eta ^2 H^2 n G^2 \left( 1 + \frac{\beta^2}{(1-\beta)^2} \right)
	\end{align*}
	Adding and subtracting $ (1+\alpha_2)(1-\omega)  \mathbb{E}\Vert (\bX^{I_{(t+\frac{3}{2})}} - \hat{\bX}^{I_{(t+1)}})(\mathbf{I} - \mathbf{P}^{(I_{(t+2)}-1)})\Vert_F^2$, we get:
	\begin{align*}
	\mathbb{E} \Vert \bX^{I_{(t+1)}} - \hat{\bX}^{I_{(t+2)}}\Vert_F^2 & \leq (1+\alpha_2) (1-\omega)  \mathbb{E} \Vert \bX^{I_{(t+\frac{3}{2})}} - \hat{\bX}^{I_{(t+1)}} \Vert_F^2  + (1+\alpha_2^{-1})2\eta ^2 H^2 n G^2 \left( 1 + \frac{\beta^2}{(1-\beta)^2} \right)  \\
	& \qquad \qquad +(1+\alpha_2) \omega  \mathbb{E} \Vert (\bX^{I_{(t+\frac{3}{2})}} - \hat{\bX}^{I_{(t+1)}})(\mathbf{I} - \mathbf{P}^{(I_{(t+2)}-1)})\Vert_F^2 
	\end{align*}
	To bound the third term in the RHS above, note that $\hat{\bX}^{I_{(t+2)}-1} = \hat{\bX}^{I_{(t+1)}}$, because $\hat{\bX}$ does not change in between the synchronization indices,
	which implies that $\mathbb{E} \Vert (\bX^{I_{(t+\frac{3}{2})}} - \hat{\bX}^{I_{(t+1)}})(\mathbf{I} - \mathbf{P}^{(I_{(t+2)}-1)})\Vert_F^2 = \mathbb{E} \Vert (\bX^{I_{(t+\frac{3}{2})}} - \hat{\bX}^{I_{(t+2)}-1})(\mathbf{I} - \mathbf{P}^{(I_{(t+2)}-1)})\Vert_F^2$, which we can upper-bound using \eqref{bound_trig} by $nc_{I_{(t+2)}-1}\eta^2$. 
	Using $c_t \leq \frac{c_0}{\eta^{(1-\epsilon)}}$ for all $t$, we get: 
	\begin{align} \label{cvx_lemm_interim2}
	\mathbb{E} \Vert \bX^{I_{(t+1)}} - \hat{\bX}^{I_{(t+2)}}\Vert_F^2  & \leq (1+\alpha_2) (1-\omega) \mathbb{E} \Vert \bX^{I_{(t+\frac{3}{2})}} - \hat{\bX}^{I_{(t+1)}} \Vert_F^2 + (1+\alpha_2)\omega nc_0 \eta^{(1+\epsilon)} \notag \\
	& \qquad + (1+\alpha_2^{-1}) 2\eta ^2 H^2 n G^2 \left( 1 + \frac{\beta^2}{(1-\beta)^2} \right)
	\end{align}
	We now bound the first term in the R.H.S. of \eqref{cvx_lemm_interim2}. From the update equation \eqref{mat_not_algo_4}, we have:  
	\begin{align}
		\mathbb{E} \Vert \bX^{I_{(t+\frac{3}{2})}} - & \hat{\bX}^{I_{(t+1)}} \Vert_F^2  =  \mathbb{E} \left\Vert \bX^{I_{(t+1)}}  - \sum_{t' = I_{(t+1)}}^{I_{(t+2)}-1}\eta (\beta \bV^{(t')}+  \boldsymbol{\nabla F} (\bX^{(t')}, \boldsymbol{\xi}^{(t')})) -  \hat{\bX}^{I_{(t+1)}} \right\Vert_F^2 \notag \\
		& \leq (1+\alpha_3) \mathbb{E} \Vert \bX^{I_{(t+1)}}  -  \hat{\bX}^{I_{(t+1)}} \Vert_F^2 + (1+\alpha_3^{-1}) 2\eta ^2H^2 n G^2 \left( 1 + \frac{\beta^2}{(1-\beta)^2} \right) \label{eq:lemm_bound_copy_param_3}
	\end{align}
	where for the last inequality, $\alpha_3$ is any positive constant (from Footnote \ref{foot:frob_bound}) and we have used \eqref{bound_gap_grad} from Fact \ref{fact-bound}. Substituting the bound \eqref{eq:lemm_bound_copy_param_3} in \eqref{cvx_lemm_interim2}, we have:
	\begin{align}
	\mathbb{E}  \Vert \bX^{I_{(t+1)}} -  \hat{\bX}^{I_{(t+2)}}\Vert_F^2  & \leq (1+\alpha_3) (1+\alpha_2) (1-\omega) \mathbb{E} \Vert \bX^{I_{(t+1)}}  -  \hat{\bX}^{I_{(t+1)}} \Vert_F^2 \notag \\
	& \quad +  (1+\alpha_3^{-1}) (1+\alpha_2) (1-\omega) 2\eta ^2H^2 n G^2 \left( 1 + \frac{\beta^2}{(1-\beta)^2} \right) \notag \\
	& \quad + (1+\alpha_2)\omega n c_0 \eta^{(1+\epsilon)} + (1+\alpha_2^{-1}) 2\eta ^2 H^2 n G^2 \left( 1 + \frac{\beta^2}{(1-\beta)^2} \right) \label{eq:copy_param_bound_3}
	\end{align}
	We now bound the first term in R.H.S. of \eqref{eq:copy_param_bound_3}. From the update equation \eqref{mat_not_algo_2} and using the fact that $ \Bar{\bX}^{I_{(t+\frac{1}{2})}} (\mathbf{W}-\mathbf{I})= 0 $, we have:
	\begin{align}
		&\mathbb{E} \Vert \bX^{I_{(t+1)}}  -  \hat{\bX}^{I_{(t+1)}} \Vert_F^2  = \mathbb{E} \Vert (\bX^{I_{(t+\frac{1}{2})}} - \hat{\bX}^{I_{(t+1)}})((1+\gamma)\mathbf{I} - \gamma \mathbf{W})+ \gamma (\bX^{I_{(t+\frac{1}{2})}}- \Bar{\bX}^{I_{(t+\frac{1}{2})}})(\mathbf{W}-\mathbf{I})\Vert_F^2 \notag \\
	& \quad \leq (1+\alpha_4) (1+\gamma \lambda)^2 \mathbb{E} \Vert \bX^{I_{(t+\frac{1}{2})}} - \hat{\bX}^{I_{(t+1)}}\Vert_F^2 + \gamma^2 \lambda^2 (1+\alpha_4^{-1}) \mathbb{E} \Vert \bX^{I_{(t+\frac{1}{2})}}- \Bar{\bX}^{I_{(t+\frac{1}{2})}} \Vert_F^2 \label{eq:param_copy_temp2}
	\end{align}
	where $\alpha_4$ is any positive constant (from Footnote \ref{foot:frob_bound}) and the fact that $\Vert (1+\gamma)\mathbf{I} - \gamma \mathbf{W} \Vert_2 = \Vert I + \gamma (\mathbf{I} - \mathbf{W})\Vert_2 = 1 +  \gamma \Vert \mathbf{I} - \mathbf{W} \Vert_2  = 1 + \gamma \lambda $ (by definition of $\lambda = \text{max}_i \{ 1- \lambda_i(\mathbf{W}) \}$) and $\verts{\mathbf{I} - \mathbf{W}}_2 = \lambda$  along with $\verts{\mathbf{AB}}_F \leq \verts{\mathbf{A}}_F\verts{\mathbf{B}}_2$. 
	Using the bound from \eqref{eq:param_copy_temp2} in \eqref{eq:copy_param_bound_3}, we get:
	\begin{align*}
	\mathbb{E} \Vert \bX^{I_{(t+1)}} - \hat{\bX}^{I_{(t+2)}}\Vert_F^2	& \leq (1+\gamma \lambda )^2 (1+\alpha_4) (1+\alpha_3) (1+\alpha_2) (1-\omega)\mathbb{E} \Vert \bX^{I_{(t+\frac{1}{2})}} - \hat{\bX}^{I_{(t+1)}}\Vert_F^2 \notag \\
	& \qquad + \gamma^2 \lambda^2 (1+\alpha_4^{-1}) (1+\alpha_3) (1+\alpha_2) (1-\omega) \mathbb{E} \Vert \bX^{I_{(t+\frac{1}{2})}}- \Bar{\bX}^{I_{(t+\frac{1}{2})}} \Vert_F^2 \notag \\
	& \qquad +2\left(  ( 1+\alpha_2^{-1}) + (1+\alpha_3^{-1}) (1+\alpha_2) (1-\omega) \right)\eta ^2 H^2 n G^2 \left( 1 + \frac{\beta^2}{(1-\beta)^2} \right) \\
	& \qquad  + (1+\alpha_2)\omega nc_0 \eta^{(1+\epsilon)}
	\end{align*}
	Define $R_3 = (1+\gamma \lambda )^2 (1+\alpha_4) (1+\alpha_3) (1+\alpha_2) (1-\omega)$ , $R_4 = \gamma^2 \lambda^2 (1+\alpha_4^{-1}) (1+\alpha_3) (1+\alpha_2) (1-\omega) $ and $R_5 = 2\left(  ( 1+\alpha_2^{-1}) + (1+\alpha_3^{-1}) (1+\alpha_2) (1-\omega) \right) H^2 n G^2 \left( 1 + \frac{\beta^2}{(1-\beta)^2} \right) + (1+\alpha_2)\omega n \frac{c_0}{\eta^{(1-\epsilon)} }  $, then the above can be rewritten as :
	\begin{align*}
	\mathbb{E} \Vert \bX^{I_{(t+1)}} - \hat{\bX}^{I_{(t+2)}}\Vert_F^2	& \leq R_3 \mathbb{E} \Vert \bX^{I_{(t+\frac{1}{2})}} - \hat{\bX}^{I_{(t+1)}}\Vert_F^2  + R_4 \mathbb{E} \Vert \bX^{I_{(t+\frac{1}{2})}}- \Bar{\bX}^{I_{(t+\frac{1}{2})}} \Vert_F^2 +  R_5\eta ^2
	\end{align*}
	Using the update steps of algorithm given in equations \eqref{mat_not_algo_4} and \eqref{mean_seq_iter} (given in Section \ref{prelim}):
	\begin{align*}
	\mathbb{E} \Vert \bX^{I_{(t+1)}} - & \hat{\bX}^{I_{(t+2)}}\Vert_F^2 \leq R_3 \mathbb{E} \left\Vert   \hat{\bX}^{I_{(t+1)}} - \bX^{I_{(t)}} + \sum_{t' = I_{(t)}}^{I_{(t+1)}-1} \eta ( \beta \bV^{(t')} +    \boldsymbol{\nabla F}(\bX^{(t')}, \boldsymbol{\xi}^{(t')} )) \right\Vert_F^2 \notag \\
	& + R_4 \mathbb{E} \left\Vert  \Bar{\bX}^{I_{(t)}} - \bX^{I_{(t)}} -  \sum_{t' = I_{(t)}}^{I_{(t+1)}-1}\eta ( \beta \bV^{(t')} +   \boldsymbol{\nabla F}(\bX^{(t')}, \boldsymbol{\xi}^{(t')} ))\left(   \frac{{\bf 1}{\bf 1}^T}{n} - \mathbf{I} \right) \right\Vert_F^2  +  R_5\eta ^2
	\end{align*}
	For the same $\alpha_5>0$ (from result in Footnote \ref{foot:frob_bound}) used in proof of Lemma \ref{lem_cvx_e1}, we get:
	\begin{align*}
	\mathbb{E} \Vert \bX^{I_{(t+1)}} - \hat{\bX}^{I_{(t+2)}}\Vert_F^2	& \leq R_3 (1+\alpha_5^{-1}) \mathbb{E} \left\Vert   \hat{\bX}^{I_{(t+1)}} - \bX^{I_{(t)}} \right\Vert^2 + R_4 (1+\alpha_5^{-1}) \mathbb{E} \left\Vert  \Bar{\bX}^{I_{(t)}} - \bX^{I_{(t)}} \right\Vert^2 \notag \\
	& \qquad  + R_4 (1+\alpha_5) \mathbb{E}\left\Vert \sum_{t' = I_{(t)}}^{I_{(t+1)}-1} \eta ( \beta \bV^{(t')} +   \boldsymbol{\nabla F}(\bX^{(t')}, \boldsymbol{\xi}^{(t')} ))\left(   \frac{{\bf 1}{\bf 1}^T}{n} - \mathbf{I} \right) \right\Vert_F^2  \notag \\
	& \qquad + R_3 (1+\alpha_5)\mathbb{E} \left\Vert \sum_{t' = I_{(t)}}^{I_{(t+1)}-1}\eta ( \beta \bV^{(t')} +   \boldsymbol{\nabla F}(\bX^{(t')}, \boldsymbol{\xi}^{(t')} )) \right\Vert_F^2  +  R_5\eta ^2
	\end{align*}	
	Using $\verts{\mathbf{AB}}_F \leq \verts{\mathbf{A}}_F \verts{\mathbf{B}}_2$ to split the third term and then using $\left\| \frac{{\bf 1}{\bf 1}^T}{n} - \mathbf{I} \right\| \leq 1$ (from Claim~\ref{J-I_eigenvalue} in supplementary material), and further using the bound in (\ref{bound_gap_grad}) for the third and fourth term, the above can be rewritten as:  
	\begin{align*}
	\mathbb{E} \Vert \bX^{I_{(t+1)}} - \Bar{\bX}^{I_{(t+2)}} \Vert_F^2 & \leq R_3 (1+\alpha_5^{-1}) \mathbb{E} \left\Vert \hat{\bX}^{I_{(t+1)}} - \bX^{I_{(t)}} \right\Vert^2 + R_4 (1+\alpha_5^{-1}) \mathbb{E} \left\Vert  \Bar{\bX}^{I_{(t)}} - \bX^{I_{(t)}} \right\Vert^2 \notag \\
	& \qquad + 2\eta^2 H^2 n G^2 \left( 1 + \frac{\beta^2}{(1-\beta)^2} \right)(1+\alpha_5)(R_3+R_4) + R_5\eta ^2
	\end{align*}
	Defining $Q_2 = 2 H^2 n G^2 \left( 1 + \frac{\beta^2}{(1-\beta)^2} \right)(1+\alpha_5)(R_3+R_4) + R_5 $ completes the proof of Lemma~\ref{lem_cvx_e2}.	
\section{Memory-Efficient Version of SQuARM-SGD} \label{suppl_mem_algo}

In this section, we provide our memory efficient version of SQuARM-SGD proposed in the main paper in Algorithm \ref{alg_dec_sgd_li}. 

\begin{algorithm}[H]
	\caption{ Memory-Efficient SQuARM-SGD}
	{\bf Parameters:} $G = ([n],E)$, $W$
	\begin{algorithmic}[1]
		\STATE {\bf Initialize:} For every $i\in[n]$, set arbitrary $\bx_i^{(0)} \in \mathbb{R}^d$, $\hat{\bx}_i^{(0)} := \bzero$, $\bs_i^{(0)} := \bzero$, $\bv_i^{(-1)} := \mathbf{0}$. Fix the momentum coefficient $\beta$, consensus step-size $\gamma$, learning rate $\eta$, triggering thresholds $\{ c_t \}_{t=0}^T$, and synchronization set $\mathcal{I}_T$.
		\FOR{$t=0$ {\bfseries to} $T-1$ in parallel for all workers $i \in [n]$ } 
		\STATE Sample $\xi_i^{(t)}$, stochastic gradient $\bg_i^{(t)}:= \nabla F_i(\bx_i^{(t)}, \xi_i^{(t)})$
		\STATE $\bv_i^{(t)} = \beta  \bv_i^{(t-1)} + \bg_i^{(t)} $ 
		\STATE $\bx_i^{(t+\frac{1}{2})} := \bx_i^{(t)} - \eta ( \beta \bv_i^{(t)} + \bg_i^{(t)}) $
		\IF{$(t+1) \in I_T $ }
		\FOR{neighbors $j \in \mathcal{N}_i $  }
		\IF{ $\Vert \bx_i^{(t+\frac{1}{2})} - \hat{\bx}_i^{(t)} \Vert_2^2 > {c_t \eta^2} $}
		\STATE Compute $\mathbf{q}_i^{(t)} := \C(\bx_i^{(t+\frac{1}{2})} - \hat{\bx}_i^{(t)} )$
		\STATE Send $\mathbf{q}_i^{(t)}$ to worker $j$ and receive $\mathbf{q}_j^{(t)}$
		\ELSE 
		\STATE Assign $\mathbf{q}_i^{(t)} := 0$
		\STATE Send $\mathbf{q}_i^{(t)}$ to worker $j$ and receive $\mathbf{q}_j^{(t)}$
		\ENDIF
		\ENDFOR
		\STATE $\hat{\bx}_i^{(t+1)} := \mathbf{q}_i^{(t)} + \hat{\bx}_j^{(t)} $
		\STATE $\bs_i^{(t+1)} := \bs_i^{(t)} + \sum \limits_{j=1}^{n} w_{ij}\mathbf{q}_j^{(t)} $
		\STATE $\bx_i^{(t+1)} = \bx_i^{(t+\frac{1}{2})} + \gamma  \left(\hat{\bs}_i^{(t+1)} - \hat{\bx}_i^{(t+1)} \right) $
		\ELSE
		\STATE $\hat{\bx}_i^{(t+1)} = \hat{\bx}_i^{(t)}$ ,  $\bx_i^{(t+1)} = \bx_i^{(t+\frac{1}{2})}$, $\bs_i^{(t+1)} = \bs_i^{(t)} $ 
		\ENDIF
		\ENDFOR
	\end{algorithmic}
\end{algorithm}
The parameter $\bs_i^{(t)}$ for $i \in [n]$ stores the weighted sum of all neighbor copies which is then used in the consensus step. Thus, the requirement for storing copies of all neighbors at a node as in algorithm given in main paper is relaxed.

\end{document}